\documentclass{article}

\usepackage{microtype}
\usepackage{graphicx}
\usepackage{subfigure}
\usepackage{booktabs} 

\usepackage[colorlinks,
linkcolor=red,
anchorcolor=blue,
citecolor=blue
]{hyperref}

\usepackage[accepted]{icml2023}

\usepackage{amsmath}
\usepackage{amssymb}
\usepackage{mathtools}
\usepackage{amsthm}

\usepackage[capitalize,noabbrev]{cleveref}

\usepackage{mathrsfs}
\usepackage{bm}
\usepackage{natbib}

\usepackage{multirow} 
\usepackage{enumitem}
\usepackage{wrapfig}

\usepackage{setspace}
\usepackage{xcolor}

\usepackage{mylatexstyle}
\allowdisplaybreaks

\renewcommand{\zeta}{\overline{\rho}}

\renewcommand{\omega}{\underline{\rho}}

\DeclareMathOperator{\polylog}{\rm polylog}

\newcommand{\la}{\langle}
\newcommand{\ra}{\rangle}

\newtheorem{condition}[theorem]{Condition}

\def \poly {\mathrm{poly}}

\ifdefined\final
\usepackage[disable]{todonotes}
\else
\usepackage[textsize=tiny]{todonotes}
\fi
\allowdisplaybreaks

\icmltitlerunning{Benign Overfitting in Two-layer ReLU Convolutional Neural Networks}

\begin{document}

\twocolumn[
\icmltitle{Benign Overfitting in Two-layer ReLU Convolutional Neural Networks}




\icmlsetsymbol{equal}{*}

\begin{icmlauthorlist}
\icmlauthor{Yiwen Kou}{equal,uclacs}
\icmlauthor{Zixiang Chen}{equal,uclacs}
\icmlauthor{Yuanzhou Chen}{uclacs}
\icmlauthor{Quanquan Gu}{uclacs}
\end{icmlauthorlist}

\icmlaffiliation{uclacs}{Department of Computer Science, University of California, Los Angeles}

\icmlcorrespondingauthor{Quanquan Gu}{qgu@cs.ucla.edu}

\icmlkeywords{Machine Learning, ICML}

\vskip 0.3in
]



\printAffiliationsAndNotice{\icmlEqualContribution} 

\begin{abstract}
Modern deep learning models with great expressive power can be trained to overfit the training data but still generalize well. This phenomenon is referred to as \textit{benign overfitting}. Recently, a few studies have attempted to theoretically understand benign overfitting in neural networks. 
However, these works are either limited to neural networks with smooth activation functions or to the neural tangent kernel regime. 
How and when benign overfitting can occur in ReLU neural networks remains an open problem. In this work, we seek to answer this question by establishing algorithm-dependent risk bounds for learning two-layer ReLU convolutional neural networks with label-flipping noise. We show that, under mild conditions, the neural network trained by gradient descent can achieve near-zero training loss and Bayes optimal test risk. 
Our result also reveals a sharp transition between benign and harmful overfitting under different conditions on data distribution in terms of test risk. 
Experiments on synthetic data back up our theory.
\end{abstract}

\section{Introduction}

Modern deep learning models have a large number of parameters, often exceeding the number of training data points. Despite being over-parameterized and overfitting the training data, these models can still make accurate predictions on the unseen test data \citep{zhang2016understanding, neyshabur2018towards}. This phenomenon, often referred to as \textit{benign overfitting} \citep{bartlett2020benign}, has revolutionized traditional theories of statistical learning and attracted great attention from the statistics and machine learning communities \citep{belkin2018understand, belkin2019reconciling, belkin2020two, hastie2022surprises}. 

There has been a line of work in recent years studying benign overfitting from the theoretical perspective. 
Despite their contributions and insights into the benign overfitting phenomenon, most of these works focus on linear models \citep{belkin2020two, bartlett2020benign, hastie2022surprises, wu2020optimal, chatterji2021finite, zou2021benign, cao2021risk} or kernel/random features models \citep{belkin2018understand,liang2020just,montanari2022interpolation}. \citet{adlam2020neural} and \citet{li2021towards} focused on benign overfitting in neural network models, yet their results are limited to the neural tangent kernel (NTK) regime \citep{jacot2018neural}, where the neural network learning is essentially equivalent to kernel regression. 

Understanding benign overfitting in neural networks beyond the NTK regime is much more challenging because of the non-convexity of the problem. 
Recently, \citet{frei2022benign} studied the problem of learning log-concave mixture data with label-flipping noise, using fully-connected two-layer neural networks with smoothed leaky ReLU activation. They proved the risk upper bound under certain regularity conditions, which matches the lower bound given in \citet{cao2021risk} when the label-flipping noise is zero. 
\citet{cao2022benign} provided an analysis for learning two-layer convolutional neural networks (CNNs) with polynomial ReLU activation function (ReLU$^q$, $q > 2$). Their analysis also identifies a condition that controls the phase transition between benign and harmful overfitting. 
The analyses in both \citet{frei2022benign} and \citet{cao2021risk} highly rely on smooth activation functions and cannot deal with the most widely used ReLU activation function. Thus, there remains an open question: 
\begin{center}
    \emph{How and when does benign overfitting occur in ReLU neural networks? } 
\end{center}

In this paper, we seek to answer the above question by establishing risk bounds for learning two-layer CNNs with ReLU activation function. 

\subsection{Problem Setup}
We consider a similar data distribution that was explored in \citet{cao2022benign}. In this particular distribution, the input data consists of two types of components: \textit{label dependent signals} and \textit{label independent noises}. This data generation model takes inspiration from image data, where the inputs are composed of various patches, and only certain patches are relevant to the class label of the image. Similar models have also been investigated in recent works by \citet{li2019towards,allen2020feature,allen2020towards,zou2021understanding,shen2022data}.

\begin{definition}\label{def:data}
Let $\bmu\in \RR^d$ be a fixed vector representing the signal contained in each data point. Each data point $(\xb,y)$ with predictor $\xb = [\xb^{(1)\top}, \xb^{(2)\top}]^{\top}\in\RR^{2d}, \xb^{(1)}, \xb^{(2)} \in \RR^{d}$ and label $y\in\{-1,1\}$ is generated from a distribution $\cD$, which we specify as follows: 
\begin{enumerate}[leftmargin = *]
    \item The true label $\hat{y}$ is generated as a Rademacher random variable, i.e. $\PP[\hat{y}=1] = \PP[\hat{y}=-1]= 1/2$. The observed label $y$ is then generated by flipping $\hat{y}$ with probability $p$ where $p<1/2$, i.e. $\PP[y=\hat{y}]=1-p$ and $\PP[y=-\hat{y}]=p$. 
    \item A noise vector $\bxi$ is generated from the Gaussian distribution $\cN(\mathbf{0}, \sigma_p^2 \Ib)$. 
    \item  One of $\xb^{(1)}, \xb^{(2)}$ is randomly selected and then assigned as $\hat{y}\cdot\bmu$, which represents the signal, while the other is given by $\bxi$, which represents noises.
\end{enumerate}

\end{definition}

Definition \ref{def:data} strictly generalizes the data distribution in \citet{cao2022benign}, in the sense that it introduces label-flipping noise to the true label $\hat{y}$, and relaxes the orthogonal condition between the signal vector $\bmu$ and the noise vectors $\bxi$ (See Definition 3.1 in \citet{cao2022benign} for a comparison). 

Given a training data set $S = \{(\xb_i,y_i)\}_{i=1}^n$ drawn from some unknown joint distribution $\cD$ over $\xb\times y$, we train a two-layer CNN with ReLU activation by minimizing the following empirical risk measured by logistic loss
\begin{align}\label{eq:objective_definition}
    L_S(\Wb) 
    &= \frac{1}{n} {\sum_{i=1}^n} \ell[ y_i \cdot f(\Wb,\xb_i) ],
\end{align}
where $\ell(z) = \log(1 + \exp(-z))$, and $f(\Wb,\xb)$ is the two-layer CNN (See Section \ref{sec:prob} for the detailed definition). We will use gradient descent to minimize the training loss $L_S(\Wb)$, and we are interested in characterizing the test error (i.e., true error) defined by 
\begin{align}\label{eq:test error01}
L_{\cD}^{0-1} (\Wb) := \PP_{(\xb,y )\sim \cD} \big[ y \neq \sign \big(f(\Wb,\xb)\big) \big].
\end{align}


\subsection{Main Contributions}
We prove the following main result, which characterizes the training loss and test error of the two-layer ReLU CNN trained by gradient descent.

\begin{theorem}[Informal]\label{thm:signal_learning_main_informal}
For any $\epsilon > 0$, under certain regularity conditions, with probability at least $ 1 - \delta$, there exists $0 \leq t \leq T$ such that:
\begin{enumerate}[leftmargin = *]
    \item The training loss converges to $\epsilon$, i.e., $L_S(\Wb^{(t)}) \leq \epsilon$.
    \item If $n\|\bmu\|_{2}^{4} \geq \Omega(\sigma_{p}^{4}d)$, we have
    $ L_{\cD}^{0-1} (\Wb^{(t)}) \leq p + \exp\Big(-  n\|\bmu\|_{2}^{4}/(C_2\sigma_{p}^{4}d)\Big)$. 
    \item If $n\|\bmu\|_{2}^{4} \leq O(\sigma_p^{4} d)$, we have $ L_{\cD}^{0-1} (\Wb^{(t)}) \geq p + 0.1$. 
\end{enumerate}
\end{theorem}

The significance of Theorem \ref{thm:signal_learning_main_informal} is highlighted as follows:
\begin{itemize}[leftmargin = *]
\item The ReLU CNN trained by standard gradient
descent on the logistic loss can interpolate the noisy training data and achieve near-zero training loss. 
\item Under the condition on the data distribution and the training sample size that $n\|\bmu\|_{2}^{4} \geq \Omega(\sigma_{p}^{4}d)$, 
the learned CNN 
can achieve nearly optimal test error (i.e., Bayes risk $p$).

\item 
On the flip side, if  $n\|\bmu\|_{2}^{4} \leq O(\sigma_p^{4} d)$, 
the interpolating CNN model will 
suffer a test error that is at least a constant worse than the Bayes risk. 
This together with the positive result reveals 
a sharp 
phase transition between benign and harmful overfitting.
\end{itemize}

 Our analysis relies on several new proof techniques that significantly generalize 
the 
signal-noise decomposition technique \citep{cao2022benign}. More specifically, to handle ReLU activation, we directly use the activation pattern and data structure to characterize the loss of each training data point rather than using the smoothness condition. To deal with the label-flipping noise, we show that the loss of each 
training data point 
decreases at roughly the same rate throughout training, 
which ensures signal learning even in the presence of label noise. 





\section{Related Work} 

In this section, we will discuss in detail some of the related work briefly mentioned before.


\noindent\textbf{Benign overfitting of linear models.} 
One line of research sought a theoretical understanding of the benign overfitting phenomenon in linear models. Some of these works focused on linear regression problems.
\citet{belkin2020two} provided a precise analysis for the shape of the risk curve in Gaussian and Fourier series models with the least squares predictor.
\citet{hastie2022surprises, wu2020optimal} studied the setting where both the dimension and sample size grow but their ratio is fixed and demonstrated a double descent risk curve with respect to this ratio. 
\citet{bartlett2020benign} established matching upper and lower risk bounds for the over-parameterized minimum norm interpolator and showed that benign overfitting can occur under certain conditions on the spectrum of the data covariance. 
\citet{zou2021benign} studied how well constant stepsize stochastic gradient descent with iterate averaging or tail averaging generalizes in the over-parameterized regime. 
Several other works studied benign overfitting of maximum margin linear classifiers.
\citet{muthukumar2021classification} showed that the max-margin predictor and the least square predict coincide in the overparametrized regime, and generalize differently when using 0-1 loss and square loss functions.
\citet{wang2021benign,cao2021risk} respectively studied Gaussian and sub-Gaussian mixtures data models without label noise and characterized the condition under which benign overfitting can occur.
\citet{chatterji2021finite} showed that the maximum margin algorithm  trained on noisy data can achieve nearly optimal risk with sufficient overparameterization. 
\citet{shamir2022implicit} studied both minimum-norm interpolating predictors for linear regression and max-margin predictors for classification and discussed the conditions under which benign overfitting can or cannot occur.

\textbf{Benign overfitting of neural networks.} 
A series of recent works studied benign overfitting of neural networks. 
\citet{liang2020multiple} showed that kernel “ridgeless” regression can lead to a multiple-descent risk curve for various scaling of input dimension and sample size.
\citet{adlam2020neural} provided a precise analysis of generalization under kernel regression and revealed non-monotonic behavior for the test error.
\citet{li2021towards} examined benign overfitting in random feature models defined as two-layer neural networks. 
\citet{montanari2022interpolation} studied two-layer neural networks in the NTK regime, focusing on its generalization properties when  dimension, sample size, and the number of neurons are overparametrized and polynomially related. 
\citet{chatterji2022deep} bounded the excess risk of interpolating deep linear networks trained by gradient flow and showed that randomly initialized deep linear networks can  closely approximate the risk bounds for the minimum norm interpolator. 

\section{Preliminaries}\label{sec:prob}

In this section, we introduce the notation, two-layer CNN models, and the gradient descent-based training algorithm. 


\noindent\textbf{Notation.} We use lower case letters, lower case bold face letters, and upper case bold face letters to denote scalars, vectors, and matrices respectively.  For a vector $\vb=(v_1,\cdots,v_d)^{\top}$, we denote by $\|\vb\|_2:=\big(\sum_{j=1}^{d}v_j^2\big)^{1/2}$ its $l_2$ norm. For two sequence $\{a_k\}$ and $\{b_k\}$, we denote $a_k=O(b_k)$ if $|a_k|\leq C|b_k|$ for some absolute constant $C$, denote $a_k=\Omega(b_k)$ if $b_k=O(a_k)$, and denote $a_k=\Theta(b_k)$ if $a_k=O(b_k)$ and $a_k=\Omega(b_k)$. We also denote $a_k=o(b_k)$ if $\lim|a_k/b_k|=0$. Finally, we use $\tilde{O}(\cdot)$ and $\tilde{\Omega}(\cdot)$ to omit logarithmic terms in the notation.


\noindent\textbf{Two-layer CNNs.} We consider a two-layer convolutional neural network described in the following: 
its first layer consists of $m$ positive filters and $m$ negative filters, with each filter applying to the two patches $\xb^{(1)}$ and $\xb^{(2)}$ separately; 
its second layer parameters are fixed as $+1/m$ and $-1/m$ respectively for positive and negative convolutional filters. 
Then the network can be written as $f(\Wb, \xb) = F_{+1}(\Wb_{+1},\xb) - F_{-1}(\Wb_{-1},\xb)$, 
where the partial network function of positive and negative filters $F_{+1}(\Wb_{+1},\xb)$, $F_{-1}(\Wb_{-1},\xb)$ are defined as: 
\begin{align*}
    F_j(\Wb_j,\xb)  &= \frac{1}{m}{\sum_{r=1}^m} \big[\sigma(\la\wb_{j,r},\xb^{(1)}\ra) + \sigma(\la\wb_{j,r}, \xb^{(2)}\ra)\big]\\
    &= \frac{1}{m} {\sum_{r=1}^m} \big[\sigma(\la\wb_{j,r},\hat{y}\cdot\bmu\ra) + \sigma(\la\wb_{j,r}, \bxi\ra)\big]
\end{align*}
for $j \in \{\pm 1\}$. Here $\sigma(z) = \max\{0,z\}$ is the ReLU activation function, $\Wb_{j}$ is the collection of model weights associated with $F_j$ 
(positive/negative filters), 
and $\wb_{j,r}\in\RR^{d}$ denotes the weight vector for the $r$-th filter / neuron in $\Wb_{j}$. We use $\Wb$ to denote the collection of all model weights. We note that our CNN model can also be viewed as a CNN with average global pooling \citep{lin2013network}.
Besides the training loss and test error defined in \eqref{eq:objective_definition} and \eqref{eq:test error01}, we also define the true loss (test loss) as $L_{\cD}(\Wb) := \EE_{(\xb,y )\sim \cD} \ell[ y \cdot f(\Wb,\xb) ]$. 

\noindent\textbf{Training algorithm.} We use gradient descent to optimize \eqref{eq:objective_definition}. The gradient descent update of the filters in the CNN can be written as
\begin{align}
    \wb_{j,r}^{(t+1)} &= \wb_{j,r}^{(t)} - \eta \cdot \nabla_{\wb_{j,r}} L_S(\Wb^{(t)}) \notag\\
    &= \wb_{j,r}^{(t)} - \frac{\eta}{nm} \sum_{i=1}^n \ell_i'^{(t)} \cdot  \sigma'(\la\wb_{j,r}^{(t)}, \bxi_{i}\ra)\cdot j y_{i}\bxi_{i} \notag\\
    &\qquad- \frac{\eta}{nm} \sum_{i=1}^n \ell_i'^{(t)} \cdot \sigma'(\la\wb_{j,r}^{(t)}, \hat{y}_{i} \bmu\ra)\cdot \hat{y}_{i}y_i j\bmu \label{eq:gdupdate}.
\end{align}
for all $j \in \{\pm 1\}$ and $r \in [m]$, where we introduce a shorthand notation $\ell_i'^{(t)} = \ell'[ y_i \cdot f(\Wb^{(t)},\xb_i) ] $ and assume the gradient of the ReLU activation function 
at $0$ to be 
$\sigma'(0) = 1$
without losing generality. We initialize the gradient descent by Gaussian initialization, where all entries of $\Wb^{(0)}$ are sampled from i.i.d. Gaussian distributions $\cN(0, \sigma_0^2)$, with $\sigma_0^2$ as the variance.

\section{Main Results}

In this section, we present our main theoretical results. Our results are based on the following conditions on the dimension $d$, sample size $n$, neural network width $m$, initialization scale $\sigma_0$, signal norm $\| \bmu \|_2$, noise rate $p$, and learning rate $\eta$. 
In this paper, we consider the learning period $0 \leq t \leq T^{*}$, where $T^{*} = \eta^{-1}\poly(\epsilon^{-1}, d, n, m)$ is the maximum admissible iterations. 
We 
can deal with any polynomial maximum admissible iterations $T^*$ greater than $\tilde{\Omega}(\eta^{-1}\epsilon^{-1}mnd^{-1}\sigma_p^{-2})$.


\begin{condition}\label{condition:d_sigma0_eta}
Suppose there exists a sufficiently large constant $C$, such that the following hold: 
\begin{enumerate}[leftmargin = *]
    \item Dimension $d$ is sufficiently large:  $d \geq C \max\{n\sigma_p^{-2}\|\bmu\|_{2}^{2}\log(T^*), n^{2}\log(nm/\delta) (\log (T^{*}))^2\}$.
    \item Training sample size $n$ and neural network width satisfy $m\geq C\log(n/\delta), n \geq C\log(m/\delta)$.
    
    \item The norm of the signal satisfies $\|\bmu\|_2^2 \geq C \cdot \sigma_p^2 \log(n / \delta)$. 
    \item The noise rate $p$ satisfies $p\leq 1/C$.
    \item The standard deviation of Gaussian initialization $\sigma_0$ is appropriately chosen such that $\sigma_0\leq \big(C\max\big\{\sigma_{p}d/\sqrt{n},\sqrt{\log(m/\delta)}\cdot\|\bmu\|_{2}\big\}\big)^{-1}$.
    
    \item The learning rate $\eta$ satisfies $\eta\leq \big(C\max\big\{\sigma_p^2 d^{3/2}/(n^2 m \sqrt{\log(n/\delta)}), \sigma_p^2 d/n\big\}\big)^{-1}$.
    
    
\end{enumerate}
\end{condition}

The conditions on $d,n,m$ are to ensure that the learning
problem 
is in a sufficiently over-parameterized setting, and similar conditions have been made in \citet{chatterji2021finite,cao2022benign,frei2022benign}. 
The conditions on $\sigma_0$ and $\eta$ are to ensure that gradient descent can effectively minimize the training loss. 
The difference between Condition \ref{condition:d_sigma0_eta} and Assumption (A1)-(A6) in \citet{frei2022benign} is that our setting assumes a milder condition 
of order 
$O(d^{-3/2})$ on learning rate $\eta$ 
rather 
than $O(d^{-2})$ ((A5) in \citet{frei2022benign}), as well as a milder condition 
of order 
$O(d^{-1}n^{1/2})$ on initialization $\sigma_0$ 
rather 
than $O(d^{-5/2}m^{-1/2})$ ((A6) in \citet{frei2022benign}). 
Another difference is that \citet{frei2022benign} allows neural networks of arbitrary width $m$, but our condition requires a mild assumption that $m$ should be no more than 
an 
exponential order of dimension $d$.
We also require another mild condition that $m$ and $n$ cannot exceed the exponential order of each other. 
Besides, 
in contrast to \citet{cao2022benign}, 
our Condition \ref{condition:d_sigma0_eta} relaxes the dependency of $m$ and $d$ 
in that, we do not require any polynomial upper bound of the neural network width $m$, whereas Condition 4.2 in \citet{cao2022benign} requires that $m$ is upper bounded by a certain fractional order of $d$. 
Another improvement to \citet{cao2022benign} is that we add label-flipping noise $p$ to the problem, but this is also included in \citet{frei2022benign}. 
Detailed comparisons are shown in 
Table \ref{table: comp with frei} and Table \ref{table: comp with cao}.

Based on these conditions, we give our main result in the following theorem. 



\begin{theorem}\label{thm:signal_learning_main}
For any $\epsilon > 0$, under Condition~\ref{condition:d_sigma0_eta}, with probability at least $ 1 - \delta$ there exists $t=  \tilde{O}(\eta^{-1}\epsilon^{-1}mnd^{-1}\\\sigma_p^{-2})$ such that:
\begin{enumerate}[leftmargin = *]
    \item The training loss converges to $\epsilon$, i.e., $L_S(\Wb^{(t)}) \leq \epsilon$.
    \item When $n\|\bmu\|_{2}^{4} \geq C_{1}\sigma_{p}^{4}d$, the trained CNN will generalize with classification error close to the noise rate $p$: 
    $ L_{\cD}^{0-1} (\Wb^{(t)}) \leq p + \exp\Big(-  n\|\bmu\|_{2}^{4}/(C_2\sigma_{p}^{4}d)\Big)$. 
    \item When $n\|\bmu\|_{2}^{4} \leq C_3\sigma_p^{4} d$, the test error $ L_{\cD}^{0-1} (\Wb^{(t)}) \geq p + 0.1$. 
\end{enumerate}
Here $C_{1}, C_{2}, C_{3}$ are some absolute constants. 
\end{theorem}

\begin{remark}
Theorem~\ref{thm:signal_learning_main} demonstrates that the training loss converges to $\epsilon$ within $ \tilde{O}(\eta^{-1}\epsilon^{-1}mnd^{-1}\sigma_p^{-2})$ iterations. Moreover, when the training loss converges, the model can achieve optimal test error if the signal-to-noise ratio is large. However, if the signal-to-noise ratio is small, the model will experience a test error that is at least a constant worse than the Bayes risk. The threshold for this distinction is determined by the condition $n\|\bmu\|_{2}^{4}=\Theta(\sigma_p^4d)$. In addition to the results mentioned in Theorem~\ref{thm:signal_learning_main}, it is important to emphasize that the second and third bullet points regarding the test error also hold true for training time $t$ that is greater than $\tilde{O}(\eta^{-1}\epsilon^{-1}mnd^{-1}\sigma_p^{-2})$, but smaller than the maximum allowable iterations $T^{*} = \eta^{-1}\poly(\epsilon^{-1}, d, n, m)$.
\end{remark}

\textbf{Comparison with prior works.} Although Theorem \ref{thm:signal_learning_main} and Theorem 3.1 in \citet{frei2022benign} both show that the network achieves arbitrarily small logistic loss, and simultaneously achieves test error close to the noise rate, our results differ from \citet{frei2022benign} since \citet{frei2022benign} considered a neural network with \emph{smoothed leaky ReLU} activation, while we consider the ReLU activation which is not smooth. 
Besides, to obtain a training error smaller than $\epsilon$, \citet{frei2022benign} needed $O(\epsilon^{-2})$ iterations, whereas our results only require $O(\epsilon^{-1})$ iterations. 
In contrast to \citet{cao2022benign} which studied 
a CNN model with ReLU$^q (q>2)$ activation function and without label noise, 
our setting is more practical as we work with ReLU activation, take label-flipping noise into consideration, and also remove the orthogonal assumption between the signal patch and the noise patch. 
Because of the label-flipping noise, it is more natural to
evaluate generalization performance by comparing 
the test error with the Bayes optimal classifier.
This is why our Theorem~\ref{thm:signal_learning_main} provides test error bounds while \citet{cao2022benign} provided test loss bounds and despite the difference both our results and theirs present exact phase transition conditions. 

\section{Overview of Proof Techniques} \label{sec:proof}
In this section, we discuss the main challenges in
studying 
benign overfitting under our setting, and explain some key techniques we implement in our proofs to overcome these challenges. Based on these techniques, the proof of our main Theorem~\ref{thm:signal_learning_main} will follow naturally. The complete proofs of all the results are given in the appendix.




\subsection{Key Technique 1: Time-invariant Coefficient Ratio}


Our first main challenge is dealing with the ReLU activation function, i.e., $\sigma(z) = \max\{0, z\}$. This is one of the most common and widely used activation functions, but as we explain below, it is also hard to analyze. The key difficulty in establishing benign overfitting guarantees is demonstrating that the neural network can interpolate the data.  \citet{frei2022benign} adopted the smoothness-based convergence proof technique proposed in \citet{frei2021proxy}. This technique requires the activation function to be strictly increasing and smooth, therefore it cannot be applied to the ReLU activation function. \citet{cao2022benign} provides an iterative analysis of the coefficients in the signal-noise decomposition, which is given in the following definition. 
\begin{definition}\label{def:w_decomposition}
Let $\wb_{j,r}^{(t)}$ for $j\in \{\pm 1\}$, $r \in [m]$ be the convolution filters of the CNN at the $t$-th iteration of gradient descent. Then there exist unique coefficients $ \gamma_{j,r}^{(t)} $ and $\rho_{j,r,i}^{(t)}$ such that 
\begin{align*}
    \wb_{j,r}^{(t)} = \wb_{j,r}^{(0)} + j \cdot \gamma_{j,r}^{(t)} \cdot \| \bmu \|_2^{-2} \cdot \bmu + \sum_{ i = 1}^n \rho_{j,r,i}^{(t) }\cdot \| \bxi_i \|_2^{-2} \cdot \bxi_{i}.
\end{align*}
Further denote $\zeta_{j,r,i}^{(t)} := \rho_{j,r,i}^{(t)}\ind(\rho_{j,r,i}^{(t)} \geq 0)$, $\omega_{j,r,i}^{(t)} := \rho_{j,r,i}^{(t)}\ind(\rho_{j,r,i}^{(t)} \leq 0)$. Then 
\begin{equation}
    \begin{aligned} \label{eq:w_decomposition}
        \wb_{j,r}^{(t)} &= \wb_{j,r}^{(0)} + j \cdot \gamma_{j,r}^{(t)} \cdot \| \bmu \|_2^{-2} \cdot \bmu \\
        &\qquad+ \sum_{ i = 1}^n \zeta_{j,r,i}^{(t) }\cdot \| \bxi_i \|_2^{-2} \cdot \bxi_{i} + \sum_{ i = 1}^n \omega_{j,r,i}^{(t) }\cdot \| \bxi_i \|_2^{-2} \cdot \bxi_{i} .
    \end{aligned}
\end{equation}
\end{definition}


\eqref{eq:w_decomposition} is called the \textit{signal-noise decomposition} of $\wb_{j,r}^{(t)}$ where the normalization factors $\|\bmu\|_{2}^{-2}, \|\bxi_{i}\|_{2}^{-2}$ are to ensure that $\gamma_{j,r}^{(t)} \approx \la \wb_{j,r}^{(t)}, j \bmu \ra$, $\rho_{j,r,i}^{(t)} \approx \la \wb_{j,r}^{(t)}, \bxi_{i} \ra$. With Definition \ref{def:w_decomposition}, one can reduce the study of the CNN learning process to a careful assessment of the coefficients $\gamma_{j,r}^{(t)}$, $\zeta_{j,r,i}^{(t)}$, $\omega_{j,r,i}^{(t)}$ throughout training. This technique does not rely on the strictly increasing and smoothness properties of the activation function and will act as the basis of our analysis. However, \citet{cao2022benign} only characterized the behavior of the leading 
neurons by studying 
$\max_{r}\gamma_{j,r}^{(t)}$, $\max_{r}\zeta_{j,r,i}^{(t)}$. To guarantee that the leading neuron can dominate other neurons after training, 
they require neurons with different initial weights to have different update speeds, which is guaranteed thanks to the activation function ReLU$^q$ with $q > 2$. 
But the ReLU function is piece-wise linear, and every activated neuron has the same learning speed $\sigma'(x) = 1$. 
Therefore dealing with ReLU requires new techniques. 


To overcome 
this difficulty, 
we propose a \emph{time-invariant coefficient ratio} analysis which generalizes \citet{cao2022benign}'s technique. The key lemma is presented as follows, which characterizes the coefficient orders at any time $t \leq T^{*}$ and 
helps derive the second and third parts 
of Theorem \ref{thm:signal_learning_main} on the upper and lower bounds of test error.

\begin{proposition}\label{Proposition: range}
Under Condition \ref{condition:d_sigma0_eta}, the following bounds hold for $t\in[0,T^*]$:
\begin{itemize}[leftmargin=*]
    \item $\zeta_{j,r,i}^{(t)}$ is an increasing sequence. Besides, $0\leq\zeta_{j,r,i}^{(t)}\leq4\log(T^*)$ for all $j\in\{\pm1\}$, $r\in[m]$ and $i\in[n]$.
    \item $\omega_{j,r,i}^{(t)}$ is a decreasing sequence. Besides, $-4\log(T^*)\leq -2\max_{i,j,r}\{|\la\wb_{j,r}^{(0)},\bmu\ra|,|\la\wb_{j,r}^{(0)},\bxi_i\ra|\}-10n \sqrt{\log(6n^2/\delta)/d}\cdot4\log(T^{*}) \leq \omega_{j,r,i}^{(t)} \leq 0$ for all $j\in\{\pm1\}$, $r\in[m]$ and $i\in[n]$.
    \item $\gamma_{j,r}^{(t)}$ is a strictly increasing sequence. Besides, $\gamma_{j, r}^{(t)}=\Theta(\|\bmu\|_{2}^{2}/(d\sigma_p^2))\sum_{i=1}^{n} \zeta_{j, r, i}^{(t)} $ for all $j\in\{\pm1\}$ and $r\in[m]$.
\end{itemize}
\end{proposition}

In 
Definition~\ref{def:w_decomposition}, 
$\gamma_{j,r}^{(t)}$ characterizes the progress of learning the signal vector $\bmu$, and $\rho_{j,r,i}^{(t)}$ characterizes the degree of noise memorization by the filter. The first and second bullets in Proposition~\ref{Proposition: range} tell us that for any iteration $t$, 
the degree of noise memorization $\zeta_{j,r,i}^{(t)}, \omega_{j,r,i}^{(t)}$ are bounded by a logarithmic order of total epochs $T^*$. In particular, when $T^{*} = \eta^{-1}\poly (\epsilon^{-1}, d, n, m) $, $\zeta_{j,r,i}^{(t)}, \omega_{j,r,i}^{(t)} = \tilde{O}(1)$. The third bullet in Proposition~\ref{Proposition: range} is the major improvement of our technique compared to \citet{cao2022benign}. It shows that $\gamma_{j,r}^{(t)}$ is strictly increasing, indicating that the CNN will learn the signal $\bmu$ 
despite 
label-flipping noise. Besides, 
the order of the coefficient ratio $\gamma_{j,r}^{(t)}/(\sum_{i=1}^{n} \zeta_{j, r, i}^{(t)})$ is time-invariant.  
When the signal strength $\| \bmu \|_2$ is large compared to the noise variance $\sqrt{d} \sigma_p$, the neurons tend to learn the signal. When $\| \bmu \|_2$ is small compared to $\sqrt{d} \sigma_p$, the neurons tend to learn the noises. By the time-invariant coefficient ratio technique, we can characterize the behavior of all the neurons during training, which enables us to deal with the ReLU activation function.

To prove the third bullet, we need to characterize the activation pattern of $\sigma(\la\wb_{j,r}^{(t)},\bxi_i\ra)$. Observing that the increment of $\sum_{i}\zeta_{i,j,r}^{(t)}$ is scaled by $\sum_i\sigma'(\la\wb_{j,r}^{(t)},\bxi_i\ra)$, for any weight $\wb_{j,r}^{(t)}$, we consider the set sequence $\{S_{j,r}^{(t)}\}_{t=0}^{T^*}$, where $S_{j,r}^{(t)}$ is defined as $\{i|y_i=j,\la\wb_{j,r}^{(t)},\bxi_i\ra>0\}$. We show that this is an increasing set sequence throughout the training, leading to $|S_{j,r}^{(t)}|=\Theta(n)$. 
This intuitively means that for a given sample, once a neuron is activated by the noise patch, it will remain activated throughout training even though the weights of the neuron are updated by gradient descent. 
Applying this finding to \eqref{iterative equation2} and \eqref{iterative equation3}, it follows directly that the increment ratio of $\gamma_{j,r}^{(t)}$ and $\sum_i\zeta_{j,r,i}^{(t)}$ will always remain $\Theta(\|\bmu\|_2^2/(\sigma_p^2 d))$.


\subsection{Key Technique 2: Automatic Balance of Coefficient Updates}

Our second main challenge is dealing with label-flipping noise. Empirical studies found that over-parameterized neural networks can generalize well when trained on data with label noise \citep{belkin2019reconciling, zhang2021understanding}, 
which is in conflict with the long-standing theories of statistical learning. 
To fit corrupted data 
with signal $-y\bmu^{\top}$ and noise $\bxi$, 
the neural network weights must capture the random noise $\bxi$, which harms generalization. Even worse, 
label-flipping noise may trick the learner into capturing the adversarial signal $-\bmu$ rather than $\bmu$. 
Let us investigate the update rule of the coefficient $\gamma_{j,r}, \zeta_{j,r,i}, \omega_{j,r,i}$.

\begin{lemma}\label{lemma:coefficient_iterative_proof}
The coefficients $\gamma_{j,r}^{(t)},\zeta_{j,r,i}^{(t)},\omega_{j,r,i}^{(t)}$ defined in Definition~\ref{def:w_decomposition} satisfy the following iterative equations:
\begin{align}
    &\gamma_{j,r}^{(0)},\zeta_{j,r,i}^{(0)},\omega_{j,r,i}^{(0)} = 0,\label{iterative equation1}\\
    &\gamma_{j,r}^{(t+1)} = \gamma_{j,r}^{(t)} - \frac{\eta}{nm} \cdot \bigg[\sum_{i\in S_{+}} \ell_i'^{(t)} \sigma'(\la\wb_{j,r}^{(t)}, \hat{y}_{i} \cdot \bmu\ra)\notag\\
    &\qquad\qquad- \sum_{i\in S_{-}} \ell_i'^{(t)} \sigma'(\la\wb_{j,r}^{(t)}, \hat{y}_{i} \cdot \bmu\ra)\bigg]  \cdot \| \bmu \|_2^2, \label{iterative equation2}\\
    &\zeta_{j,r,i}^{(t+1)} = \zeta_{j,r,i}^{(t)} - \frac{\eta}{nm} \cdot \ell_i'^{(t)}\cdot \sigma'(\la\wb_{j,r}^{(t)}, \bxi_{i}\ra) \cdot \| \bxi_i \|_2^2\notag\\
    &\qquad\qquad\qquad\qquad\qquad\qquad\qquad\cdot \ind(y_{i} = j), \label{iterative equation3}\\
    &\omega_{j,r,i}^{(t+1)} = \omega_{j,r,i}^{(t)} + \frac{\eta}{nm} \cdot \ell_i'^{(t)}\cdot \sigma'(\la\wb_{j,r}^{(t)}, \bxi_{i}\ra)\cdot \| \bxi_i \|_2^2\notag\\
    &\qquad\qquad\qquad\qquad\qquad\qquad\qquad\cdot \ind(y_{i} = -j),\label{iterative equation4}
\end{align}
for all $r\in [m]$,  $j\in \{\pm 1\}$ and $i\in [n]$, where $S_{+}:=\{i\in[n]|y_i=\hat{y}_i\}$ and $S_{-}:=\{i\in[n]|y_i\neq\hat{y}_i\}$.
\end{lemma} 
When there is no label-flipping noise, we can conclude that $S_- = \varnothing$ and the signal coefficient $\gamma_{j,r}^{(t)}$ is strictly increasing since $\ell_i'^{(t)}$ is strictly negative. This key observation plays an important role in the proof of \citet{cao2022benign}. Unfortunately, 
the presence of noisy labels introduces the presence of 
a negative 
term $\sum_{i\in S_{-}}\ell_i'^{(t)}\sigma'(\la\wb_{j,r}^{(t)},\hat{y}_i\cdot\bmu\ra)$. Therefore, we cannot conclude directly from formula \eqref{iterative equation1} whether $\gamma_{j,r}^{(t)}$ is increasing or not. If the gradient of losses $\ell_i'^{(t)}$ for (noisy) samples $i \in S_{-}$ are particularly large relative to the gradient of losses $\ell_i'^{(t)}$ for (clean) samples  $i \in S_{+}$, then indeed \eqref{iterative equation1} may fail to guarantee an increase of $\gamma_{j,r}^{(t)}$. In order to show that the neural networks can still learn signals while interpolating the noisy data $S_{-}$, we need more advanced and careful characterization of the learning process. 

To overcome the difficulty in dealing with label-flipping noise, we apply a key technique called \emph{automatic balance of coefficient updates}. As indicated in \eqref{iterative equation2}, if we can show that the loss gradients $\ell_i'^{(t)}$ are essentially `balanced' across all samples, i.e., $\ell_i'^{(t)}/\ell_k'^{(t)} \leq C, \forall i, j \in [n]$, then provided that the fraction of noisy labels is not too large, 
the effect of the noisy labels will be countered by clean labels, and one can eventually show that 
$\gamma_{j,r}^{(t)}$ is increasing. This provides motivation for our next lemma. 

\begin{lemma}\label{lemma: logit ratio bound}
Under Condition \ref{condition:d_sigma0_eta}, the following bounds hold for any $t\in[0,T^*]$: 
\begin{equation}\label{F difference}
    y_i\cdot f(\Wb^{(t)},\xb_i)-y_k\cdot f(\Wb^{(t)},\xb_k)\leq C_4,
\end{equation}
\begin{equation}\label{logit ratio}
    \ell_i'^{(t)}/\ell_k'^{(t)}\leq C_5,
\end{equation}
for any $i,k\in[n]$, where $C_4=\Theta(1)$ is a positive constant, $C_5=\exp(C_4)$, and $\ell_i'^{(t)}=\ell'(y_i f(\Wb^{(t)},\xb_i))$, $\ell_k'^{(t)}=\ell'(y_k f(\Wb^{(t)},\xb_k))$.
\end{lemma}

The strategy of bounding $\ell_i'^{(t)}/\ell_k'^{(t)}$ is first proposed by \citet{chatterji2021finite} in studying linear classification and has later been extended to neural networks with smoothed leaky ReLU activation function \citep{frei2021provable, frei2022benign}. The main idea is that according to the property of logit function $\ell'(z)=-1/(1+\exp(z))$ that $\ell'(z_1)/\ell'(z_2)\approx\exp(z_2-z_1)$, to upper bound the ratio of $\ell_i'$ and $\ell_k'$, one only needs to bound the difference between $y_i f(\Wb^{(t)},\xb_i)$ and $y_k f(\Wb^{(t)},\xb_k)$. 
To further characterize 
this difference, 
the works above utilize the smoothness property, translating the function difference to the gradient difference $\nabla f(\Wb^{(t)},\xb_i)$ and $\nabla f(\Wb^{(t)},\xb_k)$. However, such a smoothness-based technique cannot be directly applied to ReLU neural networks. 

In this paper, we apply signal-noise decomposition and approximate $y_i f(\Wb^{(t)},\xb_i)$ by $\sum_{r}\zeta_{y_i,r,i}^{(t)}$ with a small approximation error for any $i\in[n]$. Therefore, Lemma \ref{lemma: logit ratio bound} can be further simplified into proving the following intermediate result.

\begin{lemma}\label{lemma:zeta difference}
Under Condition \ref{condition:d_sigma0_eta}, the following bounds hold for $t\in[0,T^*]$: 
\begin{equation}
    \sum_{r=1}^{m}\zeta_{y_i,r,i}^{(t)}-\sum_{r=1}^{m}\zeta_{y_k,r,k}^{(t)}\leq\kappa,\label{eq:break_down}
\end{equation}
for any $i,k\in[n]$, where $\kappa=\Theta(1)$ is a positive constant.
\end{lemma}

Note that \eqref{eq:break_down} is much easier to deal with than \eqref{F difference} because we can directly use the iterative analysis of \eqref{iterative equation3}, which leads to the update rule:
\begin{equation}
\begin{aligned}
    &\sum_{r=1}^{m}[\zeta_{y_i,r,i}^{(t+1)}-\zeta_{y_k,r,k}^{(t+1)}]=\sum_{r=1}^{m}[\zeta_{y_i,r,i}^{(t)}-\zeta_{y_k,r,k}^{(t)}]\\
    &\qquad-\frac{\eta}{nm}\cdot \big(|S_i^{(t)}|\ell_i'^{(t)}\|\bxi_i\|_2^2-|S_k^{(t)}|\ell_k'^{(t)}\|\bxi_i\|_2^2 \big),
\end{aligned}\label{zeta difference update}
\end{equation}
where $S_i^{(t)}=\{r\in[m]:\la\wb_{y_i,r}^{(t)},\bxi_i\ra\geq0\}, i\in[n]$. Now, we consider two cases: \begin{itemize}[leftmargin=*]
    \item If $\sum_{r=1}^{m}\zeta_{y_i,r,i}^{(t)}-\sum_{r=1}^{m}\zeta_{y_k,r,k}^{(t)}$ is relatively small, 
    we will show that $\sum_{r=1}^{m}\zeta_{y_i,r,i}^{(t+1)}-\sum_{r=1}^{m}\zeta_{y_k,r,k}^{(t+1)}$ will not grow too much for small enough step-size $\eta$.
    \item If $\sum_{r=1}^{m}\zeta_{y_i,r,i}^{(t)}-\sum_{r=1}^{m}\zeta_{y_k,r,k}^{(t)}$ is relatively large, then it will cause $\ell_i^{(t)}/\ell_k^{(t)}$ to contract because $\ell_i^{(t)}/\ell_k^{(t)}$ can be approximated by $\exp(\sum_{r=1}^{m}\zeta_{y_k,r,k}^{(t)}-\sum_{r=1}^{m}\zeta_{y_i,r,i}^{(t)})$. Moreover, since we can prove that $\|\bxi_i\|_2^2\approx\|\bxi_k\|_2^2$ and $|S_i^{(t)}|/|S_k^{(t)}|=\Theta(1)$, we have $\sum_{r=1}^{m}\zeta_{y_i,r,i}^{(t+1)}-\sum_{r=1}^{m}\zeta_{y_k,r,k}^{(t+1)}$ will decrease according to \eqref{zeta difference update}.
\end{itemize}
Combining the two cases, $\sum_{r=1}^{m}\zeta_{y_i,r,i}^{(t)}-\sum_{r=1}^{m}\zeta_{y_k,r,k}^{(t)}$ can be upper bounded by a constant, which completes the proof of Lemma~\ref{lemma:zeta difference}, and the proof of Lemma \ref{lemma: logit ratio bound} directly follows. 



\subsection{Key Technique 3: Algorithm-dependent Test Error Analysis}

By choosing $\epsilon = 1/(4n)$, Theorem~\ref{thm:signal_learning_main} gives that $L_{S}(\Wb^{(t)}) \leq 1/(4n)$ which further implies that the training error is $0$. On the other hand, we know that the Bayes optimal test error is at least $p$ due to the presence of the label-flipping noise. Thus the gap between the test error and training error is at least $p$, which prevents us from applying commonly-used standard uniform convergence-based bounds \citep{bartlett2017spectrally,neyshabur2017pac} or stability-based bounds \citep{hardt2016train, mou2017generalization, chen2018stability}. In this paper, we will give an algorithm-dependent test error analysis. First, we can decompose the test error as follows
\begin{equation}
\begin{aligned}
&\quad \mathbb{P} \big(y \neq \sign (f(\Wb^{(t)}, \xb)) \big)\\
&= p+ (1-2p)\mathbb{P} \big(\hat{y} f(\Wb^{(t)},\xb) \leq 0\big).\label{eq:test error}
\end{aligned}
\end{equation}
With \eqref{eq:test error}, the analysis of test error can be reduced to bounding the wrong prediction probability $\mathbb{P} \big(\hat{y} f(\Wb^{(t)},\xb) \leq 0\big)$. To achieve this, we need to bound the coefficient order when the training loss converges to $\epsilon$. The following result demonstrates that a constant proportion of $\zeta_{y_i,r,i}^{(t)}$ will reach constant order at time $T_1<T^{*}$.

\begin{lemma}\label{lemma:stage 1}
    Under Condition \ref{condition:d_sigma0_eta}, there exists $T_1=\Theta(\eta^{-1}nm\sigma_p^{-2}d^{-1})$ such that $\zeta_{y_i,r,i}^{(T_1)}\geq 2$ for all $r\in S_i^{(0)}:=\{r\in[m]:\la\wb_{y_i,r}^{(0)},\bxi_i\ra>0\}$ and $i\in[n]$.
\end{lemma}

The main idea in proving this lemma is that $\ell_i'^{(t)}$ remain $\Theta(1)$ before time $T_1$, and the dynamics of the coefficients in \eqref{iterative equation3} can be greatly simplified by replacing the $\ell_i'^{(t)}$ factors by their constant lower bounds. After time $T_1$, by the monotonicity and order of coefficients in Proposition \ref{Proposition: range}, we can describe the orders of the coefficients in the following lemma. 

\begin{lemma}\label{order of coef}
    Under Condition \ref{condition:d_sigma0_eta}, the following coefficient orders hold for $t\in[T_1,T^*]$:
    \begin{itemize}[leftmargin=*]
        \item $\sum_{i=1}^{n}\zeta_{j,r,i}^{(t)}=\Omega(n)=O(n\log(T^*))$ for any $j\in\{\pm1\}$ and $r\in[m]$.
        \item $\sum_{i=1}^{n}\zeta_{j,r,i}^{(t)}/\gamma_{j',r'}^{(t)}=\Theta(\mathrm{SNR}^{-2})$ for any $j,j'\in\{\pm 1\}$ and $r,r'\in[m]$.
        \item $\max_{j, r, i} | \omega_{j, r, i}^{(t)} | = \max \big\{O\big(\sqrt{\log(mn/\delta)}\cdot\sigma_0\sigma_p\sqrt{d}\big), \\ O \big( \sqrt{\log(n / \delta)} \log(T^*) \cdot n / \sqrt{d} \big) \big\}$.
    \end{itemize}
\end{lemma}

By applying the scale of $\gamma_{j,r}^{(t)},\zeta_{j,r,i}^{(t)},\omega_{j,r,i}^{(t)}$ given in Lemma \ref{order of coef} and Gaussian concentration of Lipschitz function, we can directly get the test error upper bound (the second part of Theorem \ref{thm:signal_learning_main})
using a similar idea as 
the proofs of Theorem 1 in \citet{chatterji2021finite} and Lemma 3 in \citet{frei2022benign}. To prove the test error lower bound (the third part of Theorem \ref{thm:signal_learning_main}), we first lower bound wrong prediction probability term $\mathbb{P} \big(\hat{y} f(\Wb^{(t)},\xb) \leq 0\big)$ by
\begin{equation*}
    0.5\PP\bigg(\underbrace{\Big|\sum_{j,r}j\sigma(\la\wb_{j,r}^{(t)},\bxi\ra)\Big|\geq C_6\max_{j}\Big\{\sum_{r}\gamma_{j,r}^{(t)}\Big\}}_{\text{event } \Omega}\bigg),
\end{equation*}
where all the randomness is on the left-hand side, which can be treated as a function of Gaussian random vector $\bxi$. Next, we give our key lemma, which can be proved by leveraging decomposition of $\wb_{j,r}^{(t)}$ and scale of decomposition coefficients given in Lemma \ref{order of coef}.

\begin{lemma}\label{lemma:g lower bound}
    For $t\in[T_1,T^*]$, denote $g(\bxi)=\sum_{j, r}j\sigma(\la\wb_{j,r}^{(t)},\bxi\ra)$. There exists a fixed vector $\vb$ with $\|\vb\|_{2} \leq 0.06 \sigma_p$ such that
    \begin{equation}
    \sum_{j' \in \{\pm 1\}}[g(j'\bxi + \vb) - g(j'\bxi)] \geq 4C_{6}\max_{j \in \{\pm 1\}}\Big\{\sum_{r}\gamma_{j,r}^{(t)}\Big\},
    \end{equation}
    for all $\bxi \in \RR^{d}$.
\end{lemma}
Based on Lemma \ref{lemma:g lower bound}, by the pigeonhole principle, there must exist 
one among 
$\bxi$, $\bxi + \vb$, $-\bxi$, $-\bxi + \vb$ 
that 
belongs to $\Omega$, that is, $\Omega \cup (-\Omega) \cup (\Omega - \vb) \cup (-\Omega - \vb) = \RR^{d}$. By union bound, it follows that
\begin{equation}\label{omega prob1}
    \PP(\Omega)+\PP(-\Omega)+\PP(\Omega - \vb)+\PP(-\Omega - \vb)\geq1.
\end{equation}
Since the noise $\bxi$ follows symmetric distribution, we have that $\PP(\Omega) = \PP(-\Omega)$. We can use some techniques based on the total variation (TV) distance  
to show that 
\begin{equation}
\begin{aligned}\label{omega prob2}
&|\mathbb{P}(\Omega) - \mathbb{P}(\Omega - \vb)|, |\mathbb{P}(-\Omega) - \mathbb{P}(-\Omega - \vb)| \leq 0.03.
\end{aligned}
\end{equation}
By \eqref{omega prob1} and \eqref{omega prob2}, we have proved that $\mathbb{P}(\Omega) \geq 0.22$. By plugging $\mathbb{P}(\Omega) \geq 0.22$ into \eqref{eq:test error}, we complete the proof of test error lower bound.

\section{Experiments}
In this section, we present simulations of synthetic data to back up our theoretical analysis in the previous section. The code for our experiments can be found on Github \footnote{ \href{https://github.com/uclaml/Benign_ReLU_CNN}{https://github.com/uclaml/Benign\_ReLU\_CNN}}.

\paragraph{Synthetic-data experiments.} 
Here we generate synthetic data exactly following Definition~\ref{def:data}. Specifically, we set training data size $n = 20$ and label-flipping noise to $0.1$. Since the learning problem is rotation-invariant, without loss of generality, we set $\bmu = \| \bmu \|_2 \cdot [1,0,\ldots,0]^\top $. We then generate the noise vector $\bxi$ from the Gaussian distribution $\cN (\mathbf{0}, \sigma_p^2 \Ib)$ with fixed standard deviation $\sigma_{p} = 1$. 

We train a two-layer CNN model defined in Section~\ref{sec:prob} with ReLU activation function. The number of filters is set as $m=10$. We use the default initialization method in PyTorch to initialize the CNN parameters and train the CNN with full-batch gradient descent with a learning rate of $0.1$ for $100$ iterations. We consider different dimensions $d$ ranging from $100$ to $1100$, and different signal strengths $\|\bmu\|_{2}$ ranging from $1$ to $11$. 
Based on our results, for any dimension $d$ and signal strength $\mu$ setting we consider, our training setup can guarantee a training loss smaller than $0.01$. 
After training, we estimate the test error for each case using $1000$ test data points. The results are given as a heatmap on parameters $d$ and $\| \bmu \|_2$ in Figure~\ref{fig:synthetic}.

\begin{figure}[ht!]
	\begin{center}
  \subfigure[Original Test Error Heatmap]{\includegraphics[width=0.37\textwidth]{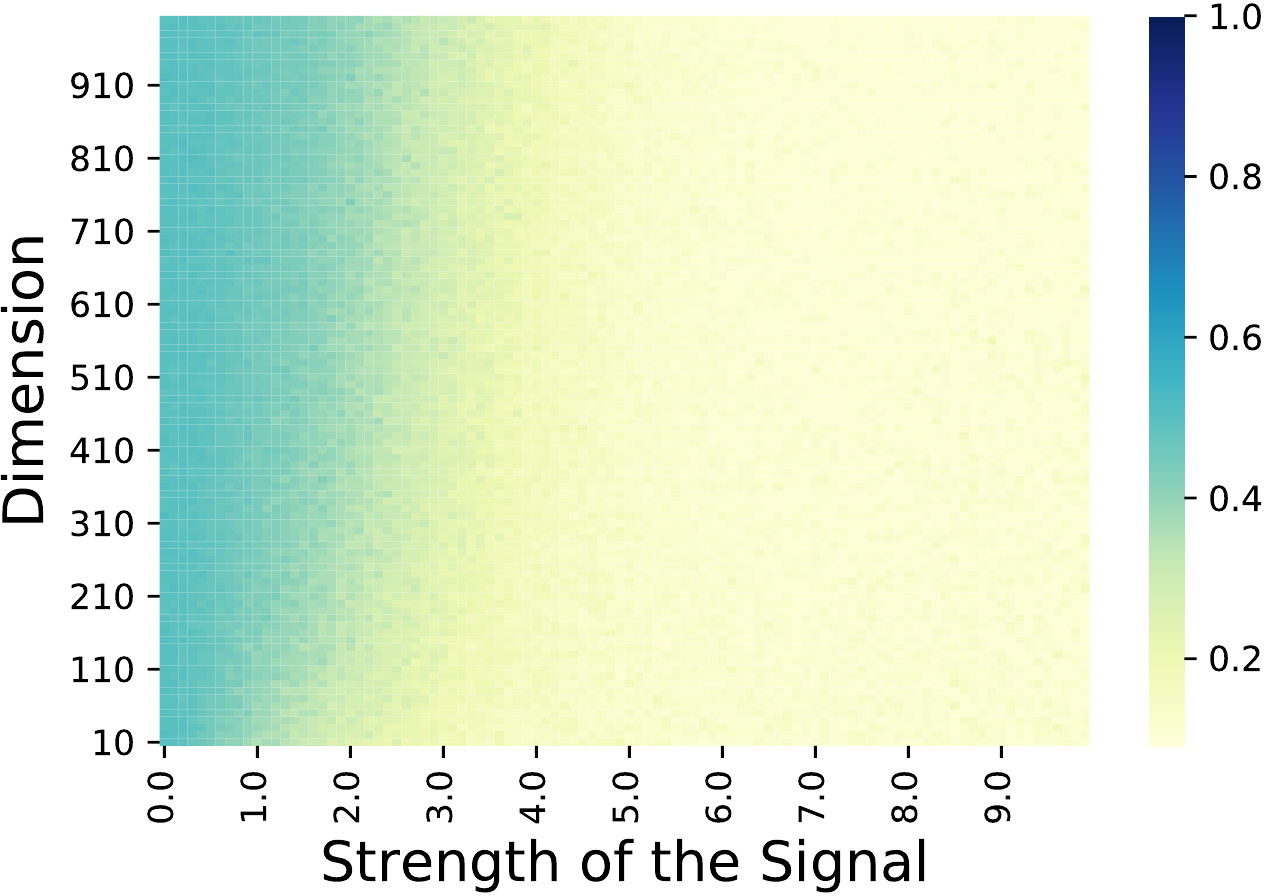}}
\subfigure[Cutoff Test Error Heatmap]{\includegraphics[width=0.37\textwidth]{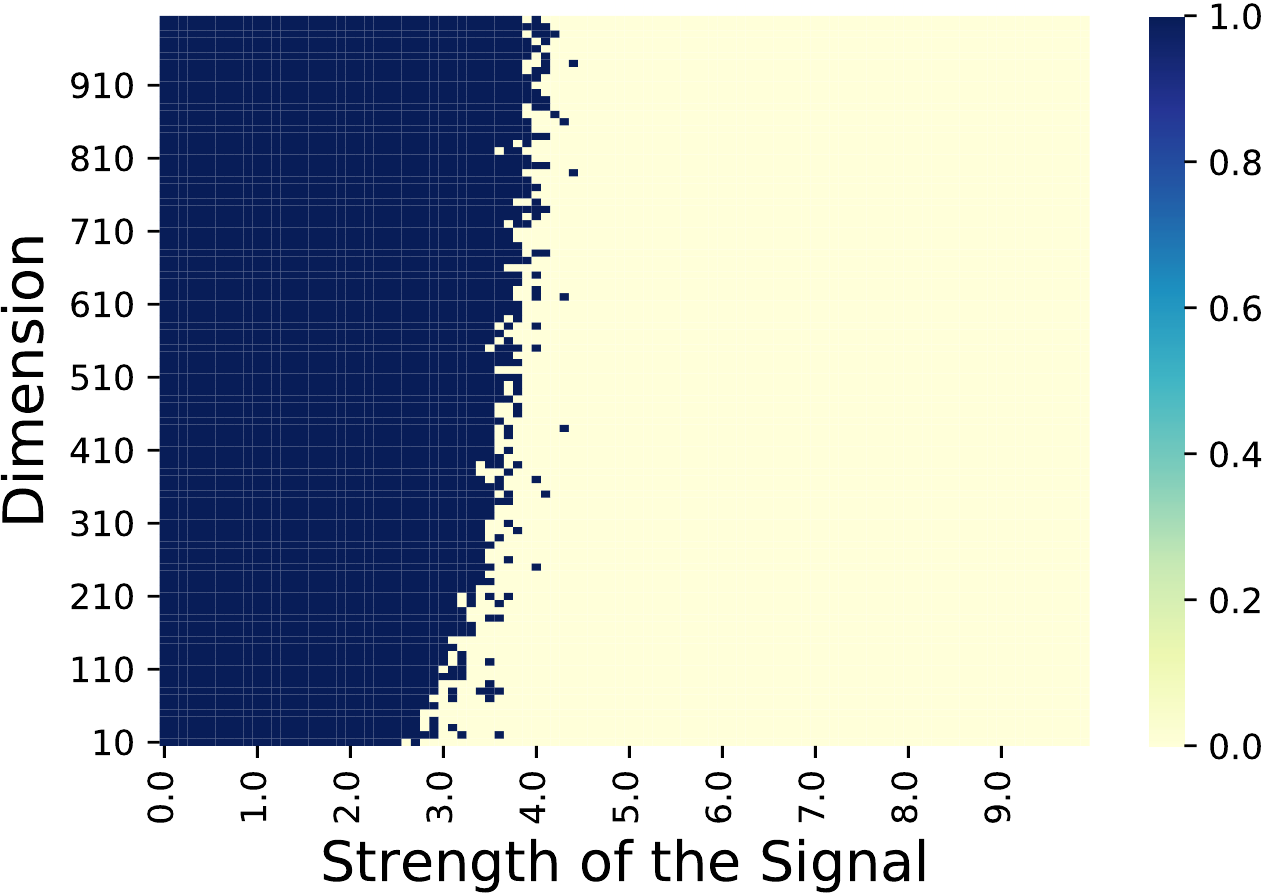}}
	\end{center}
	\caption{a) is a heatmap of test error on synthetic data under different dimensions $d$ and signal strengths $\bmu$. High test errors are marked in blue, and low test errors are marked in yellow. b) is a cutoff value heatmap that sets the values smaller than $0.2$ to be $0$ (yellow) and the values greater than $0.2$ to be $1$ (blue). }
	\label{fig:synthetic}
\end{figure}

For the specific case $\|\bmu\|_{2} = 5$ and $d = 100$, we plot the training loss, test loss, and test error throughout training in Figure~\ref{fig:synthetic3}. As we can see from the figure, the test error reaches the Bayesian optimal error of $0.1$, while the training loss converges to zero. 

\begin{figure}[ht!]
	\begin{center}
			\includegraphics[width=0.75\linewidth,angle=0]{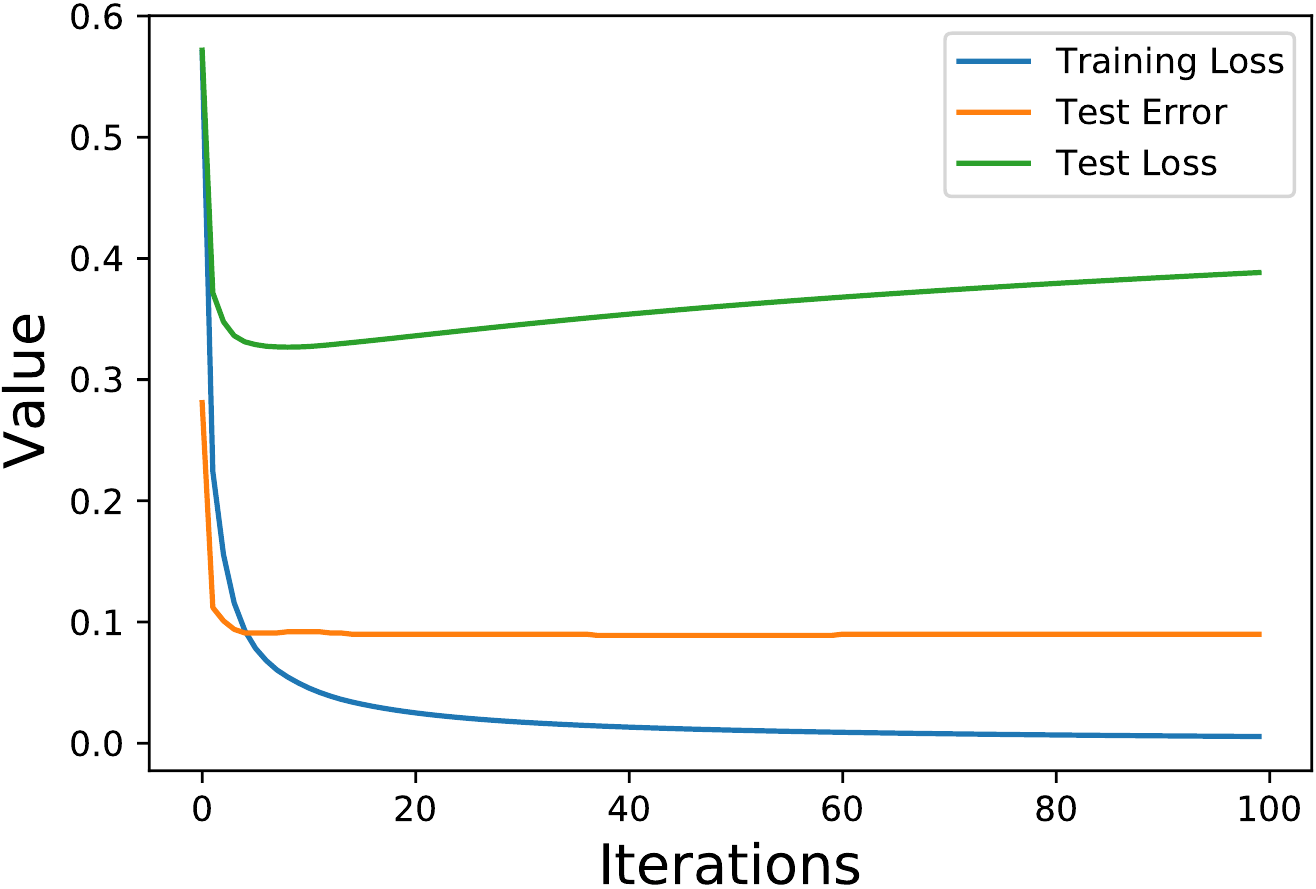}
	\end{center}
	\caption{Training loss, test loss and test error throughout $100$ iterations with $\|\bmu\|_{2} = 5$ and $d = 100$.}
	\label{fig:synthetic3}
\end{figure}

In Section~\ref{sec:proof}, we directly used the
activation pattern and data structure to characterize the
loss of each sample, and proved that $y_i\cdot f(\Wb^{(t)},\xb_i)-y_k\cdot f(\Wb^{(t)},\xb_k)\leq C_4$ for $t \leq T^{*}$ and any $i, k \in [n]$ in Lemma~\ref{logit ratio}. To demonstrate this, we conduct another experiment for the case $\|\bmu\|_{2} = 5$, $d = 100$ and plot $\max  y_i\cdot f(\Wb^{(t)},\xb_i)$ and $\min  y_i\cdot f(\Wb^{(t)},\xb_i)$ (margin) for each iteration. As we can see from Figure~\ref{fig:synthetic4}, the difference between them never grows too large during training (bounded by $6$).

\begin{figure}[ht!]
	\begin{center}
			\includegraphics[width=0.75\linewidth,angle=0]{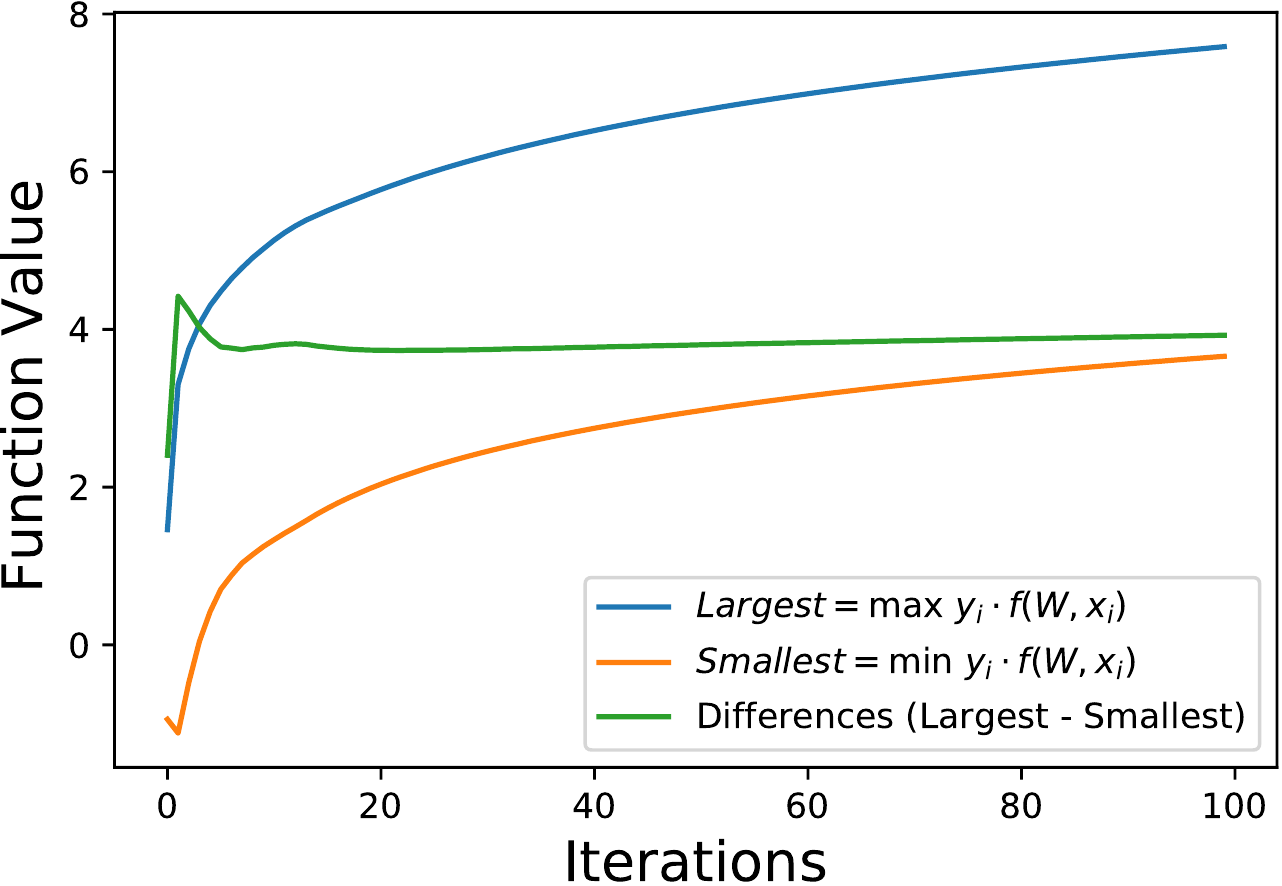}
	\end{center}
	\caption{$\max  y_i\cdot f(\Wb^{(t)},\xb_i)$ and $\min  y_i\cdot f(\Wb^{(t)},\xb_i)$ (margin) throughout $100$ iterations with $\|\bmu\|_{2} = 5$ and $d = 100$.}
	\label{fig:synthetic4}
\end{figure}

\section{Conclusion and Future Work}
This paper studies benign overfitting in two-layer ReLU CNNs with label-flipping noise. We generalize the signal-noise decomposition technique first proposed by \citet{cao2022benign} and propose three key techniques: \emph{time-invariant coefficient ratio}, \emph{automatic balance of coefficient updates} and \emph{algorithm-dependent test error analysis}. With the help of these techniques, 
we prove the convergence of training loss, give exact conditions under which the CNN achieves test error close to the noise rate, and reveal a sharp phase transition between benign and harmful overfitting. 
Our results theoretically demonstrate how and when benign overfitting can happen in ReLU neural networks. An important future work direction is to generalize our analysis to deep ReLU neural networks in learning other data models.

\section*{Acknowledgements}

We thank the anonymous reviewers for their helpful comments. YK, ZC, YC and QG are supported in part by the National Science Foundation CAREER Award 1906169 and IIS-2008981, and the Sloan Research Fellowship. The views and conclusions contained in this paper are those of the authors and should not be interpreted as representing any funding agencies.


\bibliography{deeplearningreferences}
\bibliographystyle{icml2023}

\newpage
\appendix
\onecolumn

\section{Comparison of Conditions Made by Related Works}
In this section, we present the difference between Condition \ref{condition:d_sigma0_eta} and the conditions on parameters made by two related works \citep{frei2022benign,cao2022benign} in the following two tables (Tables~\ref{table: comp with frei} and \ref{table: comp with cao}). 

\begin{table}[ht]
\centering
\begin{tabular}{|c|c|c|}
\hline
\multirow{2}{*}{Number of samples} & \citet{frei2022benign} & $n\geq C\log(1/\delta)$ \\ \cline{2-3}  & Ours & $n\geq C\log(m/\delta)$  \\ \hline
\multirow{2}{*}{Neural network width}  & \citet{frei2022benign} & - \\ \cline{2-3} & Ours & $m\geq C\log(n/\delta)$\\ \hline
\multirow{2}{*}{Dimension} & \citet{frei2022benign} & $d\geq\max\{n\|\bmu\|_2^2,n^2\log(n/\delta)\}$  \\ \cline{2-3} & Ours & $d \geq C \max\{n\sigma_p^{-2}\|\bmu\|_{2}^{2}\log(T^*),n^{2}\log(nm/\delta) (\log (T^{*}))^2\}$ \\ \hline
\multirow{2}{*}{Norm of the signal} & \citet{frei2022benign} & $\|\bmu\|_2^2\geq C\cdot\log(n/\delta)$ \\ \cline{2-3} & Ours & $\|\bmu\|_2^2 \geq C \cdot \sigma_p^2 \log(n / \delta)$\\ \hline
\multirow{2}{*}{Noise rate} & \citet{frei2022benign} & $p\leq 1/C$ \\ \cline{2-3}  & Ours & $p\leq 1/C$ \\ \hline
\multirow{2}{*}{Learning rate}   & \citet{frei2022benign} & $\eta\leq(C\max\{1,H/\sqrt{m}\}d^2)^{-1}$ \\ \cline{2-3} & Ours & $\eta\leq \big(C\max\big\{\sigma_p^2 d/n,\sigma_p^2 d^{3/2}\big/\big(n^2 m\cdot\sqrt{\log(n/\delta)}\big)\big\}\big)^{-1}$ \\ \hline
\multirow{2}{*}{Initialization variance} & \citet{frei2022benign} & $\sigma_0\leq \eta/\sqrt{md}$ \\ \cline{2-3} & Ours & $\sigma_0\leq \big(C\max\big\{\sigma_{p}d/\sqrt{n},\sqrt{\log(m/\delta)}\cdot\|\bmu\|_{2}\big\}\big)^{-1}$ \\ \hline
\end{tabular}
\caption{Comparison of conditions with \citet{frei2022benign}. $H$ is the smoothness of leaky ReLU activation under the setting of \citet{frei2022benign}. In our paper, $\sigma_p$ is the noise scale that can be treated as a constant.}
\label{table: comp with frei}
\end{table}

\begin{table}[ht]
\centering
\begin{tabular}{|c|c|c|}
\hline
\multirow{2}{*}{Number of samples} & \citet{cao2022benign} & $n=\Omega(\polylog(d))$ \\ \cline{2-3} & Ours & $n\geq C\log(m/\delta)$ \\ \hline
\multirow{2}{*}{Neural network width} & \citet{cao2022benign} & $m=\Omega(\polylog(d))$ \\ \cline{2-3} & Ours & $m\geq C\log(n/\delta)$  \\ \hline
\multirow{2}{*}{Dimension} & \citet{cao2022benign} & $d=\Omega(m^{2\lor[4/(q-2)]}n^{4\lor[(2q-2)/(q-2)]})$   \\ \cline{2-3} & Ours & $d \geq C \max\{n\sigma_p^{-2}\|\bmu\|_{2}^{2}\log(T^*),n^{2}\log(nm/\delta) (\log (T^{*}))^2\}$ \\ \hline
\multirow{2}{*}{Norm of the signal}& \citet{cao2022benign} & -  \\ \cline{2-3} & Ours & $\|\bmu\|_2^2 \geq C \cdot \sigma_p^2 \log(n / \delta)$ \\ \hline
\multirow{2}{*}{Noise rate} & \citet{cao2022benign} & $p=0$ \\ \cline{2-3}  & Ours & $p\leq 1/C$ \\ \hline
\multirow{2}{*}{Learning rate}  & \citet{cao2022benign} & $\eta\leq\tilde{O}(\min\{\|\bmu\|_2^{-2},\sigma_p^{-2}d^{-1}\})$ \\ \cline{2-3} & Ours & $\eta\leq \big(C\max\big\{\sigma_p^2 d/n,\sigma_p^2 d^{3/2}\big/\big(n^2 m\cdot\sqrt{\log(n/\delta)}\big)\big\}\big)^{-1}$ \\ \hline
\multirow{2}{*}{Initialization variance}  & \citet{cao2022benign} & \makecell[c]{$\sigma_0\leq\tilde{O}(m^{-2/(q-2)}n^{-[1/(q-2)\lor1]})\cdot\min\{(\sigma_p\sqrt{d})^{-1},\|\bmu\|_2^{-1}\}$,\\
$\sigma_{0}\geq\tilde{O}(nd^{-1/2})\cdot\min\{(\sigma_p\sqrt{d})^{-1},\|\bmu\|_2^{-1}\}$} \\ \cline{2-3} & Ours & $\sigma_0\leq \big(C\max\big\{\sigma_{p}d/\sqrt{n},\sqrt{\log(m/\delta)}\cdot\|\bmu\|_{2}\big\}\big)^{-1}$ \\ \hline
\end{tabular}
\caption{Comparison of conditions with \citet{cao2022benign}. $q$ is the order of polynomial ReLU activation function under the setting of \citet{cao2022benign}.}
\label{table: comp with cao}
\end{table}

\section{Preliminary Lemmas}\label{sec: initial}
In this section, we present some pivotal lemmas that illustrate some important properties of the data and neural network parameters at their random initialization. 

We first give some concentration lemmas regarding the data set $S$. Let $S_{+} = \{i| y_{i} = \hat{y}_{i}\}$ and $S_{-} = \{i| y_{i} \not= \hat{y}_{i}\}$ denote index sets corresponding to data points with true and flipped labels, respectively. We first have the following lemma.


\begin{lemma}
Given $\delta>0$, with probability at least $1-\delta$, 
\begin{equation*}\label{lm: estimate |S+|}
    \big||S_{+}|-(1-p)n\big|\leq\sqrt{\frac{n}{2}\log\Big(\frac{4}{\delta}\Big)},\, \big||S_{-}|-pn\big|\leq\sqrt{\frac{n}{2}\log\Big(\frac{4}{\delta}\Big)}.
\end{equation*}
\end{lemma}
\begin{proof}[Proof of Lemma \ref{lm: estimate |S+|}]Since $|S_{+}|=\sum_{i=1}^{n}\ind[\hat{y}_i=y_i],\,|S_{-}|=\sum_{i=1}^{n}\ind[\hat{y}_i\neq y_i]$, according to Hoeffding's inequality, we have for arbitrary $t>0$ that
\begin{equation*}
    \mathbb{P}\big(\big||S_{+}|-\mathbb{E}[|S_{+}|]\big|\geq t\big)\leq 2\exp\Big(-\frac{2t^2}{n}\Big),\,\mathbb{P}\big(\big||S_{-}|-\mathbb{E}[|S_{-}|]\big|\geq t\big)\leq 2\exp\Big(-\frac{2t^2}{n}\Big).
\end{equation*}
By the data distribution $\cD$ defined in Definition \ref{def:data}, we have $\mathbb{E}[S_{+}]=(1-p)n$, $\mathbb{E}[S_{-}]=pn$. Setting $t=\sqrt{(n/2)\log(4/\delta)}$ and taking a union bound, it follows that with probability at least $1-\delta$, 
\begin{equation*}
    \big||S_{+}|-(1-p)n\big|\leq\sqrt{\frac{n}{2}\log\Big(\frac{4}{\delta}\Big)},\, \big||S_{-}|-pn\big|\leq\sqrt{\frac{n}{2}\log\Big(\frac{4}{\delta}\Big)},
\end{equation*}
which completes the proof.
\end{proof}

Next, let $S_{1}=\{i|y_i=1\}$ and $S_{-1}=\{i|y_i=-1\}$. We have the following lemmas characterizing their sizes. 

\begin{lemma}\label{lm: estimate |S1|}
Suppose that $\delta>0$ and $n\geq 8\log(4/\delta)$. Then with probability at least $1-\delta$, 
\begin{equation*}
    |S_{1}|,|S_{-1}| \in [n/4, 3n/4].
\end{equation*}
\end{lemma}
\begin{proof}[Proof of Lemma \ref{lm: estimate |S1|}]
According the data distribution $\cD$ defined in Definition \ref{def:data}, for $(\xb, y) \sim \cD$, we have 
\begin{align*}
    \PP(y=1) &= \PP(\hat{y} = 1) \times \PP(y = \hat{y}) + \PP(\hat{y} = -1) \times \PP(y = -\hat{y}) \\
    &= \frac{1}{2} (1-p) + \frac{1}{2} p\\
    &= \frac{1}{2}, 
\end{align*}
and hence $\PP(y=-1) = 1/2$ as well. Since $|S_{1}|=\sum_{i=1}^{n}\ind[y_i=1]$, $|S_{-1}|=\sum_{i=1}^{n}\ind[y_i=-1]$, we have $\mathbb{E}[|S_{1}|]=\mathbb{E}[|S_{-1}|]=n/2$. By Hoeffding’s inequality, for arbitrary $t>0$ the following holds: 
\begin{align*}
    &\mathbb{P}\big(\big||S_1|-\mathbb{E}[|S_1|]\big|\geq t\big)\leq2\exp\Big(-\frac{2t^2}{n}\Big),\\
    &\mathbb{P}\big(\big||S_{-1}|-\mathbb{E}[|S_{-1}|]\big|\geq t\big)\leq2\exp\Big(-\frac{2t^2}{n}\Big).
\end{align*}
Setting $t=\sqrt{(n/2)\log(4/\delta)}$ and taking a union bound, it follows that with probability at least $1-\delta$, 
\begin{equation*}
    \Big||S_{1}|-\frac{n}{2}\Big|\leq\sqrt{\frac{n}{2}\log\Big(\frac{4}{\delta}\Big)},\Big||S_{-1}|-\frac{n}{2}\Big|\leq\sqrt{\frac{n}{2}\log\Big(\frac{4}{\delta}\Big)}.
\end{equation*}
Therefore, as long as $n\geq8\log(4/\delta)$, we have $\sqrt{n\log(4/\delta)/2}\leq n/4$ and hence $3n/4\geq |S_{1}|,|S_{-1}|\geq n/4$.

\end{proof}

\begin{lemma}\label{lm: estimate S cap S}
For $|S_{+}\cap S_{y}|$ and $|S_{-}\cap S_{y}|$ where $y\in\{\pm 1\}$, it holds with probability at least $1-\delta(\delta>0)$ that
\begin{equation*}
    \Big||S_{+}\cap S_{y}|-\frac{(1-p)n}{2}\Big|\leq\sqrt{\frac{n}{2}\log\Big(\frac{8}{\delta}\Big)}, \Big||S_{-}\cap S_{y}|-\frac{pn}{2}\Big|\leq\sqrt{\frac{n}{2}\log\Big(\frac{8}{\delta}\Big)}, \forall\, y\in\{\pm 1\}. 
\end{equation*}
\end{lemma}
\begin{proof}
Since $|S_{+}\cap S_{y}|=\sum_{i=1}^{n}\ind[\hat{y}_i=y_i=y]$, $|S_{-}\cap S_{y}|=\sum_{i=1}^{n}\ind[\hat{y}_i\neq y_i, y_i=y]$, according to Hoeffding's inequality, we have
\begin{align*}
    &\mathbb{P}\big(\big||S_{+}\cap S_{y}|-\mathbb{E}[|S_{+}\cap S_{y}|]\big|\geq t\big)\leq 2\exp\Big(-\frac{2t^2}{n}\Big),\,\forall\, y\in\{\pm 1\},\\
    &\mathbb{P}\big(\big||S_{-}\cap S_{y}|-\mathbb{E}[|S_{-}\cap S_{y}|]\big|\geq t\big)\leq 2\exp\Big(-\frac{2t^2}{n}\Big),\,\forall\, y\in\{\pm 1\}.
\end{align*}
According to the definition of $\cD$ in Definition \ref{def:data}, we have $\mathbb{E}[|S_{+}\cap S_{y}|]=(1-p)n/2$, $\mathbb{E}[|S_{-}\cap S_{y}|]=pn/2$. It follows with probability at least $1-\delta$ that
\begin{equation*}
    \Big||S_{+}\cap S_{y}|-\frac{(1-p)n}{2}\Big|\leq\sqrt{\frac{n}{2}\log\Big(\frac{8}{\delta}\Big)}, \Big||S_{-}\cap S_{y}|-\frac{pn}{2}\Big|\leq\sqrt{\frac{n}{2}\log\Big(\frac{8}{\delta}\Big)}, \forall\, y\in\{\pm 1\},
\end{equation*}
which completes the proof.
\end{proof}

The following lemma estimates the norms of the noise vectors $\bxi_i$, $i\in [n]$, and gives an upper bound of their inner products with each other and with the signal vector $\bmu$. 

\begin{lemma}\label{lm: data inner products}
Suppose that $\delta > 0$ and $d = \Omega( \log(6n / \delta) ) $. Then with probability at least $1 - \delta$, 
\begin{align*}
    &\sigma_p^2 d / 2\leq \| \bxi_i \|_2^2 \leq 3\sigma_p^2 d / 2,\\
    & |\la \bxi_i, \bxi_{i'} \ra| \leq 2\sigma_p^2 \cdot \sqrt{d \log(6n^2 / \delta)}, \\
    & |\la\bxi_i,\bmu\ra|\leq \|\bmu\|_2\sigma_p\cdot\sqrt{2\log(6n/\delta)}
\end{align*}
for all $i,i'\in [n]$.
\end{lemma}
\begin{proof}[Proof of Lemma~\ref{lm: data inner products}] By Bernstein's inequality, with probability at least $1 - \delta / (3n)$ we have
\begin{align*}
    \big| \| \bxi_i \|_2^2 - \sigma_p^2 d \big| = O(\sigma_p^2 \cdot \sqrt{d \log(6n / \delta)}).
\end{align*}
Therefore, if we set appropriately $d = \Omega( \log(6n / \delta) )$, we get 
\begin{align*}
     \sigma_p^2 d /2  \leq \| \bxi_i \|_2^2 \leq 3\sigma_p^2 d / 2.
\end{align*}
Moreover, clearly $\la \bxi_i, \bxi_{i'} \ra$ has mean zero. 
For any $i,i'$ with $i\neq i'$, by Bernstein's inequality, with probability at least $1 - \delta / (3n^2)$ we have
\begin{align*}
    | \la \bxi_i, \bxi_{i'} \ra| \leq 2\sigma_p^2 \cdot \sqrt{d \log(6n^2 / \delta)}.
\end{align*}
Finally, note that $\la\bxi_i,\bmu\ra\sim\cN(0,\|\bmu\|_2^2\sigma_p^2)$. By Gaussian tail bounds, with probability at least $1-\delta/3n$ we have
\begin{equation*}
    |\la\bxi_i,\bmu\ra|\leq\|\bmu\|_2\sigma_p\cdot\sqrt{2\log(6n/\delta)}.
\end{equation*}
Applying a union bound completes the proof.
\end{proof}

Now turning to network initialization, the following lemma studies the inner product between a randomly initialized CNN convolutional filter $\wb_{j,r}^{(0)}$ ($j\in \{\pm1\}$ and $r\in [m]$) and the signal/noise vectors in the training data. The calculations characterize how the neural network at initialization randomly captures signal and noise information.

\begin{lemma}\label{lm: initialization inner products} Suppose that $d = \Omega(\log(mn/\delta))$, $ m = \Omega(\log(1 / \delta))$. Then with probability at least $1 - \delta$, 
\begin{align*}
    & \sigma_0^2 d / 2 \leq \| \wb_{j, r}^{(0)} \|_2^2 \leq 3 \sigma_0^2 d / 2, \\
    &|\la \wb_{j,r}^{(0)}, \bmu \ra | \leq \sqrt{2 \log(12 m/\delta)} \cdot \sigma_0 \| \bmu \|_2,\\
    &| \la \wb_{j,r}^{(0)}, \bxi_i \ra | \leq 2\sqrt{ \log(12 mn/\delta)}\cdot \sigma_0 \sigma_p \sqrt{d} 
\end{align*}
for all $r\in [m]$,  $j\in \{\pm 1\}$ and $i\in [n]$. Moreover, 
\begin{align*}
    &\sigma_0 \| \bmu \|_2 / 2 \leq \max_{r\in[m]} j\cdot \la \wb_{j,r}^{(0)}, \bmu \ra \leq \sqrt{2 \log(12 m/\delta)} \cdot \sigma_0 \| \bmu \|_2,\\
    &\sigma_0 \sigma_p \sqrt{d} / 4 \leq \max_{r\in[m]} j\cdot \la \wb_{j,r}^{(0)}, \bxi_i \ra \leq 2\sqrt{ \log(12 mn/\delta)} \cdot \sigma_0 \sigma_p \sqrt{d}
\end{align*}
for all $j\in \{\pm 1\}$ and $i\in [n]$.
\end{lemma}
\begin{proof}[Proof of Lemma~\ref{lm: initialization inner products}]
First of all, the initial weights $\wb_{j,r}^{(0)} \sim \cN (\mathbf{0}, \sigma_0 \Ib)$. 
By Bernstein's inequality, with probability at least $1 - \delta / (6m)$ we have
\begin{align*}
    \big| \| \wb_{j,r}^{(0)} \|_2^2 - \sigma_0^2 d \big| = O(\sigma_0^2 \cdot \sqrt{d \log(12m / \delta)}).
\end{align*}
Therefore, if we set appropriately $d = \Omega( \log(mn / \delta) )$, we have with probability at least $1 - \delta / 3$, for all $j\in \{\pm 1\}$ and $r\in [m]$,  
\begin{align*}
     \sigma_0^2 d /2  \leq \| \wb_{j,r}^{(0)} \|_2^2 \leq 3\sigma_0^2 d / 2.
\end{align*}

Next, it is clear that for each $r\in [m]$, $j\cdot \la \wb_{j,r}^{(0)}, \bmu \ra$ is a Gaussian random variable with mean zero and variance $\sigma_0^2 \| \bmu \|_2^2$. Therefore, by Gaussian tail bound and union bound, with probability at least $1 - \delta/6$, for all $j\in \{\pm 1\}$ and $r\in [m]$, 
\begin{align*}
    j\cdot \la \wb_{j,r}^{(0)}, \bmu \ra \leq |\la \wb_{j,r}^{(0)}, \bmu \ra| \leq \sqrt{2\log(12m/\delta)} \cdot \sigma_0 \| \bmu \|_2.
\end{align*}
Moreover, $\PP( \sigma_0 \| \bmu \|_2 / 2 > j\cdot \la \wb_{j,r}^{(0)}, \bmu \ra )$ is an absolute constant, and therefore with the condition $m = \Omega (\log(1 / \delta))$, we have
\begin{align*}
    \PP\big( \sigma_0 \| \bmu \|_2 / 2 \leq \max_{r\in[m]} j\cdot \la \wb_{j,r}^{(0)}, \bmu \ra ) &= 1 - \PP( \sigma_0 \| \bmu \|_2 / 2 > \max_{r\in[m]} j\cdot \la \wb_{j,r}^{(0)}, \bmu \ra \big) \\
    &= 1 - \PP\big( \sigma_0 \| \bmu \|_2 / 2 > j\cdot \la \wb_{j,r}^{(0)}, \bmu \ra \big)^{2m} \\
    &\geq 1 - \delta / 6,
\end{align*}
hence with probability at least $1 - \delta / 3$, we have $\sigma_0 \| \bmu \|_2 / 2 \leq \max_{r\in[m]} j\cdot \la \wb_{j,r}^{(0)}, \bmu \ra \leq \sqrt{2 \log(12 m/\delta)} \cdot \sigma_0 \| \bmu \|_2$. 

Finally, under the results of Lemma~\ref{lm: data inner products}, we have $ \sigma_p \sqrt{d} / \sqrt{2} \leq \| \bxi_i \|_2 \leq \sqrt{3/2}\cdot \sigma_p \sqrt{d}$ for all $i\in [n]$. Therefore, we can get the result for $\la \wb_{j,r}^{(0)}, \bxi_i\ra$ with probability at least $1 - \delta / 3$, following the same proof outline as $j\cdot \la \wb_{j,r}^{(0)}, \bmu \ra$. 
\end{proof}

Next, we denote $S_{i}^{(0)}$ as $\{ r\in[m]:\la\wb_{y_i,r}^{(0)},\bxi_i\ra>0\}$ and $S_{j,r}^{(t)}$ as $\{i\in[n]:y_i=j,\la\wb_{j,r}^{(t)},\bxi_i\ra>0\}$, $j\in\{\pm 1\}$, $r\in[m]$. We give a lower bound of $|S_{i}^{(0)}|$ and $|S_{j,r}^{(0)}|$ in the following two lemmas. 
\begin{lemma}\label{lm: number of initial activated neurons}
Suppose that $\delta>0$ and $m\geq50\log(2n/\delta)$. Then with probability at least $1-\delta$, 
\begin{equation*}
    |S_{i}^{(0)}|\geq 0.4m,\, \forall i\in[n].
\end{equation*}
\end{lemma}
\begin{proof}[Proof of Lemma \ref{lm: number of initial activated neurons}]
Note that $|S_{i}^{(0)}|=\sum_{r=1}^{m}\ind[\la\wb_{y_i,r}^{(0)},\bxi_i\ra>0]$ and $P(\la\wb_{y_i,r}^{(0)},\bxi_i\ra>0)=1/2$, then by Hoeffding’s inequality, with probability at least $1-\delta/n$, we have
\begin{equation*}
    \bigg| \frac{|S_{i}^{(0)}|} {m} - \frac{1}{2}\bigg|  \leq\sqrt{\frac{\log(2n/\delta)}{2m}}.
\end{equation*}
Therefore, as long as $m\geq50\log(2n/\delta)$, by applying union bound, with probability at least $1-\delta$, we have
\begin{equation*}
    |S_{i}^{(0)}|\geq0.4m,\, \forall i\in[n].
\end{equation*}
\end{proof}

\begin{lemma}\label{lm: number of initial activated neurons 2}
Suppose that $\delta>0$ and $n\geq32\log(4m/\delta)$. Then with probability at least $1-\delta$,
\begin{equation*}
    |S_{j,r}^{(0)}|\geq n/8,\,\forall j\in\{\pm 1\}, r\in[m].
\end{equation*}
\end{lemma}

\begin{proof}[Proof of Lemma \ref{lm: number of initial activated neurons 2}]
Note that $|S_{j,r}^{(0)}|=\sum_{i=1}^{n}\ind[y_i=j]\ind[\la\wb_{j,r}^{(0)},\bxi_i\ra>0]$ and $\mathbb{P}(y_i=j,\la\wb_{j,r}^{(0)},\bxi_i\ra>0)=1/4$, then by Hoeffding's inequality, with probability at least $1-\delta/2m$, we have
\begin{equation*}
    \big||S_{j,r}^{(0)}|/n-1/4\big|\leq\sqrt{\frac{\log(4m/\delta)}{2n}}.
\end{equation*}
Therefore, as long as $n\geq 32\log(4m/\delta)$, by applying union bound, we have with probability at least $1 - \delta$, 
\begin{equation*}
    |S_{j,r}^{(0)}|\geq n/8,\,\forall j\in\{\pm 1\}, r\in[m].
\end{equation*}
\end{proof}

\section{Signal-noise Decomposition Coefficient Analysis}\label{section:decompositionproof}

In this section, we establish a series of results on the signal-noise decomposition. These results are based on the conclusions in Appendix~\ref{sec: initial}, which hold with high probability. Denote by  $\cE_{\mathrm{prelim}}$ the event that all the results in Appendix~\ref{sec: initial} hold (for a given $\delta$, we see $\PP (\cE_{\mathrm{prelim}}) \geq 1 - 7\delta$ by a union bound). For simplicity and clarity, we state all the results in this and the following sections conditional on $\cE_{\mathrm{prelim}}$. 

\subsection{Iterative Expression for Decomposition Coefficients}

We begin by analyzing the coefficients in the signal-noise decomposition in Definition~\ref{def:w_decomposition}. The first lemma presents an iterative expression for the coefficients. 

\begin{lemma}\label{lm: coefficient iterative} (Restatement of Lemma~\ref{lemma:coefficient_iterative_proof})
The coefficients $\gamma_{j,r}^{(t)},\zeta_{j,r,i}^{(t)},\omega_{j,r,i}^{(t)}$ defined in Definition~\ref{def:w_decomposition} satisfy the following iterative equations:
\begin{align*}
    &\gamma_{j,r}^{(0)}, \zeta_{j,r,i}^{(0)}, \omega_{j,r,i}^{(0)} = 0,\\
    &\gamma_{j,r}^{(t+1)} = \gamma_{j,r}^{(t)} - \frac{\eta}{nm} \cdot \bigg[\sum_{i\in S_{+}} \ell_i'^{(t)} \sigma'(\la\wb_{j,r}^{(t)}, \hat{y}_{i} \cdot \bmu\ra) - \sum_{i\in S_{-}} \ell_i'^{(t)} \sigma'(\la\wb_{j,r}^{(t)}, \hat{y}_{i} \cdot \bmu\ra)\bigg]  \cdot \| \bmu \|_2^2, \\
    &\zeta_{j,r,i}^{(t+1)} = \zeta_{j,r,i}^{(t)} - \frac{\eta}{nm} \cdot \ell_i'^{(t)}\cdot \sigma'(\la\wb_{j,r}^{(t)}, \bxi_{i}\ra) \cdot \| \bxi_i \|_2^2 \cdot \ind(y_{i} = j), \\
    &\omega_{j,r,i}^{(t+1)} = \omega_{j,r,i}^{(t)} + \frac{\eta}{nm} \cdot \ell_i'^{(t)}\cdot \sigma'(\la\wb_{j,r}^{(t)}, \bxi_{i}\ra) \cdot \| \bxi_i \|_2^2 \cdot \ind(y_{i} = -j),
\end{align*}
for all $r\in [m]$,  $j\in \{\pm 1\}$ and $i\in [n]$.
\end{lemma}
\begin{proof}[Proof of Lemma \ref{lm: coefficient iterative}]

First, we iterate the gradient descent update rule \eqref{eq:gdupdate} $t$ times and get
\begin{align*}
    \wb_{j,r}^{(t+1)} &=\wb_{j,r}^{(0)}- \frac{\eta}{nm}\sum_{s=0}^{t}\sum_{i=1}^n\ell_i'^{(s)}\cdot  \sigma'(\la\wb_{j,r}^{(s)}, \bxi_{i}\ra)\cdot j y_{i}\bxi_{i}\\
    &\quad\quad-\frac{\eta}{nm}\sum_{s=0}^{t}\sum_{i=1}^n \ell_i'^{(s)} \cdot \sigma'(\la\wb_{j,r}^{(s)}, \hat{y}_{i} \bmu\ra)\cdot \hat{y}_{i}y_i j\bmu.
\end{align*}
According to the definition of $\gamma_{j,r}^{(t)}$ and $\rho_{j,r,i}^{(t)}$, 
\begin{align*}
    \wb_{j,r}^{(t)} = \wb_{j,r}^{(0)} + j \cdot \gamma_{j,r}^{(t)} \cdot \| \bmu \|_2^{-2} \cdot \bmu + \sum_{ i = 1}^n \rho_{j,r,i}^{(t) }\cdot \| \bxi_i \|_2^{-2} \cdot \bxi_{i}.
\end{align*}
Note that $\bxi_i$ and $\bmu$ are linearly independent with probability $1$, under which condition we have the unique representation
\begin{align*}
    & \gamma_{j,r}^{(t)} = -\frac{\eta}{nm}\sum_{s=0}^{t}\sum_{i=1}^n \ell_i'^{(s)} \cdot \sigma'(\la\wb_{j,r}^{(s)}, \hat{y}_{i} \bmu\ra)\cdot\| \bmu \|_2^{2}\cdot \hat{y}_{i}y_i,\\
    & \rho_{j,r,i}^{(t)} = -\frac{\eta}{nm}\sum_{s=0}^{t}\ell_i'^{(s)}\cdot  \sigma'(\la\wb_{j,r}^{(s)}, \bxi_{i}\ra)\cdot\|\bxi_i\|_{2}^{2}\cdot j y_{i}.
\end{align*}
Recall $S_{+}=\{i|y_i=\hat{y}_i\}$, $S_{-}=\{i|y_i\neq \hat{y}_i\}$, we can further write
\begin{equation}
    \gamma_{j,r}^{(t)}=-\frac{\eta}{nm}\sum_{s=0}^{t}\sum_{i\in S_{+}} \ell_i'^{(s)} \cdot \sigma'(\la\wb_{j,r}^{(s)}, \hat{y}_{i} \bmu\ra)\cdot\| \bmu \|_2^{2}+\frac{\eta}{nm}\sum_{s=0}^{t}\sum_{i\in S_{-}} \ell_i'^{(s)} \cdot \sigma'(\la\wb_{j,r}^{(s)}, \hat{y}_{i} \bmu\ra)\cdot\| \bmu \|_2^{2}.\label{eq: gamma update rule}
\end{equation}
Now with the notation $\zeta_{j,r,i}^{(t)} := \rho_{j,r,i}^{(t)}\ind(\rho_{j,r,i}^{(t)} \geq 0)$, $\omega_{j,r,i}^{(t)} := \rho_{j,r,i}^{(t)}\ind(\rho_{j,r,i}^{(t)} \leq 0)$ and the fact $\ell_i'^{(s)}<0$, we get 
\begin{align}
    \zeta_{j,r,i}^{(t)}&=-\frac{\eta}{nm}\sum_{s=0}^{t}\ell_i'^{(s)}\cdot  \sigma'(\la\wb_{j,r}^{(s)}, \bxi_{i}\ra)\cdot\|\bxi_i\|_{2}^{2}\cdot\ind(y_i=j),\label{eq: zeta update rule} \\
    \omega_{j,r,i}^{(t)}&=\frac{\eta}{nm}\sum_{s=0}^{t}\ell_i'^{(s)}\cdot  \sigma'(\la\wb_{j,r}^{(s)}, \bxi_{i}\ra)\cdot\|\bxi_i\|_{2}^{2}\cdot\ind(y_i=-j).\label{eq: omega update rule}
\end{align}
Writing out the iterative versions of \eqref{eq: gamma update rule}, \eqref{eq: zeta update rule} and \eqref{eq: omega update rule} completes the proof.
\end{proof}

\subsection{Scale of Decomposition Coefficients}
The rest of this section will be dedicated to the proof of the following Proposition~\ref{proposition: range of gamma,zeta,omega}, which shows that the coefficients in the signal-noise decomposition will stay within a reasonable range for a considerable amount of time. Consider the training period $0 \leq t \leq T^*$, where $T^* = \eta^{-1} \poly(\epsilon^{-1}, d, n,m)$, as defined in Theorem~\ref{thm:signal_learning_main}, is the maximum admissible iteration. Now denote
\begin{align}
    & \alpha := 4 \log(T^*), \label{def: alpha} \\
    & \beta := 2 \max_{i, j, r} \{ | \la \wb_{j, r}^{(0)}, \bmu \ra |, | \la \wb_{j, r}^{(0)}, \bxi_i \ra | \}, \label{def: beta}\\
    & \mathrm{SNR} := \|\bmu\|_2/(\sigma_p\sqrt{d}). \label{def: SNR}
\end{align}
By Lemma~\ref{lm: initialization inner products}, $\beta$ can be bounded by $4\sigma_0\cdot \max \{ \sqrt{\log (12mn / \delta)} \cdot\sigma_p \sqrt{d}, \sqrt{\log(12m/\delta)}\cdot\| \bmu \|_2 \}$. Then, by Condition \ref{condition:d_sigma0_eta}, by choosing a large constant $C$, it is straightforward to verify the following inequality:
\begin{equation}
    \max \bigg\{ \beta, \mathrm{SNR} \sqrt{\frac{32 \log(6n / \delta)}{d}} n \alpha, 5 \sqrt{ \frac{\log (6 n^2 / \delta)} {d}} n \alpha \bigg\} \leq \frac{1}{12}. \label{ineq: alpha beta upper bound}
\end{equation}

\begin{proposition} \label{proposition: range of gamma,zeta,omega} (Partial restatement of Proposition~\ref{Proposition: range})
Under Condition \ref{condition:d_sigma0_eta}, for $0\leq t\leq T^*$, we have that
\begin{align}
    &\gamma_{j,r}^{(0)},\zeta_{j,r,i}^{(0)},\omega_{j,r,i}^{(0)}=0\label{initial gamma,zeta,omega}\\
    &0\leq \zeta_{j,r,i}^{(t)}\leq\alpha,\label{ineq: range of zeta}\\
    &0\geq\omega_{j,r,i}^{(t)}\geq-\beta-10\sqrt{\frac{\log(6n^2/\delta)}{d}} n\alpha\geq-\alpha,\label{ineq: range of omega}
\end{align}
and there exists a positive constant $C'$ such that
\begin{equation}
    0\leq \gamma_{j,r}^{(t)}\leq C'\hat{\gamma}\alpha,\label{ineq: range of gamma}
\end{equation}
for all $r\in[m]$, $j\in\{\pm 1\}$ and $i\in[n]$, where $\hat{\gamma}:=n\cdot\mathrm{SNR}^2$. Besides, $\gamma_{j,r}^{(t)}$ is non-decreasing for $0\leq t\leq T^*$.
\end{proposition}
We will use induction to prove Proposition \ref{proposition: range of gamma,zeta,omega}. We first introduce several technical lemmas (Lemmas \ref{lm: inner product range}, \ref{lm: mismatch fj upper bound} and \ref{lm: noise alignment lower bound}) that will be used for the inductive proof of Proposition \ref{proposition: range of gamma,zeta,omega}. 

\begin{lemma}\label{lm: inner product range}
Under Condition~\ref{condition:d_sigma0_eta}, suppose \eqref{ineq: range of zeta}, \eqref{ineq: range of omega} and \eqref{ineq: range of gamma} hold at iteration $t$. Then, for all $r\in[m]$, $j\in\{\pm1\}$ and $i\in[n]$, 
\begin{align}
    &\big|\la\wb_{j,r}^{(t)}-\wb_{j,r}^{(0)},\bmu\ra-j\cdot\gamma_{j,r}^{(t)}\big|\leq\mathrm{SNR}\sqrt{\frac{32\log(6n/\delta)}{d}}n\alpha,\label{ineq: feature product bound}\\
    &\big|\la\wb_{j,r}^{(t)}-\wb_{j,r}^{(0)},\bxi_{i}\ra-\omega_{j,r,i}^{(t)}\big|\leq 5\sqrt{\frac{\log(6n^2/\delta)}{d}}n\alpha,\,j\neq y_i,\label{ineq: noise product bound1}\\
    &\big|\la\wb_{j,r}^{(t)}-\wb_{j,r}^{(0)},\bxi_{i}\ra-\zeta_{j,r,i}^{(t)}\big|\leq 5\sqrt{\frac{\log(6n^2/\delta)}{d}}n\alpha,\,j= y_i. \label{ineq: noise product bound2}
\end{align}
\end{lemma}


\begin{proof}[Proof of Lemma \ref{lm: inner product range}]
First, for any time $t\geq 0$, we have from the signal-noise decomposition \eqref{eq:w_decomposition} that
\begin{align*}
    \la\wb_{j,r}^{(t)}-\wb_{j,r}^{(0)},\bmu\ra&=j\cdot\gamma_{j,r}^{(t)}+\sum_{i'=1}^{n}\zeta_{j,r,i'}^{(t)}\|\bxi_{i'}\|_{2}^{-2}\cdot\la\bxi_{i'},\bmu\ra+\sum_{i'=1}^{n}\omega_{j,r,i'}^{(t)}\|\bxi_{i'}\|_{2}^{-2}\cdot\la\bxi_{i'},\bmu\ra
\end{align*}
According to Lemma \ref{lm: data inner products}, we have
\begin{align*}
    &\quad \Bigg| \sum_{i'=1}^{n}\zeta_{j,r,i'}^{(t)}\|\bxi_{i'}\|_{2}^{-2}\cdot\la\bxi_{i'},\bmu\ra+\sum_{i'=1}^{n}\omega_{j,r,i'}^{(t)}\|\bxi_{i'}\|_{2}^{-2}\cdot\la\bxi_{i'},\bmu\ra\Bigg|\\
    &\leq\sum_{i'=1}^{n}|\zeta_{j,r,i'}^{(t)}|\|\bxi_{i'}\|_{2}^{-2}\cdot|\la\bxi_{i'},\bmu\ra|+\sum_{i'=1}^{n}|\omega_{j,r,i'}^{(t)}|\|\bxi_{i'}\|_{2}^{-2}\cdot|\la\bxi_{i'},\bmu\ra|\\
    &\leq\frac{2\|\bmu\|_2\sqrt{2\log(6n/\delta)}}{\sigma_p d}\bigg(\sum_{i'=1}^{n}|\zeta_{j,r,i'}^{(t)}|+\sum_{i'=1}^{n}|\omega_{j,r,i'}^{(t)}|\bigg)\\
    &=\mathrm{SNR}\sqrt{\frac{8\log(6n/\delta)}{d}}\bigg(\sum_{i'=1}^{n}|\zeta_{j,r,i'}^{(t)}|+\sum_{i'=1}^{n}|\omega_{j,r,i'}^{(t)}|\bigg)\\
    &\leq\mathrm{SNR}\sqrt{\frac{32\log(6n/\delta)}{d}}n\alpha,
\end{align*}
where the first inequality is by triangle inequality, the second inequality is by Lemma 
\ref{lm: data inner products}, the equality is by the definition of $\mathrm{SNR}=\|\bmu\|_2/ (\sigma_p \sqrt{d})$, and the last inequality is by \eqref{ineq: range of zeta}, \eqref{ineq: range of omega}. It follows that
\begin{equation*}
    \big|\la\wb_{j,r}^{(t)}-\wb_{j,r}^{(0)},\bmu\ra-j\cdot\gamma_{j,r}^{(t)}\big|\leq\mathrm{SNR}\sqrt{\frac{32\log(6n/\delta)}{d}}n\alpha.
\end{equation*}

Second, for $j\neq y_i$ and any $t\geq 0$, we have $\zeta_{j,r,i}^{(t)}=0$, and so
\begin{align*}
    \la\wb_{j,r}^{(t)}-\wb_{j,r}^{(0)},\bxi_{i}\ra&=j\cdot\gamma_{j,r}^{(t)}\|\bmu\|_2^{-2}\cdot\la\bmu,\bxi_{i}\ra+\sum_{i'=1}^{n}\zeta_{j,r,i'}^{(t)}\|\bxi_{i'}\|_{2}^{-2}\cdot\la\bxi_{i'},\bxi_{i}\ra+\sum_{i'=1}^{n}\omega_{j,r,i'}^{(t)}\|\bxi_{i'}\|_{2}^{-2}\cdot\la\bxi_{i'},\bxi_{i}\ra\\
    &=\omega_{j,r,i}^{(t)}+j\cdot\gamma_{j,r}^{(t)}\|\bmu\|_2^{-2}\cdot\la\bmu,\bxi_{i}\ra+\sum_{i'\neq i}\omega_{j,r,i'}^{(t)}\|\bxi_{i'}\|_{2}^{-2}\cdot\la\bxi_{i'},\bxi_{i}\ra+\sum_{i'\neq i}\zeta_{j,r,i'}^{(t)}\|\bxi_{i'}\|_{2}^{-2}\cdot\la\bxi_{i'},\bxi_{i}\ra.
\end{align*}
Now we look at 
\begin{align*}
    &\quad \bigg| j\cdot\gamma_{j,r}^{(t)}\|\bmu\|_2^{-2}\cdot\la\bmu,\bxi_{i}\ra+\sum_{i'\neq i}\omega_{j,r,i'}^{(t)}\|\bxi_{i'}\|_{2}^{-2}\cdot\la\bxi_{i'},\bxi_{i}\ra+\sum_{i'\neq i}\zeta_{j,r,i'}^{(t)}\|\bxi_{i'}\|_{2}^{-2}\cdot\la\bxi_{i'},\bxi_{i}\ra\bigg|\\
    &\leq\gamma_{j,r}^{(t)}\|\bmu\|_2^{-2}\cdot|\la\bmu,\bxi_{i}\ra|+\sum_{i'\neq i}(|\omega_{j,r,i'}^{(t)}|+|\zeta_{j,r,i'}^{(t)}|)\|\bxi_{i'}\|_{2}^{-2}\cdot|\la\bxi_{i'},\bxi_{i}\ra|\\
    &\leq\gamma_{j,r}^{(t)}\|\bmu\|_2^{-1}\sigma_p\sqrt{2\log(6n/\delta)}+4\sqrt{\frac{\log(6n^2/\delta)}{d}}\bigg(\sum_{i'\neq i}|\zeta_{j,r,i'}^{(t)}|+\sum_{i'\neq i}|\omega_{j,r,i'}^{(t)}|\bigg)\\
    &=\mathrm{SNR}^{-1}\sqrt{\frac{2\log(6n/\delta)}{d}}\gamma_{j,r}^{(t)}+4\sqrt{\frac{\log(6n^2/\delta)}{d}}\bigg(\sum_{i'\neq i}|\zeta_{j,r,i'}^{(t)}|+\sum_{i'\neq i}|\omega_{j,r,i'}^{(t)}|\bigg)\\
    &\leq \mathrm{SNR}\sqrt{\frac{8C'^2\log(6n/\delta)}{d}}n\alpha+4\sqrt{\frac{\log(6n^2/\delta)}{d}}n\alpha\\
    &\leq 5\sqrt{\frac{\log(6n^2/\delta)}{d}}n\alpha, 
\end{align*}
where the first inequality is by triangle inequality and $\gamma_{j,r}^{(t)}\geq0$; the second inequality is by Lemma \ref{lm: data inner products}; the equality is by the definition of $\mathrm{SNR} = \|\bmu\|_2 / \sigma_p \sqrt{d}$; the second last inequality is by \eqref{ineq: range of omega} and \eqref{ineq: range of gamma}; the last inequality is by $\mathrm{SNR} \leq 1/\sqrt{8C'^{2}}$. It follows that for $j\neq y_i$
\begin{equation*}
    \big|\la\wb_{j,r}^{(t)}-\wb_{j,r}^{(0)},\bxi_{i}\ra-\omega_{j,r,i}^{(t)}\big|\leq 5\sqrt{\frac{\log(6n^2/\delta)}{d}}n\alpha.
\end{equation*}

Similarly, for $y_i=j$, we have that $\omega_{j,r,i}^{(t)}=0$ and 
\begin{align*}
    \la \wb_{j,r}^{(t)} - \wb_{j,r}^{(0)}, \bxi_{i} \ra
    &= j \cdot \gamma_{j,r}^{(t)} \| \bmu \|_2^{-2} \cdot \la \bmu, \bxi_{i} \ra + \sum_{i' = 1}^{n} \zeta_{j,r,i'}^{(t)} \|\bxi_{i'}\|_{2}^{-2} \cdot \la \bxi_{i'}, \bxi_{i} \ra + \sum_{i' = 1}^{n} \omega_{j, r, i'}^{(t)} \| \bxi_{i'} \|_{2}^{-2} \cdot \la \bxi_{i'}, \bxi_{i} \ra \\
    &=\zeta_{j,r,i}^{(t)}+j\cdot\gamma_{j,r}^{(t)}\|\bmu\|_2^{-2}\cdot\la\bmu,\bxi_{i}\ra+\sum_{i'\neq i}\zeta_{j,r,i'}^{(t)}\|\bxi_{i'}\|_{2}^{-2}\cdot\la\bxi_{i'},\bxi_{i}\ra+\sum_{i'\neq i} \omega_{j,r,i'}^{(t)}\|\bxi_{i'}\|_{2}^{-2}\cdot\la\bxi_{i'},\bxi_{i}\ra,
\end{align*}
and also
\begin{align*}
    &\bigg|j\cdot\gamma_{j,r}^{(t)}\|\bmu\|_2^{-2}\cdot\la\bmu,\bxi_{i}\ra+\sum_{i'\neq i}\zeta_{j,r,i'}^{(t)}\|\bxi_{i'}\|_{2}^{-2}\cdot\la\bxi_{i'},\bxi_{i}\ra+\sum_{i'\neq i}\omega_{j,r,i'}^{(t)}\|\bxi_{i'}\|_{2}^{-2}\cdot\la\bxi_{i'},\bxi_{i}\ra\bigg|\\
    &\leq\mathrm{SNR}^{-1}\sqrt{\frac{2\log(6n/\delta)}{d}}\gamma_{j,r}^{(t)}+4\sqrt{\frac{\log(6n^2/\delta)}{d}}\bigg(\sum_{i'\neq i}|\zeta_{j,r,i'}^{(t)}|+\sum_{i'\neq i}|\omega_{j,r,i'}^{(t)}|\bigg)\\
    &\leq \mathrm{SNR}\sqrt{\frac{8C'^2\log(6n/\delta)}{d}}n\alpha+4\sqrt{\frac{\log(6n^2/\delta)}{d}}n\alpha\\
    &\leq 5\sqrt{\frac{\log(6n^2/\delta)}{d}}n\alpha,
\end{align*}
where the second last inequality is by \eqref{ineq: range of zeta}, \eqref{ineq: range of gamma}; the last inequality is by $\mathrm{SNR}\leq 1/\sqrt{8C'^{2}}$. It follows that for $j=y_i$
\begin{equation*}
    \big|\la\wb_{j,r}^{(t)}-\wb_{j,r}^{(0)},\bxi_{i}\ra-\zeta_{j,r,i}^{(t)}\big|\leq 5\sqrt{\frac{\log(6n^2/\delta)}{d}}n\alpha,
\end{equation*}
which completes the proof.
\end{proof}

\begin{lemma}\label{lm: mismatch fj upper bound}
Under Condition~\ref{condition:d_sigma0_eta}, suppose \eqref{ineq: range of zeta}, \eqref{ineq: range of omega} and \eqref{ineq: range of gamma} hold at iteration $t$. Then, for all $j\neq y_i$, $j\in\{\pm 1\}$ and $i\in[n]$, $F_j(\Wb_j^{(t)},\xb_i) \leq 0.5$. 
\end{lemma}
\begin{proof}[Proof of Lemma \ref{lm: mismatch fj upper bound}]
According to Lemma \ref{lm: inner product range}, we have
\begin{align*}
    F_j(\Wb_j^{(t)},\xb_i)&=\frac{1}{m}\sum_{r=1}^{m}[\sigma(\la\wb_{j,r}^{(t)},\hat{y}_i\bmu\ra)+\sigma(\la\wb_{j,r}^{(t)},\bxi_i\ra)]\\
    &\leq5\max\bigg\{|\la\wb_{j,r}^{(0)},\hat{y}_i\bmu\ra|,|\la\wb_{j,r}^{(0)},\bxi_i\ra|,\mathrm{SNR}\sqrt{\frac{32\log(6n/\delta)}{d}}n\alpha,5\sqrt{\frac{\log(6n^2/\delta)}{d}}n\alpha,C'\hat{\gamma}\alpha\bigg\}\\
    &\leq5\max\bigg\{\beta,\mathrm{SNR}\sqrt{\frac{32\log(6n/\delta)}{d}}n\alpha,5\sqrt{\frac{\log(6n^2/\delta)}{d}}n\alpha,C'\hat{\gamma}\alpha\bigg\}\\
    &< 0.5,
\end{align*}
where the first inequality is by \eqref{ineq: feature product bound}, \eqref{ineq: noise product bound1} and \eqref{ineq: noise product bound2}; the second inequality is due to the definition of $\beta$; the third inequality is by \eqref{ineq: alpha beta upper bound}.
\end{proof}

\begin{lemma}\label{lm: noise alignment lower bound}
Under Condition~\ref{condition:d_sigma0_eta}, suppose \eqref{ineq: range of zeta}, \eqref{ineq: range of omega} and \eqref{ineq: range of gamma} hold at iteration $t$.  Then, it holds that
\begin{align*}
    \la\wb_{y_i,r}^{(t)},\bxi_{i}\ra&\geq-0.25, \\
    \la\wb_{y_i,r}^{(t)},\bxi_{i}\ra&\leq\sigma(\la\wb_{y_i,r}^{(t)},\bxi_{i}\ra)\leq\la\wb_{y_i,r}^{(t)},\bxi_{i}\ra+0.25,
\end{align*}
for any $i\in[n]$.
\end{lemma}
\begin{proof}[Proof of Lemma \ref{lm: noise alignment lower bound}]
According to \eqref{ineq: noise product bound2} in Lemma \ref{lm: inner product range}, we have
\begin{align*}
    \la\wb_{y_i,r}^{(t)},\bxi_i\ra&\geq\la\wb_{y_i,r}^{(0)},\bxi_i\ra+\zeta_{y_i,r,i}^{(t)}-5n\sqrt{\frac{\log(4n^2/\delta)}{d}}\alpha\\
    &\geq-\beta-5n\sqrt{\frac{\log(4n^2/\delta)}{d}}\alpha\\
    &\geq -0.25,
\end{align*}
where the second inequality is due to $\zeta_{y_i,r,i}^{(t)}\geq0$, the third inequality is due to $\beta<1/8$ and $5n\sqrt{\log(4n^2/\delta)/d}\cdot\alpha<1/8$. 

For the second inequality, LHS holds naturally since $z\leq\sigma(z)$. For RHS, if $\la\wb_{y_i,r}^{(t)},\bxi_{i}\ra\leq 0$, then
\begin{equation*}
    \sigma(\la\wb_{y_i,r}^{(t)},\bxi_{i}\ra)=0\leq\la\wb_{y_i,r}^{(t)},\bxi_{i}\ra+0.25.
\end{equation*}
If $\la\wb_{y_i,r}^{(t)},\bxi_{i}\ra> 0$, then
\begin{equation*}
    \sigma(\la\wb_{y_i,r}^{(t)},\bxi_{i}\ra)=\la\wb_{y_i,r}^{(t)},\bxi_{i}\ra<\la\wb_{y_i,r}^{(t)},\bxi_{i}\ra+0.25.
\end{equation*}
\end{proof}

Next we present an important Lemma \ref{lm: balanced logit}, which ensures the logits $\ell_i'^{(t)}$ for different $i \in[n] $ are balanced. As we will see later, this guarantees the coefficients $\gamma_{j, r}^{(t)}$ are monotone with respect to $t$ despite label-flipping noise, which is essential for the proof of Proposition \ref{proposition: range of gamma,zeta,omega}. In preparation, we first present a supplementary lemma. 

\begin{lemma}\label{lm: basic ANA}
Let $g(z) = \ell'(z) = -1/(1+\exp(z))$, then for all $z_{2} - c \geq z_{1} \geq -1$ where $c \geq 0$ we have that 
\begin{align*}
\frac{\exp(c)}{4}\leq \frac{g(z_{1})}{g(z_{2})} \leq \exp(c).    
\end{align*}
\end{lemma}

\begin{proof}[Proof of Lemma~\ref{lm: basic ANA}]
On one hand, we have
\begin{align*}
 \frac{1+\exp(z_{2})}{1+\exp(z_{1})}\leq \max\{1, \exp(z_{2}-z_{1})\} = \exp(c),    
\end{align*}
while on the other hand, we have
\begin{align*}
 \frac{1+\exp(z_{2})}{1+\exp(z_{1})} =  \frac{\exp(-z_{1})+\exp(z_{2}-z_{1})}{\exp(-z_{1}) + 1} \geq \frac{\exp(-z_{1})+\exp(c)}{\exp(-z_{1}) + 1} \geq \frac{\exp(1)+\exp(c)}{\exp(1) + 1} \geq \frac{\exp(c)}{4}. 
\end{align*}
\end{proof}

\begin{lemma} \label{lm: balanced logit}
    Under Condition \ref{condition:d_sigma0_eta}, suppose \eqref{ineq: range of zeta}, \eqref{ineq: range of omega} and \eqref{ineq: range of gamma} hold for any iteration $t' \leq t$. Then, the following conditions hold for any iteration $t' \leq t$: 
    \begin{enumerate}
        \item $\sum_{r=1}^{m}\big[\zeta_{y_i,r,i}^{(t')}-\zeta_{y_k,r,k}^{(t')}\big]\leq\kappa$ for all $i,k\in[n]$. 
        \item $y_i\cdot f(\Wb^{(t')},\xb_i)-y_k\cdot f(\Wb^{(t')},\xb_k)\leq C_1$ for all $i,k\in[n]$, 
        \item $\ell_i'^{(t')}/\ell_k'^{(t')}\leq C_2=\exp(C_1)$ for all $i,k\in[n]$. 
        \item $S_i^{(0)} \subseteq S_i^{(t')}$, where $S_{i}^{(t')}:=\{r\in[m]:\la\wb_{y_i,r}^{(t')},\bxi_i\ra>0\}$, and hence $|S_{i}^{(t')}|\geq0.4m$ for all $i\in[n]$.
        \item $S_{j,r}^{(0)} \subseteq S_{j,r}^{(t')}$ , where $S_{j,r}^{(t')}:=\{i\in[n]:y_i=j,\la\wb_{j,r}^{(t')},\bxi_i\ra>0\}$, and hence $|S_{j,r}^{(t')}|\geq n/8$ for all $j\in\{\pm1\},r\in[m]$.
    \end{enumerate}
    Here we take $\kappa$ and $C_1$ as $3.25$ and $5$ respectively.
\end{lemma}
\begin{proof}[Proof of Lemma \ref{lm: balanced logit}]
We prove this lemma by induction. When $t'=0$, the fourth and fifth conditions hold naturally, so we only need to verify the first three hypotheses. Since according to \eqref{initial gamma,zeta,omega} we have $\zeta_{j,r,i}^{(0)}=0$ for any $j,r,i$, it follows that $\sum_{r=1}^{m}\big[\zeta_{y_i,r,i}^{(0)}-\zeta_{y_k,r,k}^{(0)}\big]=0$ for all $i, k \in [n]$, and so the first condition holds for $t' = 0$. For the second condition, we have for any $i, k\in[n]$
\begin{align*}
    &\quad y_i\cdot f(\Wb^{(0)},\xb_i)-y_k\cdot f(\Wb^{(0)},\xb_k)\\
    &= F_{y_i}(\Wb_{y_i}^{(0)},\xb_i)-F_{-y_i}(\Wb_{-y_i}^{(0)},\xb_i)+F_{-y_k}(\Wb_{-y_k}^{(0)},\xb_i)-F_{y_k}(\Wb_{y_k}^{(0)},\xb_i)\\
    &\leq F_{y_i}(\Wb_{y_i}^{(0)},\xb_i)+F_{-y_k}(\Wb_{-y_k}^{(0)},\xb_i)\\
    &=\frac{1}{m}\sum_{r=1}^{m}[\sigma(\la\wb_{y_i,r}^{(t)},\hat{y}_i\bmu\ra)+\sigma(\la\wb_{y_i,r}^{(t)},\bxi_i\ra)]+\frac{1}{m}\sum_{r=1}^{m}[\sigma(\la\wb_{-y_k,r}^{(t)},\hat{y}_k\bmu\ra)+\sigma(\la\wb_{-y_k,r}^{(t)},\bxi_i\ra)]\\
    &\leq 2\beta\leq 1/3 \leq C_1,
\end{align*}
where the first inequality is by the fact that $F_{j}(\Wb_{j}^{(0)}, \xb_i) > 0$ for all $i \in [n], j \in [m]$, the second inequality is by the definition of $\beta$ in \eqref{def: beta}, while the third inequality follows from \eqref{ineq: alpha beta upper bound}. Finally, using the second condition, the third condition follows by
\begin{equation*}
    \frac{\ell_i'^{(0)}}{\ell_k'^{(0)}}\leq\exp\big(y_k\cdot f(\Wb^{(0)},\xb_k)-y_i\cdot f(\Wb^{(0)},\xb_i)\big)\leq\exp(C_1),
\end{equation*}
according to Lemma \ref{lm: basic ANA}. 

Now suppose there exists $\tilde{t} \leq t$ such that these five conditions hold for any $0\leq t'\leq \tilde{t} - 1$. We aim to prove that these conditions also hold for $t'=\tilde{t}$. 

We first show that, for any $0 \leq t' \leq t$, $y_i\cdot f(\Wb^{(t')},\xb_i)-y_k\cdot f(\Wb^{(t')},\xb_k)$ can be approximated by $\sum_{r=1}^{m} \big[\zeta_{y_i,r,i}^{(t')}-\zeta_{y_k,r,k}^{(t')}\big]$ with a small constant approximation error. We begin by writing out
\begin{equation}
\begin{aligned}\label{eq: decomposition of F difference}
    &\quad y_i\cdot f(\Wb^{(t')},\xb_i)-y_k\cdot f(\Wb^{(t')},\xb_k)\\
    &= y_i\sum_{j\in\{\pm1\}}j\cdot F_{j}(\Wb_{j}^{(t')},\xb_i)-y_k\sum_{j\in\{\pm1\}}j\cdot F_{j}(\Wb_{j}^{(t')},\xb_k)\\
    &= F_{-y_k}(\Wb_{-y_k}^{(t')},\xb_k)-F_{-y_i}(\Wb_{-y_i}^{(t')},\xb_i)+F_{y_i}(\Wb_{y_i}^{(t')},\xb_i)-F_{y_k}(\Wb_{y_k}^{(t')},\xb_k)\\
    &= F_{-y_k}(\Wb_{-y_k}^{(t')},\xb_k)-F_{-y_i}(\Wb_{-y_i}^{(t')},\xb_i)+\frac{1}{m}\sum_{r=1}^{m}[\sigma(\la\wb_{y_i,r}^{(t')},\hat{y}_{i}\cdot\bmu\ra)+\sigma(\la\wb_{y_i,r}^{(t')},\bxi_{i}\ra)]\\
    &\qquad -\frac{1}{m}\sum_{r=1}^{m}[\sigma(\la\wb_{y_k,r}^{(t')},\hat{y}_{k}\cdot\bmu\ra)+\sigma(\la\wb_{y_k,r}^{(t')},\bxi_{k}\ra)]\\
    &= \underbrace{F_{-y_k}(\Wb_{-y_k}^{(t')},\xb_k)-F_{-y_i}(\Wb_{-y_i}^{(t')},\xb_i)}_{I_1}+\underbrace{\frac{1}{m}\sum_{r=1}^{m}[\sigma(\la\wb_{y_i,r}^{(t')},\hat{y}_{i}\cdot\bmu\ra)-\sigma(\la\wb_{y_k,r}^{(t')},\hat{y}_{k}\cdot\bmu\ra)]}_{I_2}\\
    &\qquad +\underbrace{\frac{1}{m}\sum_{r=1}^{m}[\sigma(\la\wb_{y_i,r}^{(t')},\bxi_{i}\ra)-\sigma(\la\wb_{y_k,r}^{(t')},\bxi_{k}\ra)]}_{I_3},
\end{aligned}
\end{equation}
where all the equalities are due to the network definition. Next we estimate $I_1$, $I_2$ and $I_3$ one by one. For $|I_1|$, we have the following upper bound according to Lemma \ref{lm: mismatch fj upper bound}: 
\begin{equation}
    |I_1|\leq|F_{-y_k}(\Wb_{-y_k}^{(t')},\xb_k)|+|F_{-y_i}(\Wb_{-y_i}^{(t')},\xb_i)|= F_{-y_k}(\Wb_{-y_k}^{(t')},\xb_k)+F_{-y_i}(\Wb_{-y_i}^{(t')},\xb_i)\leq1.\label{ineq: |I1| upper bound}
\end{equation}
For $|I_2|$, we have the following upper bound: 
\begin{equation}\label{ineq: |I2| upper bound}
    \begin{aligned}
    |I_2|&\leq\max\bigg\{\frac{1}{m}\sum_{r=1}^{m}\sigma(\la\wb_{y_i,r}^{(t')},\hat{y}_{i}\cdot\bmu\ra),\frac{1}{m}\sum_{r=1}^{m}\sigma(\la\wb_{y_k,r}^{(t')},\hat{y}_{k}\cdot\bmu\ra)\bigg\}\\
    &\leq3\max\Bigg\{|\la\wb_{y_i,r}^{(0)},\hat{y}_{i}\cdot\bmu\ra|,|\la\wb_{y_k,r}^{(0)},\hat{y}_{k}\cdot\bmu\ra|,\gamma_{j,r}^{(t')},\mathrm{SNR}\sqrt{\frac{32\log(6n/\delta)}{d}}n\alpha\Bigg\}\\
    &\leq3\max\Bigg\{\beta,C'\hat{\gamma}\alpha,\mathrm{SNR}\sqrt{\frac{32\log(6n/\delta)}{d}}n\alpha\Bigg\}\\
    &\leq0.25,
    \end{aligned}
\end{equation}
where the second inequality is due to \eqref{ineq: feature product bound}; the second inequality is due to the definition of $\beta$ and \eqref{ineq: range of gamma}; the last inequality is due to Condition \ref{condition:d_sigma0_eta} and \eqref{ineq: alpha beta upper bound}.

For $I_3$, we have the following upper bound
\begin{equation}\label{ineq: I3 upper bound}
\begin{aligned}
    I_3&=\frac{1}{m}\sum_{r=1}^{m}\big[ \sigma(\la\wb_{y_i,r}^{(t')},\bxi_{i}\ra)-\sigma(\la\wb_{y_k,r}^{(t')},\bxi_{k}\ra) \big] \\
    &\leq\frac{1}{m}\sum_{r=1}^{m} \big[ \la \wb_{y_k, r}^{(t')}, \bxi_{i}\ra - \la \wb_{y_k, r}^{(t')}, \bxi_{k}\ra \big] + 0.25\\
    &\leq\frac{1}{m}\sum_{r=1}^{m}\bigg[\zeta_{y_i,r,i}^{(t')}-\zeta_{y_k,r,k}^{(t')}+10 \sqrt{\frac{\log(6n^2/\delta)}{d}} n\alpha+0.25\bigg]\\
    &\leq\frac{1}{m}\sum_{r=1}^{m}\big[\zeta_{y_i,r,i}^{(t')}-\zeta_{y_k,r,k}^{(t')}\big]+0.5,
\end{aligned}
\end{equation}
where the first inequality is due to $\sigma(\la\wb_{y_i,r}^{(t')},\bxi_{i}\ra)\leq\la\wb_{y_i,r}^{(t')},\bxi_{i}\ra+0.25$ and $\sigma(\la\wb_{y_k,r}^{(t')},\bxi_{k}\ra)\geq\la\wb_{y_k,r}^{(t')},\bxi_{k}\ra$ according to Lemma \ref{lm: noise alignment lower bound}; the second inequality is due to \eqref{ineq: noise product bound2} in Lemma \ref{lm: inner product range}; the last inequality is due to $5\sqrt{\log(6n^2/\delta)/d} n\alpha\leq1/8$ according to Condition \ref{condition:d_sigma0_eta}.
Similarly, we have the following lower bound
\begin{equation}\label{ineq: I3 lower bound}
    \begin{aligned}
    I_3&=\frac{1}{m}\sum_{r=1}^{m} \big[ \sigma(\la\wb_{y_i,r}^{(t')},\bxi_{i}\ra)-\sigma(\la\wb_{y_k,r}^{(t')},\bxi_{k}\ra) \big]\\
    &\geq\frac{1}{m} \sum_{r=1}^{m} \big[\la\wb_{y_i,r}^{(t')},\bxi_{i}\ra-\la\wb_{y_k,r}^{(t')},\bxi_{k}\ra \big] - 0.25\\
    &\geq\frac{1}{m}\sum_{r=1}^{m}\bigg[\zeta_{y_i,r,i}^{(t')}-\zeta_{y_k,r,k}^{(t')}-10 \sqrt{\frac{\log(6n^2/\delta)}{d}} n\alpha-0.25\bigg]\\
    &\geq\frac{1}{m}\sum_{r=1}^{m}\big[\zeta_{y_i,r,i}^{(t')}-\zeta_{y_k,r,k}^{(t')}\big]-0.5,
    \end{aligned}
\end{equation}
where the first inequality is due to $\sigma(\la\wb_{y_i,r}^{(t')},\bxi_{i}\ra)\geq\la\wb_{y_i,r}^{(t')},\bxi_{i}\ra$ and $\sigma(\la\wb_{y_k,r}^{(t')},\bxi_{k}\ra)\leq\la\wb_{y_k,r}^{(t')},\bxi_{k}\ra+0.25$ according to Lemma \ref{lm: noise alignment lower bound}; the second inequality is due to \eqref{ineq: noise product bound2} in Lemma \ref{lm: inner product range}; the last inequality is due to $5\sqrt{\log(6n^2/\delta)/d} n\alpha\leq1/8$ according to Condition \ref{condition:d_sigma0_eta}. 
Now, by plugging \eqref{ineq: |I1| upper bound}-\eqref{ineq: I3 upper bound} into \eqref{eq: decomposition of F difference}, we get
\begin{align*}
    y_i\cdot f(\Wb^{(t')},\xb_i)-y_k\cdot f(\Wb^{(t')},\xb_k)&\leq |I_1|+|I_2|+I_3\leq\frac{1}{m}\sum_{r=1}^{m}\big[\zeta_{y_i,r,i}^{(t')}-\zeta_{y_k,r,k}^{(t')}\big]+1.75\\
    y_i\cdot f(\Wb^{(t')},\xb_i)-y_k\cdot f(\Wb^{(t')},\xb_k)&\geq -|I_1|-|I_2|+I_3\geq\frac{1}{m}\sum_{r=1}^{m}\big[\zeta_{y_i,r,i}^{(t')}-\zeta_{y_k,r,k}^{(t')}\big]-1.75,
\end{align*}
which is equivalent to  
\begin{equation}\label{ineq: output difference approximation}
    \bigg|y_i\cdot f(\Wb^{(t')},\xb_i)-y_k\cdot f(\Wb^{(t')},\xb_k)-\frac{1}{m}\sum_{r=1}^{m}\big[\zeta_{y_i,r,i}^{(t')}-\zeta_{y_k,r,k}^{(t')}\big]\bigg|\leq 1.75.
\end{equation}
With this, we see that when the first condition holds for $t'$, the second condition immediately follows for $t'$. 

Next, we prove the first condition holds for $t'=\tilde{t}$. We first write an iterative update rule for $\sum_{r=1}^{m}\big[\zeta_{y_i,r,i}^{(\tilde{t})}-\zeta_{y_k,r,k}^{(\tilde{t})}]$. Recall that from Lemma \ref{lemma:coefficient_iterative_proof} that
\begin{equation*}
    \zeta_{j,r,i}^{(t+1)} = \zeta_{j,r,i}^{(t)} - \frac{\eta}{nm} \cdot \ell_i'^{(t)}\cdot\sigma'(\la\wb_{j,r}^{(t)},\bxi_i\ra)\cdot\ind(y_i=j)\|\bxi_i\|_2^2
\end{equation*}
for all $j\in\{\pm 1\}, r\in[m], i\in[n], t\in[0,T^*]$. Also recall the definition of $S_{i}^{(t)} = \{r\in[m]:\la\wb_{y_i,r}^{(t)},\bxi_i\ra>0\}$, it follows that
\begin{equation*}
    \sum_{r=1}^{m}\big[\zeta_{y_i,r,i}^{(t+1)}-\zeta_{y_k,r,k}^{(t+1)}\big]=\sum_{r=1}^{m}\big[\zeta_{y_i,r,i}^{(t)}-\zeta_{y_k,r,k}^{(t)}\big]-\frac{\eta}{nm}\cdot\big(|S_{i}^{(t)}|\ell_i'^{(t)}\cdot\|\bxi_i\|_2^2-|S_{k}^{(t)}|\ell_k'^{(t)}\cdot\|\bxi_k\|_2^2\big),
\end{equation*}
for all $i, k\in [n]$ and $0 \leq t \leq T^*$. Now we consider two separate cases: $\sum_{r=1}^{m}\big[\zeta_{y_i,r,i}^{(\tilde{t}-1)}-\zeta_{y_k,r,k}^{(\tilde{t}-1)}\big]\leq0.9\kappa$ and $\sum_{r=1}^{m}\big[\zeta_{y_i,r,i}^{(\tilde{t}-1)}-\zeta_{y_k,r,k}^{(\tilde{t}-1)}\big]>0.9\kappa$. 

For when $\sum_{r=1}^{m}\big[\zeta_{y_i,r,i}^{(\tilde{t}-1)}-\zeta_{y_k,r,k}^{(\tilde{t}-1)}\big]\leq0.9\kappa$, we have
\begin{align*}
    \sum_{r=1}^{m}\big[\zeta_{y_i,r,i}^{(\tilde{t})}-\zeta_{y_k,r,k}^{(\tilde{t})}\big]&=\sum_{r=1}^{m}\big[\zeta_{y_i,r,i}^{(\tilde{t}-1)}-\zeta_{y_k,r,k}^{(\tilde{t}-1)}\big] -\frac{\eta}{nm}\cdot\Big(|S_{i}^{(\tilde{t}-1)}|\ell_i'^{(\tilde{t}-1)}\cdot\|\bxi_i\|_2^2-|S_{k}^{(\tilde{t}-1)}|\ell_k'^{(\tilde{t}-1)}\cdot\|\bxi_k\|_2^2\Big)\\
    &\leq\sum_{r=1}^{m}\big[\zeta_{y_i,r,i}^{(\tilde{t}-1)}-\zeta_{y_k,r,k}^{(\tilde{t}-1)}\big]-\frac{\eta}{nm}\cdot|S_{i}^{(\tilde{t}-1)}|\ell_i'^{(\tilde{t}-1)}\cdot\|\bxi_i\|_2^2\\
    &\leq\sum_{r=1}^{m}\big[\zeta_{y_i,r,i}^{(\tilde{t}-1)}-\zeta_{y_k,r,k}^{(\tilde{t}-1)}\big]+\frac{\eta}{n}\cdot\|\bxi_i\|_2^2\\
    &\leq0.9\kappa+0.1\kappa\\
    &=\kappa,
\end{align*}
where the first inequality is due to $\ell_i'^{(\tilde{t}-1)}<0$; the second inequality is due to $|S_{i}^{(\tilde{t}-1)}|\leq m$ and $-\ell_i'^{(\tilde{t}-1)}<1$; the third inequality is due to the assumption in this case and $\eta \leq C^{-1} \cdot n\sigma_p^{-2}d^{-1}$ from Condition~\ref{condition:d_sigma0_eta}.

On the other hand, for when $\sum_{r=1}^{m}\big[\zeta_{y_i,r,i}^{(\tilde{t}-1)}-\zeta_{y_k,r,k}^{(\tilde{t}-1)}\big]>0.9\kappa$, we have from the \eqref{ineq: output difference approximation} that 
\begin{equation}
    \begin{aligned}
    y_i\cdot f(\Wb^{(\tilde{t}-1)},\xb_i)-y_k\cdot f(\Wb^{(\tilde{t}-1)},\xb_k)
    &\geq\frac{1}{m}\sum_{r=1}^{m}\big[\zeta_{y_i,r,i}^{(\tilde{t}-1)}-\zeta_{y_k,r,k}^{(\tilde{t}-1)}\big]-1.75\\
    &\geq0.9\kappa-0.54\kappa\\
    &=0.36\kappa,
    \end{aligned}
\end{equation}
where the second inequality is due to $\kappa= 3.25$. Thus, according to Lemma \ref{lm: basic ANA}, we have
\begin{equation*}
    \frac{\ell_i'^{(\tilde{t}-1)}}{\ell_k'^{(\tilde{t}-1)}}\leq\exp\big(y_k\cdot f(\Wb^{(\tilde{t}-1)},\xb_k)-y_i\cdot f(\Wb^{(\tilde{t}-1)},\xb_i)\big)\leq\exp(-0.36\kappa).
\end{equation*}
Since we have $|S_{i}^{(\tilde{t}-1)}|\leq m$ and $|S_{k}^{(\tilde{t}-1)}|\geq 0.4m$ according to the fourth condition, it follows that
\begin{equation*}
    \frac{\big|S_{i}^{(\tilde{t}-1)}\big|\ell_i'^{(\tilde{t}-1)}}{\big|S_{k}^{(\tilde{t}-1)}\big|\ell_k'^{(\tilde{t}-1)}}\leq 2.5\exp(-0.36\kappa)<0.8.
\end{equation*}
According to Lemma \ref{lm: data inner products}, under event $\cE_{\mathrm{prelim}}$, we have
\begin{equation*}
    \big|\|\bxi_i\|_2^2-d\cdot\sigma_p^2\big|=O\big(\sigma_p^2\cdot\sqrt{d\log(6n/\delta)}\big),\, \forall i\in[n].
\end{equation*}
Note that $d=\Omega(\log(6n/\delta))$ from Condition \ref{condition:d_sigma0_eta}, it follows that
\begin{equation*}
    |S_{i}^{(\tilde{t}-1)}|(-\ell_i'^{(\tilde{t}-1)})\cdot\|\bxi_i\|_2^2<|S_{k}^{(\tilde{t}-1)}|(-\ell_k'^{(\tilde{t}-1)})\cdot\|\bxi_k\|_2^2.
\end{equation*}
Then we have
\begin{equation*}
    \sum_{r=1}^{m}\big[\zeta_{y_i,r,i}^{(\tilde{t})}-\zeta_{y_k,r,k}^{(\tilde{t})}\big]\leq\sum_{r=1}^{m}\big[\zeta_{y_i,r,i}^{(\tilde{t}-1)}-\zeta_{y_k,r,k}^{(\tilde{t}-1)}\big]\leq\kappa,
\end{equation*}
which completes the proof of the first hypothesis at iteration $t'=\tilde{t}$. Next, by applying the approximation in \eqref{ineq: output difference approximation}, we are ready to verify the second hypothesis at iteration $\tilde{t}$. In fact, we have
\begin{align*}
    y_i\cdot f(\Wb^{(\tilde{t})},\xb_i)-y_k\cdot f(\Wb^{(\tilde{t})},\xb_k)&\leq\frac{1}{m}\sum_{r=1}^{m}\big[\zeta_{y_i,r,i}^{(\tilde{t})}-\zeta_{y_k,r,k}^{(\tilde{t})}\big]+1.75\leq C_1,
\end{align*}
where the first inequality is by \eqref{ineq: output difference approximation}; the last inequality is by induction hypothesis and taking $\kappa$ as 3.25 and $C_1$ as 5. And the third hypothesis directly follows by noting that
\begin{equation*}
    \frac{\ell_i'^{(\tilde{t})}}{\ell_k'^{(\tilde{t})}}\leq\exp\big(y_k\cdot f(\Wb^{(\tilde{t})},\xb_k)-y_i\cdot f(\Wb^{(\tilde{t})},\xb_i)\big)\leq\exp(C_1)=C_2.
\end{equation*}
To verify the fourth hypothesis, according to the gradient descent rule, we have
\begin{align*}
    \la\wb_{y_i,r}^{(\tilde{t})},\bxi_i\ra&=\la\wb_{y_i,r}^{(\tilde{t}-1)},\bxi_i\ra-\frac{\eta}{nm}\cdot\sum_{i'=1}^{n}\ell_{i'}^{(\tilde{t}-1)}\cdot\sigma'(\la\wb_{y_i,r}^{(\tilde{t}-1)},\hat{y}_{i'}\bmu\ra)\cdot\la\hat{y}_{i'}\bmu,\bxi_i\ra\\
    &\qquad-\frac{\eta}{nm}\cdot\sum_{i'=1}^{n}\ell_{i'}^{(\tilde{t}-1)}\cdot\sigma'(\la\wb_{y_i,r}^{(\tilde{t}-1)},\bxi_{i'}\ra)\cdot\la\bxi_{i'},\bxi_{i}\ra\\
    &=\la\wb_{y_i,r}^{(\tilde{t}-1)},\bxi_i\ra-\frac{\eta}{nm}\cdot\sum_{i'=1}^{n}\ell_{i'}^{(\tilde{t}-1)}\cdot\sigma'(\la\wb_{y_i,r}^{(\tilde{t}-1)},\hat{y}_{i'}\bmu\ra)\cdot\la\hat{y}_{i'}\bmu,\bxi_i\ra\\
    &\qquad-\frac{\eta}{nm}\cdot\ell_{i}^{(\tilde{t}-1)}\cdot\sigma'(\la\wb_{y_i,r}^{(\tilde{t}-1)},\bxi_{i}\ra)\cdot\|\bxi_i\|_2^2-\frac{\eta}{nm}\cdot\sum_{i'\neq i}\ell_{i'}^{(\tilde{t}-1)}\cdot\sigma'(\la\wb_{y_i,r}^{(\tilde{t}-1)},\bxi_{i'}\ra)\cdot\la\bxi_{i'},\bxi_{i}\ra\\
    &=\la\wb_{y_i,r}^{(\tilde{t}-1)},\bxi_i\ra-\frac{\eta}{nm}\cdot\underbrace{\ell_{i}^{(\tilde{t}-1)}\cdot\|\bxi_i\|_2^2}_{I_4}-\frac{\eta}{nm}\cdot\underbrace{\sum_{i'\neq i}\ell_{i'}^{(\tilde{t}-1)}\cdot\sigma'(\la\wb_{y_i,r}^{(\tilde{t}-1)},\bxi_{i'}\ra)\cdot\la\bxi_{i'},\bxi_{i}\ra}_{I_5}\\
    &\qquad -\frac{\eta}{nm}\cdot\underbrace{\sum_{i'=1}^{n}\ell_{i'}^{(\tilde{t}-1)}\cdot\sigma'(\la\wb_{y_i,r}^{(\tilde{t}-1)},\hat{y}_{i'}\bmu\ra)\cdot\la\hat{y}_{i'}\bmu,\bxi_i\ra}_{I_6},
\end{align*}
for any $r\in S_i^{(\tilde{t}-1)}$, where the last equality is by $\la\wb_{y_i,r}^{(\tilde{t}-1)},\bxi_i\ra>0$. Then we respectively estimate $I_4,I_5,I_6$. For $I_4$, according to Lemma \ref{lm: data inner products}, we have
\begin{equation*}
    -I_4\geq|\ell_{i}^{(\tilde{t}-1)}|\cdot\sigma_p^2 d/2.
\end{equation*}
For $I_5$, we have following upper bound
\begin{align*}
    |I_5|&\leq\sum_{i'\neq i}|\ell_{i'}^{(\tilde{t}-1)}|\cdot\sigma'(\la\wb_{y_i,r}^{(\tilde{t}-1)},\bxi_{i'}\ra)\cdot|\la\bxi_{i'},\bxi_{i}\ra|\\
    &\leq\sum_{i'\neq i}|\ell_{i'}^{(\tilde{t}-1)}|\cdot|\la\bxi_{i'},\bxi_{i}\ra|\\
    &\leq\sum_{i'\neq i}|\ell_{i'}^{(\tilde{t}-1)}|\cdot2\sigma_p^2\cdot\sqrt{d\log(6n^2/\delta)}\\
    &\leq nC_2|\ell_{i}^{(\tilde{t}-1)}|\cdot2\sigma_p^2\cdot\sqrt{d\log(6n^2/\delta)},
\end{align*}
where the first inequality is due to triangle inequality; the second inequality is due to $\sigma'(z)\in\{0,1\}$; the third inequality is due to Lemma \ref{lm: data inner products}; the last inequality is due to the third hypothesis at iteration $\tilde{t}-1$.

For $I_6$, we have following upper bound
\begin{align*}
    |I_6|&\leq\sum_{i'=1}^{n}|\ell_{i'}^{(\tilde{t}-1)}|\cdot\sigma'(\la\wb_{y_i,r}^{(\tilde{t}-1)},\hat{y}_{i'}\bmu\ra)\cdot|\la\hat{y}_{i'}\bmu,\bxi_i\ra|\\
    &\leq\sum_{i'=1}^{n}|\ell_{i'}^{(\tilde{t}-1)}|\cdot|\la\hat{y}_{i'}\bmu,\bxi_i\ra|\\
    &\leq\sum_{i'=1}^{n}|\ell_{i'}^{(\tilde{t}-1)}|\cdot\|\bmu\|_2\sigma_p\sqrt{2\log(6n/\delta)}\\
    &\leq nC_2|\ell_{i}^{(\tilde{t}-1)}|\cdot\|\bmu\|_2\sigma_p\sqrt{2\log(6n/\delta)},
\end{align*}
where the first inequality is by triangle inequality; the second inequality is due to $\sigma'(z)\in\{0,1\}$; the third inequality is by Lemma \ref{lm: data inner products}; the last inequality is due to the third hypothesis at iteration $\tilde{t}-1$. Since $d\geq \max\{32C_2^2n^2\cdot\log(6n^2/\delta),4C_2n\|\bmu\|\sigma_p^{-1}\sqrt{2\log(6n/\delta)}\}$, we have $-I_4\geq\max\{|I_5|/2,|I_{6}|/2\}$ and hence $-I_4\geq |I_5|+|I_{6}|$. It follows that
\begin{equation*}
    \la\wb_{y_i,r}^{(\tilde{t})},\bxi_i\ra\geq\la\wb_{y_i,r}^{(\tilde{t}-1)},\bxi_i\ra>0,
\end{equation*}
for any $r\in S_{i}^{(\tilde{t}-1)}$. Therefore, $S_{i}^{(0)}\subseteq S_{i}^{(\tilde{t}-1)}\subseteq S_{i}^{(\tilde{t})}$. And it directly follows by Lemma \ref{lm: number of initial activated neurons} that $|S_{i}^{(\tilde{t})}|\geq0.4m,\, \forall i\in[n]$, which implies that the fourth hypothesis holds for $t'=\tilde{t}$. For the fifth hypothesis, similar to the proof of the fourth hypothesis, we also have
\begin{align*}
    \la\wb_{j,r}^{(\tilde{t})},\bxi_i\ra&=\la\wb_{j,r}^{(\tilde{t}-1)},\bxi_i\ra-\frac{\eta}{nm}\cdot\ell_{i}^{(\tilde{t}-1)}\cdot\|\bxi_i\|_2^2-\frac{\eta}{nm}\cdot\sum_{i'\neq i}\ell_{i'}^{(\tilde{t}-1)}\cdot\sigma'(\la\wb_{y_i,r}^{(\tilde{t}-1)},\bxi_{i'}\ra)\cdot\la\bxi_{i'},\bxi_{i}\ra\\
    &\qquad -\frac{\eta}{nm}\cdot\sum_{i'=1}^{n}\ell_{i'}^{(\tilde{t}-1)}\cdot\sigma'(\la\wb_{y_i,r}^{(\tilde{t}-1)},\hat{y}_{i'}\bmu\ra)\cdot\la\hat{y}_{i'}\bmu,\bxi_i\ra
\end{align*}
for any $i\in S_{j,r}^{(\tilde{t}-1)}$, where the equality holds due to $\la\wb_{j,r}^{(\tilde{t}-1)},\bxi_i\ra>0$ and $y_i=j$. By applying the same technique used in the proof of the fourth hypothesis, it follows that
\begin{equation*}
    \la\wb_{j,r}^{(\tilde{t})},\bxi_i\ra\geq\la\wb_{j,r}^{(\tilde{t}-1)},\bxi_i\ra>0,
\end{equation*}
for any $i\in S_{j,r}^{(\tilde{t}-1)}$. Thus, we have $S_{j,r}^{(0)}\subseteq S_{j,r}^{(\tilde{t}-1)}\subseteq S_{j,r}^{(\tilde{t})}$. And it directly follows by Lemma \ref{lm: number of initial activated neurons 2} that $|S_{j,r}^{(\tilde{t})}|\geq n/8$, which implies that the fourth hypothesis holds for $t'=\tilde{t}$.
Therefore, the five hypotheses hold for $t'=\tilde{t}$, which completes the induction.

\end{proof}

Now we are ready to prove Proposition \ref{proposition: range of gamma,zeta,omega}. 

\begin{proof}[Proof of Proposition \ref{proposition: range of gamma,zeta,omega}]
Our proof is based on induction. The results are obvious at $t=0$ as all the coefficients are zero. Suppose that there exists $\tilde{T}\leq T^*$ such that the results in Proposition \ref{proposition: range of gamma,zeta,omega} hold for all time $0\leq t\leq\tilde{T}-1$. We aim to prove that they also hold for $t=\tilde{T}$. Note that according to Lemma \ref{lm: balanced logit}, we also have for any $0\leq t\leq \tilde{T}-1$ that
\begin{enumerate}
    \item $\sum_{r=1}^{m}\big[\zeta_{y_i,r,i}^{(t)}-\zeta_{y_k,r,k}^{(t)}\big]\leq\kappa$ for all $i,k\in[n]$. 
    \item $y_i\cdot f(\Wb^{(t)},\xb_i)-y_k\cdot f(\Wb^{(t)},\xb_k)\leq C_1$ for all $i,k\in[n]$, 
    \item $\ell_i'^{(t)}/\ell_k'^{(t)}\leq C_2=\exp(C_1)$ for all $i,k\in[n]$. 
    \item $S_i^{(0)} \subseteq S_i^{(t)}$ for all $i\in[n]$, where $S_{i}^{(t)}:=\{r\in[m]:\la\wb_{y_i,r}^{(t)},\bxi_i\ra>0\}$, and hence $|S_{i}^{(t)}|\geq0.4m$, for all $i\in[n]$.
    \item $S_{j,r}^{(0)} \subseteq S_{j,r}^{(t)}$ , where $S_{j,r}^{(t)}:=\{i\in[n]:y_i=j,\la\wb_{j,r}^{(t)},\bxi_i\ra>0\}$, and hence $|S_{j,r}^{(t)}|\geq n/8$ for all $j\in\{\pm1\},r\in[m]$.
\end{enumerate}

We first prove that \eqref{ineq: range of omega} holds for $t=\tilde{T}$, i.e., $\omega_{j,r,i}^{(t)} \geq - \beta - 10 \sqrt{ \log(6n^2/\delta) / d} \cdot n \alpha$ for $t=\tilde{T}$ and any $r\in[m]$, $j\in\{\pm1\}$ and $i\in[n]$. Notice that $\omega_{j,r,i}^{(t)}=0$ for $j=y_i$, therefore we only need to consider the case that $j\neq y_i$. When $\omega_{j,r,t}^{(\tilde{T}-1)}<-0.5\beta-5 \sqrt{ \log(6n^2/\delta) / d} \cdot n\alpha$, by \eqref{ineq: noise product bound2} in Lemma \ref{lm: inner product range} we have that
\begin{equation*}
    \la \wb_{j, r}^{(\tilde{T} - 1)}, \bxi_i \ra \leq \omega_{j, r, i}^{(\tilde{T} - 1)} + \la \wb_{j, r}^{(0)}, \bxi_i \ra + 5 \sqrt{\frac{\log(6 n^2 / \delta)}{d}} n \alpha < 0,
\end{equation*}
and thus
\begin{align*}
    \omega_{j, r, i}^{(\tilde{T})}&=\omega_{j,r,i}^{(\tilde{T}-1)}+\frac{\eta}{nm}\cdot\ell_i'^{(\tilde{T}-1)}\cdot\ind(\la\wb_{j,r}^{(\tilde{T}-1)},\bxi_i\ra \geq 0)\cdot\ind(y_i=-j)\|\bxi_i\|_2^2\\
    &=\omega_{j,r,i}^{(\tilde{T}-1)}\\
    &\geq-\beta-10\sqrt{\frac{\log(6n^2/\delta)}{d}} n\alpha,
\end{align*}
where the last inequality is by induction hypothesis. When $\omega_{j,r,t}^{(\tilde{T}-1)}\geq-0.5\beta-5 \sqrt{ \log(6n^2/\delta) / d} \cdot n\alpha$, we have
\begin{align*}
    \underline{\rho}_{j,r,i}^{(\tilde{T})} &= \omega_{j,r,i}^{(\tilde{T}-1)} + \frac{\eta}{nm}\cdot\ell_i'^{(\tilde{T}-1)}\cdot\ind(\la\wb_{j,r}^{(\tilde{T}-1)},\bxi_i\ra\geq0)\cdot\ind(y_i=-j)\|\bxi_i\|_2^2\\
    &\geq-0.5\beta-5\sqrt{\frac{\log(6n^2/\delta)}{d}} n\alpha- \frac{3\eta\sigma_p^2d}{2nm}\\
    &\geq-0.5\beta-10\sqrt{\frac{\log(6n^2/\delta)}{d}} n\alpha\\
    &\geq-\beta-10\sqrt{\frac{\log(6n^2/\delta)}{d}} n\alpha,
\end{align*}
where the first equality is by $\ell_i'^{(\tilde{T}-1)}\in(-1,0)$ and $\|\bxi_i\|_2^2 \leq (3/2) \sigma_p^2 d$ by Lemma~\ref{lm: data inner products}; the second inequality is due to  $5\sqrt{ \log(6n^2/\delta) / d} \cdot n \alpha \geq 3\eta \sigma_p^2 d /2 nm$ by the condition for $\eta$ in Condition \ref{condition:d_sigma0_eta}.

Next we prove \eqref{ineq: range of zeta} holds for $t=\tilde{T}$. Consider
\begin{equation}
    \begin{aligned}\label{ineq: logit}
        |\ell_i'^{(t)}|&=\frac{1}{1+\exp\{y_i\cdot[F_{+1}(\Wb_{+1}^{(t)},\xb_i)-F_{-1}(\Wb_{-1}^{(t)},\xb_i)]\}}\\
        &\leq\exp\{-y_i\cdot[F_{+1}(\Wb_{+1}^{(t)},\xb_i)-F_{-1}(\Wb_{-1}^{(t)},\xb_i)]\}\\
        &\leq\exp\{-F_{y_i}(\Wb_{y_i}^{(t)},\xb_i)+0.5\},
    \end{aligned}
\end{equation}
where the last inequality is by $F_{j}(\Wb_{j}^{(t)},\xb_i)\leq 0.5$ for $j\neq y_i$ according to Lemma \ref{lm: mismatch fj upper bound}. Now recall the iterative update rule of $\zeta_{j,r,i}^{(t)}$:
\begin{equation*}
    \zeta_{j,r,i}^{(t+1)} = \zeta_{j,r,i}^{(t)} - \frac{\eta}{nm}\cdot\ell_i'^{(t)}\cdot\ind(\la\wb_{j,r}^{(t)},\bxi_i\ra \geq 0) \cdot \ind(y_i=j) \|\bxi_i\|_{2}^{2}.
\end{equation*}

Let $t_{j,r,i}$ be the last time $t<T^*$ that $\zeta_{j,r,i}^{(t)}\leq0.5\alpha$. Then by iterating the update rule from $t = t_{j, r, i}$ to $t = \tilde{T} - 1$, we get
\begin{equation}\label{eq: zeta}
    \begin{aligned}
    \zeta_{j,r,i}^{(\tilde{T})}&=\zeta_{j,r,i}^{(t_{j,r,i})}-\underbrace{\frac{\eta}{nm}\cdot\ell_i'^{(t_{j,r,i})}\cdot\ind(\la\wb_{j,r}^{(t_{j,r,i})},\bxi_i\ra \geq 0) \cdot \ind(y_i=j) \|\bxi_i\|_2^2}_{I_7} \\
    &\qquad-\underbrace{\sum_{t_{j,r,i}<t<\tilde{T}}\frac{\eta}{nm}\cdot\ell_i'^{(t)}\cdot\ind(\la\wb_{j,r}^{(t)},\bxi_i\ra \geq 0) \cdot\ind(y_i=j)\|\bxi_i\|_2^2}_{I_8}.
    \end{aligned}
\end{equation}

We first bound $I_7$ as follows: 
\begin{equation*}
    |I_7|
    \leq (\eta / nm) \cdot\|\bxi_i\|_2^2
    \leq (\eta / nm) \cdot 3\sigma_p^2d/2 
    \leq 1
    \leq 0.25\alpha,
\end{equation*}
where the first inequality is by $\ell_i'^{(t_{j,r,i})}\in(-1,0)$; the second inequality is by Lemma \ref{lm: data inner products}; the third inequality is by $\eta\leq C^{-1} \cdot n/ (\sigma_p^2 d)$ from Condition~\ref{condition:d_sigma0_eta}; the last inequality is by our choice of $\alpha=4\log(T^*)$ and $T^*\geq e$.

Second, we bound $I_8$. For $t_{j,r,i} < t < \tilde{T}$ and $y_i=j$, we can lower bound the inner product $\la \wb_{j,r}^{(t)}, \bxi_i \ra$ as follows
\begin{equation}\label{ineq: inner product w,bxi}
    \begin{aligned}
    \la \wb_{j,r}^{(t)}, \bxi_i \ra &\geq \la \wb_{j,r}^{(0)}, \bxi_i \ra + \zeta_{j,r,i}^{(t)} - 5 \sqrt{ \frac{ \log(6 n^2 / \delta)} {d}} n \alpha \\
    &\geq-0.5\beta+0.5\alpha-5\sqrt{\frac{\log(6n^2/\delta)}{d}} n\alpha\\
    &\geq0.25\alpha,
    \end{aligned}
\end{equation}
where the first inequality is by \eqref{ineq: noise product bound1} in Lemma \ref{lm: inner product range}; 
the second inequality is by $\zeta_{j,r,i}^{(t)}>0.5\alpha$ and $\la\wb_{j,r}^{(0)},\bxi_i\ra\geq-0.5\beta$ due to the definition of $t_{j,r,i}$ and $\beta$; 
the last inequality is by $\beta\leq 1/8 \leq 0.1\alpha$ and $5 \sqrt{ \log( 6 n^2 / \delta) / d} \cdot n \alpha \leq 0.2 \alpha$ by $d \geq C \cdot n^{2} \log(nm/\delta) (\log T^{*})^2$ from Condition~\ref{condition:d_sigma0_eta}. 
Thus, plugging the lower bounds of $\la\wb_{j,r}^{(t)},\bxi_i\ra$ into $I_8$ gives
\begin{align*}
    |I_8|&\leq\sum_{t_{j,r,i}<t<\tilde{T}}\frac{\eta}{nm}\cdot\exp(-\sigma(\la\wb_{j,r}^{(t)},\bxi_i\ra)+0.5)\cdot\sigma'(\la\wb_{j,r}^{(t)},\bxi_i\ra)\cdot\ind(y_i=j)\|\bxi_i\|_2^2\\
    &\leq\frac{2\eta(\tilde{T}-t_{j,r,i}-1)}{nm}\cdot\exp(-0.25\alpha)\cdot\frac{3\sigma_p^2 d}{2}\\
    &\leq\frac{2\eta T^*}{nm}\cdot\exp(-\log(T^*))\cdot\frac{3\sigma_p^2 d}{2}\\
    &=\frac{2\eta}{nm}\cdot\frac{3\sigma_p^2 d}{2}\leq1\leq0.25\alpha,
\end{align*}
where the first inequality is by \eqref{ineq: logit}; 
the second inequality is by \eqref{ineq: inner product w,bxi}; 
the third inequality is by $\alpha = 4 \log(T^*)$; 
the fourth inequality is by $\eta\leq C^{-1} n^2 m \sqrt{\log(n/\delta)}\sigma_p^{-2} d^{-3/2} \leq nm / (3\sigma_p^2 d)$ based on the conditions for $\eta$ and $d$ in Condition~\ref{condition:d_sigma0_eta}; 
the last inequality is by $\log(T^*) \geq1$ and $\alpha=4\log(T^*)$.  
Plugging the bound of $I_7, I_8$ into \eqref{eq: zeta} completes the proof for $\zeta$. 

Next, we prove \eqref{ineq: range of gamma} holds for $t=\tilde{T}$. Recall the iterative update rule of $\gamma_{j,r}^{(t)}$, we have
\begin{align*}
    \gamma_{j,r}^{(\tilde{T})} &= \gamma_{j,r}^{(\tilde{T}-1)} - \frac{\eta}{nm} \cdot \bigg[\sum_{i\in S_{+}} \ell_i'^{(\tilde{T}-1)} \sigma'(\la\wb_{j,r}^{(\tilde{T}-1)}, \hat{y}_{i} \cdot \bmu\ra) - \sum_{i\in S_{-}} \ell_i'^{(\tilde{T}-1)} \sigma'(\la\wb_{j,r}^{(\tilde{T}-1)}, \hat{y}_{i} \cdot \bmu\ra)\bigg] \cdot\|\bmu\|_2^2
\end{align*}
We first prove that the coefficients $\gamma_{j, r}^{(\tilde{T})}\geq\gamma_{j, r}^{(\tilde{T}-1)}$ and hence $\gamma_{j,r}^{(\tilde{T})}\geq\gamma_{j,r}^{(0)}=0$ for any $j\in\{\pm1\},r\in[m]$. Recall the definition of $S_{+} = \{i| y_{i} = \hat{y}_{i}\}$, $S_{-} = \{i| y_{i} \not= \hat{y}_{i}\}$, $S_{1}=\{i|\hat{y}_i=1\}$ and $S_{-1}=\{i|\hat{y}_i=-1\}$. We will consider the following two cases separately: $\la\wb_{j, r}^{(\tilde{T}-1)}, \bmu \ra \geq 0$ and $\la\wb_{j, r}^{(\tilde{T}-1)}, \bmu \ra < 0$. If $\la\wb_{j, r}^{(\tilde{T}-1)}, \bmu \ra \geq 0 $, then \begin{align*}
    & \quad -\sum_{i\in S_{+}} \ell_i'^{(\tilde{T}-1)} \sigma'(\la\wb_{j,r}^{(\tilde{T}-1)}, \hat{y}_{i} \cdot \bmu\ra)+ \sum_{i\in S_{-}} \ell_i'^{(\tilde{T}-1)} \sigma'(\la\wb_{j,r}^{(\tilde{T}-1)}, \hat{y}_{i} \cdot \bmu\ra) \\
    &= -\sum_{i \in S_{+}} \ell_i'^{(\tilde{T}-1)} \ind (\hat{y}_{i} \cdot \la \wb_{j, r}^{(\tilde{T}-1)}, \bmu \ra \geq 0)+ \sum_{i\in S_{-}} \ell_i'^{(\tilde{T}-1)} \ind (\hat{y}_{i} \cdot \la \wb_{j, r}^{(\tilde{T}-1)}, \bmu \ra \geq 0) \\
    &=\sum_{i\in S_{+}\cap S_{1}}|\ell_i'^{(\tilde{T}-1)}|-\sum_{i\in S_{-}\cap S_{-1}}|\ell_i'^{(\tilde{T}-1)}| \\
    &\geq|S_{+}\cap S_{1}|\cdot\min_{i\in S_{+}\cap S_{1}}|\ell_i'^{(\tilde{T}-1)}| -|S_{-}\cap S_{-1}|\cdot\max_{i\in S_{-}\cap S_{-1}}|\ell_i'^{(\tilde{T}-1)}|,
\end{align*}
where the second equality is due to $\ell_i'^{(\tilde{T}-1)} < 0$. If $\la\wb_{j,r}^{(\tilde{T}-1)},\bmu\ra<0$, then with a similar reasoning we have
\begin{align*}
    & \quad -\sum_{i\in S_{+}} \ell_i'^{(\tilde{T}-1)} \sigma'(\la\wb_{j,r}^{(\tilde{T}-1)}, \hat{y}_{i} \cdot \bmu\ra)+\sum_{i\in S_{-}} \ell_i'^{(\tilde{T}-1)} \sigma'(\la\wb_{j,r}^{(\tilde{T}-1)}, \hat{y}_{i} \cdot \bmu\ra) \\
    &= -\sum_{i\in S_{+}} \ell_i'^{(\tilde{T}-1)} \ind (\hat{y}_{i} \cdot \la\wb_{j,r}^{(\tilde{T}-1)}, \bmu\ra \geq 0) + \sum_{i\in S_{-}} \ell_i'^{(\tilde{T}-1)} \ind (\hat{y}_{i} \cdot \la\wb_{j,r}^{(\tilde{T}-1)}, \bmu\ra \geq 0) \\
    &=\sum_{i\in S_{+}\cap S_{-1}}|\ell_i'^{(\tilde{T}-1)}|-\sum_{i\in S_{-}\cap S_{1}}|\ell_i'^{(\tilde{T}-1)}| \\
    &\geq|S_{+}\cap S_{-1}|\cdot\min_{i\in S_{+}\cap S_{-1}}|\ell_i'^{(\tilde{T}-1)}| -|S_{-}\cap S_{1}|\cdot\max_{i\in S_{-}\cap S_{1}}|\ell_i'^{(\tilde{T}-1)}|.
\end{align*}
According to Lemma \ref{lm: estimate S cap S} and the third statement from Lemma \ref{lm: balanced logit} that $\ell_i '^{(\tilde{T}-1)} / \ell_k '^{(\tilde{T}-1)}\leq C_2, \forall i, k \in [n]$, under event $\cE_{\mathrm{prelim}}$, we have
\begin{align*}
    &\frac{|S_{+}\cap S_{1}|\cdot\min_{i\in S_{+}\cap S_{1}}|\ell_i'^{(\tilde{T}-1)}|}{|S_{-}\cap S_{-1}|\cdot\max_{i\in S_{-}\cap S_{-1}}|\ell_i'^{(\tilde{T}-1)}|}\geq\frac{|S_{+}\cap S_{1}|}{C_2|S_{-}\cap S_{-1}|}\geq\frac{(1-p)n-\sqrt{2n\log(8/\delta)}}{C_2\cdot(pn+\sqrt{2n\log(8/\delta)})},\\
    &\frac{|S_{+}\cap S_{-1}|\cdot\min_{i\in S_{+}\cap S_{-1}}|\ell_i'^{(\tilde{T}-1)}|}{|S_{-}\cap S_{1}|\cdot\max_{i\in S_{-}\cap S_{1}}|\ell_i'^{(\tilde{T}-1)}|}\geq\frac{|S_{+}\cap S_{-1}|}{C_2|S_{-}\cap S_{1}|}\geq\frac{(1-p)n-\sqrt{2n\log(8/\delta)}}{C_2\cdot(pn+\sqrt{2n\log(8/\delta)})}.
\end{align*}
As long as $p<1/[2(1+C_2)]$ and $n\geq 8(C_2+1)^2\log(8/\delta)$, we have
\begin{equation*}
    \frac{|S_{+}\cap S_{1}|\cdot\min_{i\in S_{+}\cap S_{1}}|\ell_i'^{(\tilde{T}-1)}|}{|S_{-}\cap S_{-1}|\cdot\max_{i\in S_{-}\cap S_{-1}}|\ell_i'^{(\tilde{T}-1)}|}\geq1,\,\frac{|S_{+}\cap S_{-1}|\cdot\min_{i\in S_{+}\cap S_{-1}}|\ell_i'^{(\tilde{T}-1)}|}{|S_{-}\cap S_{1}|\cdot\max_{i\in S_{-}\cap S_{1}}|\ell_i'^{(\tilde{T}-1)}|}\geq1.
\end{equation*}
And it follows for both cases $\la\wb_{j,r}^{(\tilde{T}-1)},\bmu\ra\geq0$ and $\la\wb_{j,r}^{(\tilde{T}-1)},\bmu\ra<0$ that
\begin{equation}
    \gamma_{j,r}^{(\tilde{T})}\geq\gamma_{j,r}^{(\tilde{T}-1)},\label{ineq: gamma increases}
\end{equation}
and hence
\begin{equation*}
    \gamma_{j,r}^{(\tilde{T})}\geq\gamma_{j,r}^{(0)}=0.
\end{equation*}

For the other part of \eqref{ineq: range of gamma}, we prove a strengthened hypothesis that there exists a $i^*\in[n]$ with $y_{i^*}=j$ such that for $1\leq t\leq T^*$ we have that
\begin{equation*}
    \gamma_{j,r}^{(t)}/\zeta_{j,r,i^*}^{(t)}\leq C'n\|\bmu\|_2^2/\sigma_p^2 d,
\end{equation*}
and $i^*$ can be taken as any sample from set $S_{j,r}^{(0)}$ and $C'$ can be taken as $2C_2$.

Recall the update rule of $\gamma_{j,r}^{(t)}$ and $\omega_{j,r,i}^{(t)}$, we have
\begin{align*}
    &\gamma_{j,r}^{(\tilde{T})} = \gamma_{j,r}^{(\tilde{T}-1)} - \frac{\eta}{nm} \cdot \bigg[\sum_{i\in S_{+}} \ell_i'^{(\tilde{T}-1)} \ind(\la\wb_{j,r}^{(\tilde{T}-1)}, \hat{y}_{i} \cdot \bmu\ra\geq0) - \sum_{i\in S_{-}} \ell_i'^{(\tilde{T}-1)} \ind(\la\wb_{j,r}^{(\tilde{T}-1)}, \hat{y}_{i} \cdot \bmu\ra\geq0)\bigg] \cdot\|\bmu\|_2^2,\\
    &\zeta_{j,r,i}^{(\tilde{T})}=\zeta_{j,r,i}^{(\tilde{T}-1)}-\frac{\eta}{nm}\cdot\ell_i'^{(\tilde{T}-1)}\cdot\ind(\la\wb_{j,r}^{(\tilde{T}-1)},\bxi_i\ra \geq 0) \cdot \ind(y_i=j) \|\bxi_i\|_{2}^{2}.
\end{align*}

According to the fifth statement of Lemma \ref{lm: balanced logit}, for any $i^*\in S_{j,r}^{(0)}$ it holds that $j=y_{i^*}$ and $\la\wb_{j,r}^{(t)},\bxi_{i^*}\ra \geq 0$ for any $0\leq t\leq \tilde{T}-1$. Thus, we have
\begin{equation*}
    \zeta_{j,r,i^*}^{(\tilde{T})}=\zeta_{j,r,i^*}^{(\tilde{T}-1)}-\frac{\eta}{nm}\cdot\ell_{i^*}'^{(\tilde{T}-1)}\cdot\|\bxi_{i^*}\|_{2}^{2}\geq\zeta_{j,r,i^*}^{(\tilde{T}-1)}-\frac{\eta}{nm}\cdot\ell_{i^*}'^{(\tilde{T}-1)}\cdot\sigma_p^2 d/2.
\end{equation*}
For the update rule of $\gamma_{j,r}^{(\tilde{T})}$, we have
\begin{align*}
    \Bigg|\sum_{i\in S_{+}} \ell_i'^{(\tilde{T}-1)} \ind(\la\wb_{j,r}^{(\tilde{T}-1)}, \hat{y}_{i} \cdot \bmu\ra\geq0) - \sum_{i\in S_{-}} \ell_i'^{(\tilde{T}-1)} \ind(\la\wb_{j,r}^{(\tilde{T}-1)}, \hat{y}_{i} \cdot \bmu\ra\geq0)\Bigg|&\leq\sum_{i=1}^{n}|\ell_i'^{(\tilde{T}-1)}|\cdot\|\bmu\|_{2}^{2}\\
    &\leq C_2n\cdot|\ell_{i^*}'^{(\tilde{T}-1)}|\cdot\|\bmu\|_{2}^{2},
\end{align*}
where the first inequality is due to triangle inequality; the second inequality is due to the third statement of Lemma \ref{lm: balanced logit} where $C_2$ is a positive constant. Then, we have

\begin{equation*}
    \frac{\gamma_{j,r}^{(\tilde{T})}}{\zeta_{j,r,i^*}^{(\tilde{T})}}\leq\max\Bigg\{\frac{\gamma_{j,r}^{(\tilde{T}-1)}}{\zeta_{j,r,i^*}^{(\tilde{T}-1)}},\frac{C_2n\cdot|\ell_{i^*}'^{(\tilde{T}-1)}|\cdot\|\bmu\|_{2}^{2}}{|\ell_{i^*}'^{(\tilde{T}-1)}|\cdot\sigma_p^2 d/2}\Bigg\}=\max\Bigg\{\frac{\gamma_{j,r}^{(\tilde{T}-1)}}{\zeta_{j,r,i^*}^{(\tilde{T}-1)}},\frac{2C_2 n\|\bmu\|_2^2}{\sigma_p^2 d}\Bigg\}\leq \frac{2C_2 n\|\bmu\|_2^2}{\sigma_p^2 d},
\end{equation*}
where the last inequality is by $\gamma_{j,r}^{(\tilde{T}-1)}/\zeta_{j,r,i^*}^{(\tilde{T}-1)}\leq C'\hat{\gamma}=C'n\|\bmu\|_2^2/\sigma_p^2 d$ and $C'$ can be taken as $2C_2$, which completes the induction.
\end{proof}

By then, we have already proved Proposition \ref{proposition: range of gamma,zeta,omega}. Then, according to Lemma \ref{lm: balanced logit}, next proposition directly follows.

\begin{proposition}\label{proposition: balanced logit}
Under Condition \ref{condition:d_sigma0_eta}, for $0\leq t\leq T^*$, we have that
\begin{enumerate}
    \item $\sum_{r=1}^{m}\big[\zeta_{y_i,r,i}^{(t)}-\zeta_{y_k,r,k}^{(t)}\big]\leq\kappa$ for all $i,k\in[n]$. 
    \item $y_i\cdot f(\Wb^{(t)},\xb_i)-y_k\cdot f(\Wb^{(t)},\xb_k)\leq C_1$ for all $i,k\in[n]$, 
    \item $\ell_i'^{(t)}/\ell_k'^{(t)}\leq C_2=\exp(C_1)$ for all $i,k\in[n]$. 
    \item $S_i^{(0)} \subseteq S_i^{(t)}$, where $S_{i}^{(t)}:=\{r\in[m]:\la\wb_{y_i,r}^{(t)},\bxi_i\ra>0\}$, and hence $|S_{i}^{(t)}|\geq0.4m$ for all $i\in[n]$.
    \item $S_{j,r}^{(0)} \subseteq S_{j,r}^{(t)}$ , where $S_{j,r}^{(t)}:=\{i\in[n]:y_i=j,\la\wb_{j,r}^{(t)},\bxi_i\ra>0\}$, and hence $|S_{j,r}^{(t)}|\geq n/8$ for all $j\in\{\pm1\},r\in[m]$.
    \end{enumerate}
Here $\kappa$ and $C_1$ can be taken as 3.25 and 5 respectively. 
\end{proposition}

\section{Decoupling with a Two-Stage Analysis}
We utilize a two-stage analysis to decouple the complicated relations between the coefficients $\gamma_{j,r}^{(t)}$, $\zeta_{j,r,i}^{(t)}$ and $\omega_{j,r,i}^{(t)}$.  Intuitively, the initial neural network weights are small enough so that the neural network at initialization has constant level cross-entropy loss derivatives on all the training data: $\ell_i'^{(0)}=\ell'[y_i\cdot f(\Wb^{(0)},\xb_i)]=\Theta(1)$ for all $i\in[n]$. Motivated by this, we can consider the first stage of the training process where $\ell_i'^{(0)}=\Theta(1)$, in which case we can show significant scale differences among $\gamma_{j,r}^{(t)}$, $\zeta_{j,r,i}^{(t)}$ and $\omega_{j,r,i}^{(t)}$. Based on the result in the first stage, we then proceed to the second stage of the training process where the loss derivatives are no longer at a constant level and show that the training loss can be optimized to be arbitrarily small and meanwhile, the scale differences shown in the first learning stage remain the same throughout the training process. Recall that we denote $\alpha=4\log(T^*)$, $\beta=2\max_{i,j,r}\{|\la\wb_{j,r}^{(0)},\bmu\ra|,|\la\wb_{j,r}^{(0)},\bxi_i\ra|\}$ and $\mathrm{SNR}=\|\bmu\|_2/(\sigma_p\sqrt{d})$. We remind the readers that the proofs in this section are based on the results in Section \ref{section:decompositionproof}, which hold with high probability.

\subsection{First Stage}

\begin{lemma}\label{lm: first stage}
If we denote
\begin{equation*}
    n\cdot\mathrm{SNR}^{2}=\hat{\gamma},
\end{equation*}
then there exist
\begin{equation*}
    T_1=C_3\eta^{-1}nm\sigma_p^{-2}d^{-1}, T_2=C_4\eta^{-1}nm\sigma_p^{-2}d^{-1}
\end{equation*}
where $C_3=\Theta(1)$ is a large constant and $C_4=\Theta(1)$ is a small constant, such that
\begin{itemize}
    \item $\zeta_{j,r^*,i}^{(T_1)}\geq 2$ for any $r^*\in S_{i}^{(0)}=\{r\in[m]:\la\wb_{y_i,r}^{(0)},\bxi_i\ra>0\}$, $j\in\{\pm 1\}$ and $i\in[n]$ with $y_i=j$.
    \item $\max_{j,r}\gamma_{j,r}^{(t)}=O(\hat{\gamma})$ for all $0\leq t\leq T_1$.
    \item $\max_{j,r,i}|\omega_{j,r,i}^{(t)}|=\max\{O\big(\sqrt{\log(mn/\delta)}\cdot\sigma_0\sigma_p\sqrt{d}\big),O \big(n \sqrt{\log(n / \delta)} \log(T^*)/\sqrt{d} \big)\}$ for all $0\leq t\leq T_1$.
    \item $\min_{j,r}\gamma_{j,r}^{(t)}=\Omega(\hat{\gamma})$ for all $t\geq T_2$.
    \item $\max_{j,r}\zeta_{j,r,i}^{(T_1)}=O(1)$ for all $i\in[n]$.
\end{itemize}
\end{lemma}
\begin{proof}[Proof of Lemma \ref{lm: first stage}]
By Proposition \ref{proposition: range of gamma,zeta,omega}, we have that $\omega_{j,r,i}^{(t)}\geq-\beta-10n\sqrt{\frac{\log(6n^2/\delta)}{d}}\alpha$ for all $j\in\{\pm 1\}$, $r\in[m]$, $i\in[n]$ and $0\leq t\leq T^*$. According to Lemma \ref{lm: initialization inner products}, for $\beta$ we have
\begin{align*}
    \beta&=2\max_{i,j,r}\{|\la\wb_{j,r}^{(0)},\bmu\ra|,|\la\wb_{j,r}^{(0)},\bxi_i\ra|\}\\
    &\leq2\max\{\sqrt{2\log(12m/\delta)}\cdot\sigma_0\|\bmu\|_2,2\sqrt{\log(12mn/\delta)}\cdot\sigma_0\sigma_p\sqrt{d}\}\\
    &=O\big(\sqrt{\log(mn/\delta)}\cdot\sigma_0\sigma_p\sqrt{d}\big)\,
\end{align*}
where the last equality is by the first condition of Condition \ref{condition:d_sigma0_eta}.
Since $\omega_{j,r,i}^{(t)}\leq 0$, we have that
\begin{align*}
    \max_{j,r,i}|\omega_{j,r,i}^{(t)}|&=\max_{j,r,i}-\omega_{j,r,i}^{(t)}\\
    &\leq\beta+10\sqrt{\frac{\log(4n^2/\delta)}{d}}n\alpha\\
    &=\max \bigg \{O\big(\sqrt{\log(mn/\delta)}\cdot\sigma_0\sigma_p\sqrt{d}\big), O \big( \sqrt{\log(n / \delta)} \log(T^*) \cdot n/\sqrt{d} \big) \bigg\}.
\end{align*}
Next, for the growth of $\gamma_{j,r}^{(t)}$, we have following upper bound
\begin{align*}
    \gamma_{j,r}^{(t+1)} &= \gamma_{j,r}^{(t)} - \frac{\eta}{nm} \cdot \bigg[\sum_{i\in S_{+}} \ell_i'^{(t)} \sigma'(\la\wb_{j,r}^{(t)}, \hat{y}_{i} \cdot \bmu\ra) - \sum_{i\in S_{-}} \ell_i'^{(t)} \sigma'(\la\wb_{j,r}^{(t)}, \hat{y}_{i} \cdot \bmu\ra)\bigg]  \cdot \| \bmu \|_2^2\\
    &=\gamma_{j,r}^{(t)}-\frac{\eta}{nm}\cdot\sum_{i=1}^{n}\ell_i'^{(t)}\cdot\sigma'(\la\wb_{j,r}^{(t)},\hat{y}_i\cdot\bmu\ra)\|\bmu\|_2^2\\
    &\leq\gamma_{j,r}^{(t)}+\frac{\eta}{m}\cdot\|\bmu\|_2^2,
\end{align*}
where the inequality is by $|\ell'|\leq1$. Note that $\gamma_{j,r}^{(0)}=0$ and recursively use the inequality $t$ times we have
\begin{equation}
    \gamma_{j,r}^{(t)}\leq\frac{\eta t}{m}\cdot\|\bmu\|_2^2.\label{ineq: gamma upbound}
\end{equation}
Since $n\cdot\mathrm{SNR}^{2}=n\|\bmu\|_2^2/\sigma_p^2 d=\hat{\gamma}$, we have
\begin{equation*}
    T_1=C_3\eta^{-1}nm\sigma_p^{-2}d^{-1}=C_3\eta^{-1}m\|\bmu\|_2^{-2}\hat{\gamma}.
\end{equation*}
And it follows that
\begin{equation*}
    \gamma_{j,r}^{(t)}\leq\frac{\eta t}{m}\cdot\|\bmu\|_2^2\leq\frac{\eta T_1}{m}\cdot\|\bmu\|_2^2\leq C_3\hat{\gamma},
\end{equation*}
for all $0\leq t\leq T_1$. 

For $\zeta_{j,r,i}^{(t)}$, recall from \eqref{iterative equation3} that
\begin{equation*}
    \zeta_{j,r,i}^{(t+1)}=\zeta_{j,r,i}^{(t)}-\frac{\eta}{nm}\cdot\ell_i'^{(t)}\cdot\sigma'(\la\wb_{j,r}^{(t)},\bxi_i\ra)\cdot\ind(y_i=j)\|\bxi_i\|_2^2.
\end{equation*}
According to Proposition \ref{proposition: balanced logit}, for any $r^*\in S_{i}^{(0)}=\{r\in[m]:\la\wb_{y_i,r}^{(0)},\bxi_i\ra>0\}$, we have $\la\wb_{y_i,r^*}^{(t)},\bxi_i\ra>0$ for all $0\leq t\leq T^*$ and hence
\begin{equation*}
    \zeta_{y_i,r^*,i}^{(t+1)}=\zeta_{y_i,r^*,i}^{(t)}-\frac{\eta}{nm}\cdot\ell_i'^{(t)}\cdot\|\bxi_i\|_2^2.
\end{equation*}
Note that $\zeta_{y_i,r^*,i}^{(0)}=0$ and recursively use the equation $t$ times, we have
\begin{equation*}
    \zeta_{y_i,r^*,i}^{(t)}=-\frac{\eta}{nm}\cdot\sum_{s=0}^{t-1}\ell_i'^{(s)}\cdot\|\bxi_i\|_2^2.
\end{equation*}
For each $i$, denote by $T_1^{(i)}$ the last time in the period $[0,T_1]$ satisfying that $\max_{j,r}|\rho_{j,r,i}^{(t)}|\leq2$. Then for $0\leq t\leq T_1^{(i)}$, $\max_{j,r}\{|\zeta_{j,r,i}^{(t)}|,|\omega_{j,r,i}^{(t)}|\}=O(1)$ and $\max_{j,r}\gamma_{j,r}^{(t)}=O(1)$. Therefore, we know that $F_{-1}(\Wb^{(t)},\xb_i), F_{+1}(\Wb^{(t)},\xb_i)=O(1)$. Thus there exists a positive constant $C$ such that $-\ell_i'^{(t)}\geq C$ for $0\leq t\leq T_1^{(i)}$. Then we have
\begin{equation*}
    \zeta_{y_i,r^*,i}^{(t)}\geq\frac{C\eta\sigma_p^2 dt}{2nm}.
\end{equation*}
Therefore, $\zeta_{y_i,r^*,i}^{(t)}$ will reach 2 within
\begin{equation*}
    T_1=C_3\eta^{-1}nm\sigma_p^{-2}d^{-1}
\end{equation*}
iterations for any $r^*\in S_{i}^{(0)}$, where $C_3$ can be taken as $4/C$. 

Next, we will discuss the lower bound of the growth of $\gamma_{j,r}^{(t)}$. For $\zeta_{j,r,i}^{(t)}$, we have
\begin{align*}
    \zeta_{j,r,i}^{(t+1)}&=\zeta_{j,r,i}^{(t)}-\frac{\eta}{nm}\cdot\ell_i'^{(t)}\cdot\sigma'(\la\wb_{j,r}^{(t)},\bxi_i\ra)\cdot\ind(y_i=j)\|\bxi_i\|_2^2\leq \zeta_{j,r,i}^{(t)}+\frac{\eta}{nm}\|\bxi_i\|_2^2\leq\zeta_{j,r,i}^{(t)}+\frac{3\eta\sigma_p^2 d}{2nm},
\end{align*}
where the first inequality is by $-\ell_i'\in(0,1)$ and $\sigma'\in\{0,1\}$; the second inequality is by Lemma \ref{lm: data inner products}. According to \eqref{ineq: gamma upbound} and $\zeta_{j,r,i}^{(0)}=0$, it follows that
\begin{equation}
    \zeta_{j,r,i}^{(t)}\leq\frac{3\eta\sigma_p^2 dt}{2nm}, \gamma_{j,r}^{(t)}\leq\frac{\eta t}{m}\cdot\|\bmu\|_2^2.\label{ineq6}
\end{equation}
Therefore, $\max_{j,r,i}\zeta_{j,r,i}^{(t)}$ will be smaller than $1$ and $\gamma_{j,r}^{(t)}$ smaller than $\Theta(n\|\bmu\|_2^2/\sigma_p^2 d)=\Theta(n\cdot\mathrm{SNR}^2)=\Theta(\hat{\gamma})=O(1)$ within
\begin{equation*}
    T_2=C_4\eta^{-1}nm\sigma_p^{-2} d^{-1}
\end{equation*}
iterations, where $C_4$ can be taken as $2/3$. Therefore, we know that $F_{-1}(\Wb^{(t)},\xb_i), F_{+1}(\Wb^{(t)},\xb_i)=O(1)$ in $[0,T_2]$. Thus there exists a positive constant $C$ such that $-\ell_i'^{(t)}\geq C$ for $0\leq t\leq T_2$.

Recall that we denote $\{i\in[n]|y_i=y\}$ as $S_{y}$. For the growth of $\gamma_{j,r}^{(t)}$, if $\la\wb_{j,r}^{(t)},\bmu\ra\geq0$, we have
\begin{equation}
    \begin{aligned}
    \gamma_{j,r}^{(t+1)} &= \gamma_{j,r}^{(t)} - \frac{\eta}{nm} \cdot \bigg[\sum_{i\in S_{+}} \ell_i'^{(t)} \sigma'(\la\wb_{j,r}^{(t)}, \hat{y}_{i} \cdot \bmu\ra) - \sum_{i\in S_{-}} \ell_i'^{(t)} \sigma'(\la\wb_{j,r}^{(t)}, \hat{y}_{i} \cdot \bmu\ra)\bigg]  \cdot \| \bmu \|_2^2\\
    &=\gamma_{j,r}^{(t)} - \frac{\eta}{nm}\cdot\bigg[\sum_{i\in S_{+}\cap S_{1}}\ell_i'^{(t)}-\sum_{i\in S_{-}\cap S_{-1}}\ell_i'^{(t)}\bigg]\cdot \| \bmu \|_2^2\\
    &\geq\gamma_{j,r}^{(t)} + \frac{\eta}{nm}\cdot(C|S_{+}\cap S_{1}|-|S_{-}\cap S_{-1}|)\cdot \| \bmu \|_2^2.\label{ineq9}
    \end{aligned}
\end{equation}
And if $\la\wb_{j,r}^{(t)},\bmu\ra<0$, we have
\begin{equation}
    \begin{aligned}
    \gamma_{j,r}^{(t+1)}&=\gamma_{j,r}^{(t)} - \frac{\eta}{nm}\cdot\bigg[\sum_{i\in S_{+}\cap S_{-1}}\ell_i'^{(t)}-\sum_{i\in S_{-}\cap S_{1}}\ell_i'^{(t)}\bigg]\cdot \| \bmu \|_2^2\\
    &\geq\gamma_{j,r}^{(t)} + \frac{\eta}{nm}\cdot(C|S_{+}\cap S_{-1}|-|S_{-}\cap S_{1}|)\cdot \| \bmu \|_2^2.\label{ineq10}
    \end{aligned}
\end{equation}

According to Lemma \ref{lm: estimate S cap S}, under event $\cE_{\mathrm{prelim}}$, we have
\begin{equation}
\begin{aligned}
    &\frac{|S_{+}\cap S_{1}|}{|S_{-}\cap S_{-1}|},\frac{|S_{+}\cap S_{-1}|}{|S_{-}\cap S_{1}|}\geq\frac{(1-p)n-\sqrt{2n\log(8/\delta)}}{pn+\sqrt{2n\log(8/\delta)}},\\
    &|S_{+}\cap S_{1}|,|S_{+}\cap S_{-1}|\geq(1-p)n-\sqrt{2n\log(8/\delta)}.\label{ineq11}
\end{aligned}
\end{equation}

As long as $p<C/6$ and $n\geq 72C^{-2}\log(8/\delta)$, it follows that
\begin{align*}
    &\frac{|S_{+}\cap S_{1}|}{|S_{-}\cap S_{-1}|},\frac{|S_{+}\cap S_{-1}|}{|S_{-}\cap S_{1}|}\geq2/C,\\
    &|S_{+}\cap S_{1}|,|S_{+}\cap S_{-1}|\geq n/4.
\end{align*}

Therefore, we have
\begin{equation}
    \begin{aligned}
    &\gamma_{j,r}^{(t+1)}\geq\gamma_{j,r}^{(t)} + \frac{C\eta}{2nm}\cdot|S_{+}\cap S_{-1}|\cdot \| \bmu \|_2^2\geq\gamma_{j,r}^{(t)} + \frac{C\eta}{8m}\cdot \| \bmu \|_2^2, \text{ if }\la\wb_{j,r}^{(t)},\bmu\ra\geq0,\\
    &\gamma_{j,r}^{(t+1)}\geq\gamma_{j,r}^{(t)} + \frac{C\eta}{2nm}\cdot|S_{+}\cap S_{1}|\cdot \| \bmu \|_2^2\geq\gamma_{j,r}^{(t)} + \frac{C\eta}{8m}\cdot \| \bmu \|_2^2, \text{ if }\la\wb_{j,r}^{(t)},\bmu\ra<0.\label{ineq12}
    \end{aligned}
\end{equation}
Note that $\gamma_{j,r}^{(0)}=0$, it follows that
\begin{equation*}
    \gamma_{j,r}^{(t)}\geq\frac{C\|\bmu\|_2^2\eta t}{8m},\,\gamma_{j,r}^{(T_2)}\geq\frac{CC_4 n\|\bmu\|_2^2}{8\sigma_p^2 d}=\Theta(n\cdot\mathrm{SNR}^2)=\Theta(\hat{\gamma}).
\end{equation*}
Note that we have proved \eqref{ineq: gamma increases} in Lemma \ref{proposition: range of gamma,zeta,omega} that $\gamma_{j,r}^{(t)}$ is increasing for $0\leq t\leq T^*$, thus we have
\begin{equation*}
    \gamma_{j,r}^{(t)}=\Omega(\hat{\gamma})
\end{equation*}
for $T_2 \leq t\leq T^*$. And it follows directly from \eqref{ineq6} that
\begin{equation*}
    \zeta_{j,r,i}^{(T_1)}\leq\frac{3\eta\sigma_p^2 dT_1}{2nm}=\frac{3C_3}{2},\,\zeta_{j,r,i}^{(T_1)}=O(1),
\end{equation*}
which completes the proof.
\end{proof}
\subsection{Second Stage}
By the signal-noise decomposition, at the end of the first stage, we have
\begin{equation*}
    \wb_{j,r}^{(T_1)}=\wb_{j,r}^{(0)}+j\cdot\gamma_{j,r}^{(T_1)}\cdot\frac{\bmu}{\|\bmu\|_2^2}+\sum_{i=1}^{n}\zeta_{j,r,i}^{(T_1)}\cdot\frac{\bxi_i}{\|\bxi_i\|_2^2}+\sum_{i=1}^{n}\omega_{j,r,i}^{(T_1)}\cdot\frac{\bxi_i}{\|\bxi_i\|_2^2}
\end{equation*}
for $j\in[\pm1]$ and $r\in[m]$. By the results we get in the first stage, we know that at the beginning of this stage, we have the following property holds:
\begin{itemize}
    \item $\zeta_{j,r^*,i}^{(T_1)}\geq 2$ for any $r^*\in S_{i}^{(0)}=\{r\in[m]:\la\wb_{y_i,r}^{(0)},\bxi_i\ra>0\}$, $j\in\{\pm 1\}$ and $i\in[n]$ with $y_i=j$.
    \item $\max_{j,r,i}|\omega_{j,r,i}^{(T_1)}|=\max\{O\big(\sqrt{\log(mn/\delta)}\cdot\sigma_0\sigma_p\sqrt{d}\big),O \big(n \sqrt{\log(n / \delta)} \log(T^*)/\sqrt{d} \big)\}$.
    \item $\gamma_{j,r}^{(T_1)}=\Theta(\hat{\gamma})$ for any $j\in\{\pm 1\},r\in[m]$.
\end{itemize}
where $\hat{\gamma}=n\cdot\mathrm{SNR}^2$. Now we choose $\Wb^*$ as follows 
\begin{equation*}
    \wb_{j,r}^{*}=\wb_{j,r}^{(0)}+5\log(2/\epsilon)\Big[\sum_{i=1}^{n}\ind(j=y_i)\cdot\frac{\bxi_i}{\|\bxi_i\|_2^2}\Big].
\end{equation*}
\begin{lemma}\label{lm: W difference}
    Under the same conditions as Theorem \ref{thm:signal_learning_main}, we have that $\|\Wb^{(T_1)}-\Wb^*\|_{F}\leq\tilde{O}(m^{1/2}n^{1/2}\sigma_p^{-1}d^{-1/2})$.
\end{lemma}
\begin{proof}[Proof of Lemma \ref{lm: W difference}]
We have
\begin{align*}
    \|\Wb^{(T_1)}-\Wb^*\|_{F}&\leq\|\Wb^{(T_1)}-\Wb^{(0)}\|_{F}+\|\Wb^{*}-\Wb^{(0)}\|_{F}\\
    &\leq O(\sqrt{m})\max_{j,r}\gamma_{j,r}^{(T_1)}\|\bmu\|_2^{-1}+O(\sqrt{m})\max_{j,r}\bigg\|\sum_{i=1}^{n}\zeta_{j,r,i}^{(T_1)}\cdot\frac{\bxi_i}{\|\bxi_i\|_2^2}+\sum_{i=1}^{n}\omega_{j,r,i}^{(T_1)}\cdot\frac{\bxi_i}{\|\bxi_i\|_2^2}\bigg\|_2\\
    &\qquad+O(m^{1/2}n^{1/2}\log(1/\epsilon)\sigma_p^{-1}d^{-1/2})\\
    &=O(m^{1/2}\hat{\gamma}\|\bmu\|_2^{-1})+\tilde{O}(m^{1/2}n^{1/2}\sigma_p^{-1}d^{-1/2})+O(m^{1/2}n^{1/2}\log(1/\epsilon)\sigma_p^{-1}d^{-1/2})\\
    &=O(m^{1/2}n\cdot\mathrm{SNR}\cdot\sigma_p^{-1}d^{-1/2})+\tilde{O}(m^{1/2}n^{1/2}\log(1/\epsilon)\sigma_p^{-1}d^{-1/2})\\
    &=\tilde{O}(m^{1/2}n^{1/2}\sigma_p^{-1}d^{-1/2}),
\end{align*}
where the first inequality is by triangle inequality, the second inequality and the first equality are by our decomposition of $\Wb^{(T_1)}$, $\Wb^*$ and Lemma \ref{lm: data inner products}; the second equality is by $n\cdot\mathrm{SNR}^2=\Theta(\hat{\gamma})$ and $\mathrm{SNR}=\|\bmu\|/\sigma_pd^{1/2}$; the third equality is by $n^{1/2}\cdot\mathrm{SNR}=O(1)$.
\end{proof}

\begin{lemma}\label{lm: gradient W alignment}
Under the same conditions as Theorem \ref{thm:signal_learning_main}, we have that
\begin{equation*}
    y_i\la\nabla f(\Wb^{(t)},\xb_i),\Wb^*\ra\geq\log(2/\epsilon)
\end{equation*}
for all $T_1\leq t\leq T^*$.
\end{lemma}
\begin{proof}[Proof of Lemma \ref{lm: gradient W alignment}]
Recall that $f(\Wb^{(t)})=(1/m)\sum_{j,r}j\cdot[\sigma(\la\wb_{j,r},y_i\cdot\bmu\ra)+\sigma(\la\wb_{j,r},\bxi_i\ra)]$, thus we have
\begin{align}
    &y_i\la\nabla f(\Wb^{(t)},\xb_i),\Wb^*\ra\notag\\
    &=\frac{1}{m}\sum_{j,r}\sigma'(\la\wb_{j,r}^{(t)},\hat{y}_i\bmu\ra)\la\bmu,j\wb_{j,r}^{*}\ra+\frac{1}{m}\sum_{j,r}\sigma'(\la\wb_{j,r}^{(t)},\bxi_i\ra)\la y_i\bxi_i,j\wb_{j,r}^*\ra\notag\\
    &=\frac{1}{m}\sum_{j,r}\sum_{i'=1}^{n}\sigma'(\la\wb_{j,r}^{(t)},\bxi_i\ra)5\log(2/\epsilon)\ind(j=y_{i'})\cdot\frac{\la\bxi_{i'},\bxi_{i}\ra}{\|\bxi_{i'}\|_2^2}\notag\\
    &\qquad+\frac{1}{m}\sum_{j,r}\sum_{i'=1}^{n}\sigma'(\la\wb_{j,r}^{(t)},\hat{y}_{i}\bmu\ra)5\log(2/\epsilon)\ind(j=y_{i'})\cdot\frac{\la\bmu,\bxi_{i'}\ra}{\|\bxi_{i'}\|_2^2}\notag\\
    &\qquad+\frac{1}{m}\sum_{j,r}\sigma'(\la\wb_{j,r}^{(t)},\hat{y}_i\bmu\ra)\la\bmu,j\wb_{j,r}^{(0)}\ra+\frac{1}{m}\sum_{j,r}\sigma'(\la\wb_{j,r}^{(t)},\bxi_i\ra)\la y_i\bxi_i,j\wb_{j,r}^{(0)}\ra\notag\\
    &\geq\frac{1}{m}\sum_{j=y_i,r}\sigma'(\la\wb_{j,r}^{(t)},\bxi_i\ra)5\log(2/\epsilon)-\frac{1}{m}\sum_{j,r}\sum_{i'\neq i}\sigma'(\la\wb_{j,r}^{(t)},\bxi_i\ra)5\log(2/\epsilon)\cdot\frac{|\la\bxi_{i'},\bxi_{i}\ra|}{\|\bxi_{i'}\|_2^2}\notag\\
    &\qquad-\frac{1}{m}\sum_{j,r}\sum_{i'=1}^{n}\sigma'(\la\wb_{j,r}^{(t)},\hat{y}_{i}\bmu\ra)5\log(2/\epsilon)\cdot\frac{|\la\bmu,\bxi_{i'}\ra|}{\|\bxi_{i'}\|_2^2}\notag\\
    &\qquad-\frac{1}{m}\sum_{j,r}\sigma'(\la\wb_{j,r}^{(t)},\hat{y}_i\bmu\ra)O\big(\sqrt{\log(m/\delta)}\cdot\sigma_0\|\bmu\|_2\big)-\frac{1}{m}\sum_{j,r}\sigma'(\la\wb_{j,r}^{(t)},\bxi_i\ra)O\big(\sqrt{\log(mn/\delta)}\cdot\sigma_0\sigma_p\sqrt{d}\big)\notag\\
    &\geq\underbrace{\frac{1}{m}\sum_{j=y_i,r}\sigma'(\la\wb_{j,r}^{(t)},\bxi_i\ra)5\log(2/\epsilon)}_{I_9}-\underbrace{\frac{1}{m}\sum_{j,r}\sigma'(\la\wb_{j,r}^{(t)},\bxi_i\ra)5\log(2/\epsilon)O\big(n\sqrt{\log(n/\delta)}/\sqrt{d}\big)}_{I_{10}}\notag\\
    &\qquad-\underbrace{\frac{1}{m}\sum_{j,r}\sigma'(\la\wb_{j,r}^{(t)},\hat{y}_i\bmu\ra)5\log(2/\epsilon)O\big(n\sqrt{\log(n/\delta)}\cdot\mathrm{SNR}\cdot d^{-1/2}\big)}_{I_{11}}\label{ineq1}\\
    &\qquad-\underbrace{\frac{1}{m}\sum_{j,r}\sigma'(\la\wb_{j,r}^{(t)},y_i\bmu\ra)O\big(\sqrt{\log(m/\delta)}\cdot\sigma_0\|\bmu\|_2\big)}_{I_{12}}-\underbrace{\frac{1}{m}\sum_{j,r}\sigma'(\la\wb_{j,r}^{(t)},\bxi_i\ra)O\big(\sqrt{\log(mn/\delta)}\cdot\sigma_0\sigma_p\sqrt{d}\big)}_{I_{14}},\notag
\end{align}
where the first inequality is by Lemma \ref{lm: initialization inner products} and the last inequality is by Lemma \ref{lm: data inner products}. Next, we will bound the inner-product terms in \eqref{ineq1} respectively. For $I_{10}$, $I_{11}$, $I_{12}$, $I_{14}$, note that $\sigma'\in\{0,1\}$ we have that
\begin{equation}\label{ineq: I9-11 upper bound}
    \begin{aligned}
        &|I_{10}|\leq \log(2/\epsilon)O\big(n\sqrt{\log(n/\delta)}/\sqrt{d}\big),\, |I_{11}|\leq \log(2/\epsilon)O\big(n\sqrt{\log(n/\delta)}\cdot\mathrm{SNR}\cdot d^{-1/2}\big),\\
        &|I_{12}|\leq O\big(\sqrt{\log(m/\delta)}\cdot\sigma_0\|\bmu\|_2\big),\, |I_{14}|\leq O\big(\sqrt{\log(mn/\delta)}\cdot\sigma_0\sigma_p\sqrt{d}\big).
    \end{aligned}
\end{equation}
For $j=y_i$ and $r\in S_{i}^{(0)}$, according to Lemma \ref{lm: inner product range}, we have
\begin{align*}
    \la\wb_{j,r}^{(t)},\bxi_i\ra&\geq\la\wb_{j,r}^{(0)},\bxi_i\ra+\zeta_{j,r,i}^{(t)}-5n\sqrt{\frac{\log(4n^2/\delta)}{d}}\alpha\\
    &\geq 2-\beta-5n\sqrt{\frac{\log(4n^2/\delta)}{d}}\alpha\\
    &\geq1
\end{align*}
where the first inequality is by Lemma \ref{lm: inner product range}; the last inequality is by $\beta\leq 0.5$ and $5n\sqrt{\frac{\log(4n^2/\delta)}{d}}\leq 0.5$. Therefore, for $I_{9}$, according to the fourth statement of Proposition \ref{proposition: balanced logit}, we have
\begin{equation}\label{ineq: I8 lower bound}
    I_9\geq\frac{1}{m}|S_{i}^{(t)}|5\log(2/\epsilon)\geq2\log(2/\epsilon).
\end{equation}
By plugging \eqref{ineq: I9-11 upper bound} and \eqref{ineq: I8 lower bound} into \eqref{ineq1} and according to triangle inequality we have
\begin{equation*}
    y_i\la\nabla f(\Wb^{(t)},\xb_i),\Wb^*\ra\geq I_9-|I_{10}|-|I_{11}|-|I_{12}|-|I_{14}|\geq\log(2/\epsilon),
\end{equation*}
which completes the proof.
\end{proof}

\begin{lemma}\label{lm: gradient upbound}
Under Condition~\ref{condition:d_sigma0_eta}, for $0 \leq t\leq T^{*}$, the following result holds.
\begin{align*}
\|\nabla L_{S}(\Wb^{(t)})\|_{F}^{2} \leq  O(\max\{\|\bmu\|_{2}^{2}, \sigma_{p}^{2}d\}) L_{S}(\Wb^{(t)}).   
\end{align*}
\end{lemma}
\begin{proof}[Proof of Lemma~\ref{lm: gradient upbound}]
We first prove that 
\begin{align}
 \|\nabla f(\Wb^{(t)}, \xb_{i}\big)\|_{F} = O(\max\{\|\bmu\|_{2}, \sigma_{p}\sqrt{d}\}).  \label{eq: exptail2}  
\end{align}
Without loss of generality,  
we suppose that $\hat{y}_{i} = 1$ and $\xb_{i} = [\bmu^{\top},\bxi_{i}]$. Then we have that 
\begin{align*}
\|\nabla f(\Wb^{(t)}, \xb_{i}) \|_{F}&\leq \frac{1}{m}\sum_{j,r}\bigg\|  \big[\sigma'(\la\wb_{j,r}^{(t)},\bmu\ra)\bmu + \sigma'(\la\wb_{j,r}^{(t)}, \bxi_i\ra)\bxi_i\big] \bigg\|_{2}\\
&\leq \frac{1}{m}\sum_{j,r}\sigma'(\la\wb_{j,r}^{(t)}, \bmu\ra)\|\bmu\|_{2} +   \frac{1}{m}\sum_{j,r}\sigma'(\la\wb_{j,r}^{(t)}, \bxi_{i}\ra)\|\bxi_{i}\|_{2}\\
&\leq 4\max\{\|\bmu\|_{2}, 2\sigma_{p}\sqrt{d}\} ,
\end{align*}
where the first and second inequalities are by triangle inequality, the third inequality is by Lemma~\ref{lm: data inner products} and $\sigma' \leq 1$. Now we can upper bound the gradient norm $\|\nabla L_{S}(\Wb^{(t)})\|_{F}$ as follows,
\begin{align*}
\|\nabla L_{S}(\Wb^{(t)})\|_{F}^{2} &\leq \bigg[\frac{1}{n}\sum_{i=1}^{n}\ell'\big(y_{i}f(\Wb^{(t)},\xb_{i})\big)\|\nabla f(\Wb^{(t)}, \xb_{i})\|_{F}\bigg]^{2}\\
&\leq \bigg[\frac{1}{n}\sum_{i=1}^{n}O(\max\{\|\bmu\|_{2}^{2}, \sigma_{p}^{2}d\})-\ell'\big(y_{i}f(\Wb^{(t)},\xb_{i})\big)\bigg]^{2}\\
&\leq O(\max\{\|\bmu\|_{2}^{2}, \sigma_{p}^{2}d\}) \cdot \frac{1}{n}\sum_{i=1}^{n}-\ell'\big(y_{i}f(\Wb^{(t)},\xb_{i})\big)\\
&\leq O(\max\{\|\bmu\|_{2}^{2}, \sigma_{p}^{2}d\}) L_{S}(\Wb^{(t)}),
\end{align*}
where the first inequality is by triangle inequality, the second inequality is by \eqref{eq: exptail2}, the third inequality is by Cauchy-Schwartz inequality and the last inequality is due to the property of the cross entropy loss $-\ell' \leq \ell$.
\end{proof}

\begin{lemma}\label{lm: F-norm difference}
Under the same conditions as Theorem \ref{thm:signal_learning_main}, we have that
\begin{equation*}
    \|\Wb^{(t)}-\Wb^*\|_{F}^{2}-\|\Wb^{(t+1)}-\Wb^*\|_{F}^{2}\geq\eta L_{S}(\Wb^{(t)})-\eta\epsilon
\end{equation*}
for all $T_1\leq t\leq T^*$.
\end{lemma}
\begin{proof}[Proof of Lemma \ref{lm: F-norm difference}]
We have
\begin{align*}
    &\|\Wb^{(t)}-\Wb^*\|_{F}^{2}-\|\Wb^{(t+1)}-\Wb^*\|_{F}^{2}\\
    &=2\eta\la\nabla L_{S}(\Wb^{(t)}),\Wb^{(t)}-\Wb^*\ra-\eta^2\|\nabla L_{S}(\Wb^{(t)})\|_{F}^{2}\\
    &=\frac{2\eta}{n}\sum_{i=1}^{n}\ell_i'^{(t)}[y_i f(\Wb^{(t)},\xb_i)-\la\nabla f(\Wb^{(t)},\xb_i),\Wb^*\ra]-\eta^2\|\nabla L_{S}(\Wb^{(t)})\|_{F}^{2}\\
    &\geq\frac{2\eta}{n}\sum_{i=1}^{n}\ell_i'^{(t)}[y_i f(\Wb^{(t)},\xb_i)-\log(2/\epsilon)]-\eta^2\|\nabla L_{S}(\Wb^{(t)})\|_{F}^{2}\\
    &\geq\frac{2\eta}{n}\sum_{i=1}^{n}[\ell\big(y_i f(\Wb^{(t)},\xb_i)\big)-\epsilon/2]-\eta^2\|\nabla L_{S}(\Wb^{(t)})\|_{F}^{2}\\
    &\geq\eta L_{S}(\Wb^{(t)})-\eta\epsilon,
\end{align*}
where the first inequality is by Lemma \ref{lm: gradient W alignment}; the second inequality is due to the convexity of the
cross entropy function; the last inequality is due to Lemma \ref{lm: gradient upbound}.
\end{proof}


\begin{lemma}\label{lm: loss upper bound}
Under the same conditions as Theorem \ref{thm:signal_learning_main}, for all $T_1\leq t\leq T^{*}$, we have $\max_{j,r,i}|\omega_{j,r,i}^{(t)}|=\max\big\{O\big(\sqrt{\log(mn/\delta)}\cdot\sigma_0\sigma_p\sqrt{d}\big),O \big(n \sqrt{\log(n / \delta)} \log(T^*)/\sqrt{d} \big)\big\}$.
Besides, 
\begin{equation*}
    \frac{1}{t-T_1+1}\sum_{s=T_1}^{t}L_S(\Wb^{(s)})\leq\frac{\|\Wb^{(T_1)}-\Wb^*\|_F^2}{\eta(t-T_1+1)}+\epsilon
\end{equation*}
for all $T_1\leq t\leq T^{*}$. Therefore, we can find an iterate with training loss smaller than $2\epsilon$ within $T=T_1+\Big\lfloor \|\Wb^{(T_1)}-\Wb^*\|_F^2 / (\eta\epsilon) \Big \rfloor=T_1+\tilde{O}(\eta^{-1}\epsilon^{-1}mnd^{-1}\sigma_p^{-2})$ iterations.
\end{lemma}
\begin{proof}[Proof of Lemma \ref{lm: loss upper bound}]
Note that $\max_{j,r,i}|\omega_{j,r,i}^{(t)}|=\max\big\{O\big(\sqrt{\log(mn/\delta)}\cdot\sigma_0\sigma_p\sqrt{d}\big),O \big(n \sqrt{\log(n / \delta)} \log(T^*)/\sqrt{d} \big)\big\}$ can be proved in the same way as Lemma \ref{lm: first stage}, we eliminate the proof details here. For any $t\in[T_1,T]$,
by taking a summation of the inequality in Lemma \ref{lm: F-norm difference} and dividing $(t-T_1+1)$ on both sides, we obtain that
\begin{equation*}
    \frac{1}{t-T_1+1}\sum_{s=T_1}^{t}L_{S}(\Wb^{(s)})\leq\frac{\|\Wb^{(T_1)}-\Wb^*\|_F^2}{\eta(t-T_1+1)}+\epsilon
\end{equation*}
for all $T_1\leq t\leq T$. According to the definition of $T$, we have
\begin{equation*}
    \frac{1}{T-T_1+1}\sum_{s=T_1}^{T}L_{S}(\Wb^{(s)})\leq2\epsilon.
\end{equation*}
Then there exists iteration $T_1\leq t\leq T$ such that the training loss is smaller than $\epsilon$.
\end{proof}

Besides, we have the following lemma about the order of $\zeta_{j,r,i}^{(t)},\gamma_{j,r}^{(t)}$ ratio when training loss is smaller than $\epsilon$. And this lemma will help us prove the theorem about test error.

\begin{lemma}\label{lm: zeta,gamma ratio}
Under the same conditions as Theorem \ref{thm:signal_learning_main}, we have
\begin{equation}
    \sum_{i=1}^{n}\zeta_{j,r,i}^{(t)}/\gamma_{j',r'}^{(t)}=\Theta(\mathrm{SNR^{-2}})\label{eq: order of zeta,gamma ratio}
\end{equation}
for all $j, j'\in\{\pm 1\}$, $r, r'\in[m]$ and $T_1\leq t\leq T^*$.
\end{lemma}


\begin{proof}[Proof of Lemma \ref{lm: zeta,gamma ratio}]
We will prove this lemma by using induction. We first verify that \eqref{eq: order of zeta,gamma ratio} holds for $t=T_1$. By Lemma \ref{lm: first stage}, we have $\gamma_{j',r'}^{(T_1)}=\Theta(\hat{\gamma})=\Theta(n\cdot\mathrm{SNR}^2)$ and $\sum_{i=1}^{n}\zeta_{j,r,i}^{(T_1)}=\Theta(n)$, and \eqref{eq: order of zeta,gamma ratio} follows directly. Now suppose
that there exists $\tilde{T}\in[T_1,T^*]$ such that $\sum_{i=1}^{n}\zeta_{j,r,i}^{(t)}/\gamma_{j',r'}^{(t)}=\Theta(\mathrm{SNR}^2)$ for all $t\in[T_1,\tilde{T}-1]$. Then for $\zeta_{j,r,i}^{(t)}$, according to Lemma \ref{lm: coefficient iterative}, we have
\begin{align*}
    \zeta_{j,r,i}^{(t+1)}&=\zeta_{j,r,i}^{(t)}-\frac{\eta}{nm}\cdot\ell_i'^{(t)}\cdot\sigma'(\la\wb_{j,r}^{(t)},\bxi_i\ra)\cdot\ind(y_i=j)\|\bxi_i\|_2^2.\\
    \gamma_{j',r'}^{(t+1)} &= \gamma_{j',r'}^{(t)} - \frac{\eta}{nm} \cdot \bigg[\sum_{i\in S_{+}} \ell_i'^{(t)} \sigma'(\la\wb_{j',r'}^{(t)}, \hat{y}_{i} \cdot \bmu\ra) - \sum_{i\in S_{-}} \ell_i'^{(t)} \sigma'(\la\wb_{j',r'}^{(t)}, \hat{y}_{i} \cdot \bmu\ra)\bigg]  \cdot \| \bmu \|_2^2
\end{align*}
It follows that
\begin{equation}
    \begin{aligned}
    \sum_{i=1}^{n}\zeta_{j,r,i}^{(\tilde{T})}=\sum_{i:y_i=j}\zeta_{j,r,i}^{(\tilde{T})}&=\sum_{i: y_i=j}\zeta_{j,r,i}^{(\tilde{T}-1)}-\frac{\eta}{nm}\cdot\sum_{i: y_i=j}\ell_i'^{(\tilde{T}-1)}\cdot\sigma'(\la\wb_{j,r}^{(\tilde{T}-1)},\bxi_i\ra)\|\bxi_i\|_2^2\\
    &=\sum_{i=1}^{n}\zeta_{j,r,i}^{(\tilde{T}-1)}-\frac{\eta}{nm}\cdot\sum_{i\in S_{j,r}^{(\tilde{T}-1)}}\ell_i'^{(\tilde{T}-1)}\|\bxi_i\|_2^2\\
    &\geq\sum_{i=1}^{n}\zeta_{j,r,i}^{(\tilde{T}-1)}+\frac{\eta\sigma_p^2 d}{16m}\cdot\min_{i\in S_{j,r}^{(\tilde{T}-1)}}|\ell_i'^{(\tilde{T}-1)}|,
    \end{aligned}\label{ineq7}
\end{equation}
where the last equality is by the definition of $S_{j,r}^{(\tilde{T}-1)}$ as $\{i\in[n]:y_i=j,\la\wb_{j,r}^{(\tilde{T}-1)},\bxi_i\ra>0\}$; the last inequality is by Lemma \ref{lm: data inner products} and the fifth statement of Proposition \ref{proposition: balanced logit}. And 
\begin{equation}
    \begin{aligned}
    \gamma_{j',r'}^{(\tilde{T})} &\leq \gamma_{j',r'}^{(\tilde{T}-1)} - \frac{\eta}{nm} \cdot \sum_{i\in S_{+}} \ell_i'^{(\tilde{T}-1)} \sigma'(\la\wb_{j',r'}^{(\tilde{T}-1)}, \hat{y}_{i} \cdot \bmu\ra)\cdot \|\bmu\|_2^2\\
    &\leq\gamma_{j',r'}^{(\tilde{T}-1)}+\frac{\eta\|\bmu\|_2^2}{m}\cdot\max_{i\in S_{+}}|\ell_i'^{(\tilde{T}-1)}|.\label{ineq8}
    \end{aligned}
\end{equation}
According to the third statement of Proposition \ref{proposition: balanced logit}, we have $\max_{i\in S_{+}}|\ell_i'^{(\tilde{T}-1)}|\leq C_2\min_{i\in S_{j,r}^{(\tilde{T}-1)}}|\ell_i'^{(\tilde{T}-1)}|$. Then by combining \eqref{ineq7} and \eqref{ineq8}, we have
\begin{equation}
    \frac{\sum_{i=1}^{n}\zeta_{j,r,i}^{(\tilde{T})}}{\gamma_{j',r'}^{(\tilde{T})}}\geq\min\Bigg\{\frac{\sum_{i=1}^{n}\zeta_{j,r,i}^{(\tilde{T}-1)}}{\gamma_{j',r'}^{(\tilde{T}-1)}},\frac{\sigma_p^2 d}{16C_2\|\bmu\|_2^2}\Bigg\}=\Theta(\mathrm{SNR}^{-2}).\label{ineq16}
\end{equation}

On the other hand, according to \eqref{ineq7} and by Lemma \ref{lm: data inner products}, we have
\begin{equation}
    \sum_{i=1}^{n}\zeta_{j,r,i}^{(\tilde{T})}\leq\sum_{i=1}^{n}\zeta_{j,r,i}^{(\tilde{T}-1)}+\frac{9\eta\sigma_p^2 d}{8m}\cdot\max_{i\in S_{j,r}^{(\tilde{T}-1)}}|\ell_i'^{(\tilde{T}-1)}|,\label{ineq14}
\end{equation}
where the inequality is by $|S_{j,r}^{(\tilde{T}-1)}|\leq |S_{j}|\leq 3n/4$. And by arguing in a similar way as \eqref{ineq9}, \eqref{ineq10}, \eqref{ineq11} and \eqref{ineq12}, we can obtain that as long as $q<C_2/6$ and $n\geq 72C_{2}^{-2}\log(8/\delta)$, it holds that
\begin{equation*}
    \sum_{i\in S_{+}} |\ell_i'^{(\tilde{T}-1)}| \sigma'(\la\wb_{j',r'}^{(\tilde{T}-1)}, \hat{y}_{i} \cdot \bmu\ra)\geq2\sum_{i\in S_{-}} |\ell_i'^{(\tilde{T}-1)}| \sigma'(\la\wb_{j',r'}^{(\tilde{T}-1)}, \hat{y}_{i} \cdot \bmu\ra)
\end{equation*}
and hence
\begin{align*}
    \gamma_{j',r'}^{(\tilde{T})} &= \gamma_{j',r'}^{(\tilde{T}-1)} - \frac{\eta}{nm} \cdot \bigg[\sum_{i\in S_{+}} \ell_i'^{(\tilde{T}-1)} \sigma'(\la\wb_{j',r'}^{(\tilde{T}-1)}, \hat{y}_{i} \cdot \bmu\ra) - \sum_{i\in S_{-}} \ell_i'^{(\tilde{T}-1)} \sigma'(\la\wb_{j',r'}^{(\tilde{T}-1)}, \hat{y}_{i} \cdot \bmu\ra)\bigg]  \cdot \| \bmu \|_2^2\\
    &\geq\gamma_{j',r'}^{(\tilde{T}-1)} - \frac{\eta}{2nm} \cdot\sum_{i\in S_{+}} \ell_i'^{(\tilde{T}-1)} \sigma'(\la\wb_{j',r'}^{(\tilde{T}-1)}, \hat{y}_{i} \cdot \bmu\ra) \cdot \| \bmu \|_2^2.
\end{align*}
Then we have
\begin{equation}
    \begin{aligned}
    \gamma_{j',r'}^{(\tilde{T})}&\geq\gamma_{j',r'}^{(\tilde{T}-1)}- \frac{\eta}{2nm} \sum_{i\in S_{+}\cap S_1} \ell_i'^{(\tilde{T}-1)}\cdot \| \bmu \|_2^2\geq\gamma_{j',r'}^{(\tilde{T}-1)}+\frac{\eta\| \bmu \|_2^2}{8m}\min_{i\in S_{+}\cap S_1}\ell_i'^{(\tilde{T}-1)}, \text{ if }\la\wb_{j',r'}^{(\tilde{T}-1)}, \bmu\ra\geq0,\\
    \gamma_{j',r'}^{(\tilde{T})}&\geq\gamma_{j',r'}^{(\tilde{T}-1)}- \frac{\eta}{2nm} \sum_{i\in S_{+}\cap S_{-1}} \ell_i'^{(\tilde{T}-1)}\cdot \| \bmu \|_2^2\geq\gamma_{j',r'}^{(\tilde{T}-1)}+\frac{\eta\| \bmu \|_2^2}{8m}\min_{i\in S_{+}\cap S_{-1}}\ell_i'^{(\tilde{T}-1)}, \text{ if }\la\wb_{j',r'}^{(\tilde{T}-1)}, \bmu\ra<0,\label{ineq15}
    \end{aligned}
\end{equation}
where the second inequality is by Lemma \ref{lm: estimate S cap S}. According to the fourth statement of Proposition \ref{proposition: balanced logit}, we have $\max_{i\in S_{j,r}^{(\tilde{T}-1)}}|\ell_i'^{(\tilde{T}-1)}|\leq C_2\min_{i\in S_{+}\cap S_{1}}\ell_i'^{(\tilde{T}-1)}$ and $\max_{i\in S_{j,r}^{(\tilde{T}-1)}}|\ell_i'^{(\tilde{T}-1)}|\leq C_2\min_{i\in S_{+}\cap S_{-1}}\ell_i'^{(\tilde{T}-1)}$. Then by combining \eqref{ineq14} and \eqref{ineq15}, we have
\begin{equation}
    \frac{\sum_{i=1}^{n}\zeta_{j,r,i}^{(\tilde{T})}}{\gamma_{j',r'}^{(\tilde{T})}}\leq\max\Bigg\{\frac{\sum_{i=1}^{n}\zeta_{j,r,i}^{(\tilde{T}-1)}}{\gamma_{j',r'}^{(\tilde{T}-1)}},\frac{9C_2\sigma_p^2 d}{\|\bmu\|_2^2}\Bigg\}=\Theta(\mathrm{SNR}^{-2}).\label{ineq17}
\end{equation}
By \eqref{ineq16} and \eqref{ineq17}, we have
\begin{equation*}
    \frac{\sum_{i=1}^{n}\zeta_{j,r,i}^{(\tilde{T})}}{\gamma_{j',r'}^{(\tilde{T})}}=\Theta(\mathrm{SNR}^{-2}),
\end{equation*}
which completes the induction.

\end{proof}

Actually, the result in Lemma \ref{lm: zeta,gamma ratio} also holds for $0\leq t\leq T^*$, that is,
\begin{lemma}
Under the same conditions as Theorem \ref{thm:signal_learning_main}, we have
\begin{equation}
    \sum_{i=1}^{n}\zeta_{j,r,i}^{(t)}/\gamma_{j',r'}^{(t)}=\Theta(\mathrm{SNR^{-2}})
\end{equation}
for all $j, j'\in\{\pm 1\}$, $r, r'\in[m]$ and $0\leq t\leq T^*$.    
\end{lemma}
The proof argument is nearly the same as Lemma \ref{lm: zeta,gamma ratio}, and we only need to use Lemma \ref{lm: zeta,gamma ratio} in later arguments, so we eliminate the proof details here.
 
\section{Test Error Analysis}

\subsection{Test Error Upper Bound}

Next, we give an upper bound for the test error at iteration $t$ defined in Theorem \ref{thm:signal_learning_main} when the training loss converges to $\epsilon$. First of all, notice that $T_1\leq t\leq T^*$ by Lemma \ref{lm: loss upper bound}, we can summarize previous results into the following:

\begin{itemize}
    \item $\sum_{i=1}^{n}\zeta_{j,r,i}^{(t)}/\gamma_{j',r'}^{(t)}=\Theta(\mathrm{SNR^{-2}})$ (from Lemma~\ref{lm: zeta,gamma ratio}), 
    \item $\sum_{i=1}^{n} \zeta_{j,r,i}^{(t)} = \Omega(n) = O(n \log(T^*)) = \tilde{\Theta}(n)$ (from Proposition~\ref{proposition: range of gamma,zeta,omega} and Lemma~\ref{lm: first stage})
    \item $\max_{j, r, i} | \omega_{j, r, i}^{(t)} | = \max \big\{O\big(\sqrt{\log(mn/\delta)}\cdot\sigma_0\sigma_p\sqrt{d}\big), O \big( \sqrt{\log(n / \delta)} \log(T^*) \cdot n / \sqrt{d} \big) \big\}$ (from Lemma~\ref{lm: loss upper bound}). 
\end{itemize}

Additionally, recalling the definition $\hat{\gamma} = n \cdot \mathrm{SNR}^2$, from the first two conclusions, we have $\gamma_{j, r}^{(t)} = \tilde{\Theta} (\hat{\gamma})$ for all $j, r$. Also note that from the third conclusion, since $\sigma_0 \sigma_p \sqrt{d} = \tilde{O} (\sqrt{n} / \sqrt{d}) = o(1)$ and $\sqrt{\log(n / \delta)} \log(T^*) \cdot n / \sqrt{d} = O (1)$ from Condition~\ref{condition:d_sigma0_eta}, we have $\max_{j, r, i} | \omega_{j, r, i}^{(t)} | = O (1)$ and so $\sum_{i=1}^n | \omega_{\hat{y}, r, i}^{(t)} | = O \big( \sum_{i=1}^n \zeta_{\hat{y}, r, i}^{(t)} \big)$, hence we can ignore the sum of $\omega$ whenever it appears together with the sum of $\zeta$. We are now ready to analyze the test error in the following theorem. 


\begin{theorem}[Second part of Theorem~\ref{thm:signal_learning_main}]
    Under the same conditions as Theorem~\ref{thm:signal_learning_main}, then there exists a large constant $C_1$ such that when $n\|\bmu\|_2^2\geq C_1\sigma_p^4 d$, for time $t$ defined in Lemma~\ref{lm: loss upper bound}, we have the test error 
    \begin{equation*}
        \mathbb{P}_{(\xb, y) \sim \cD} \big( y \neq \sign (f(\Wb^{(t)},\xb)) \big) \leq p + \exp\bigg(-  n\|\bmu\|_{2}^{4}/(C_2\sigma_{p}^{4}d)\bigg), 
    \end{equation*}
    where $C_2 = O(1)$. 
\end{theorem}

\begin{proof}
For the sake of convenience, we use $(\xb, \hat{y}, y) \sim \cD$ to denote the following: data point $(\xb, y)$ follows distribution $\cD$ defined in Definition~\ref{def:data}, and $\hat{y}$ is its true label. We can write out the test error as
\begin{equation}
\begin{aligned}
    &\quad \mathbb{P}_{(\xb,y)\sim\cD} \big(y\neq \sign(f(\Wb^{(t)},\xb))\big) \\
    &= \mathbb{P}_{(\xb,y)\sim\cD}\big(y f(\Wb^{(t)},\xb) \leq 0\big)\\
    &= \mathbb{P}_{(\xb, y) \sim \cD} \big(y f(\Wb^{(t)},\xb) \leq 0, y \neq \hat{y} \big) + \mathbb{P}_{(\xb, \hat{y}, y)\sim\cD} \big(y f(\Wb^{(t)},\xb) \leq 0, y = \hat{y} \big)\\
    &= p\cdot \mathbb{P}_{(\xb, \hat{y}, y) \sim \cD} \big( \hat{y} f(\Wb^{(t)},\xb) \geq 0\big) + (1-p) \cdot \mathbb{P}_{(\xb, \hat{y}, y) \sim \cD} \big( \hat{y} f(\Wb^{(t)},\xb) \leq 0 \big)  \\
    &\leq p+\mathbb{P}_{(\xb,\hat{y}, y)\sim\cD}\big(\hat{y} f(\Wb^{(t)},\xb) \leq 0\big), \label{ineq: test error analysis 1}
\end{aligned}
\end{equation}
where in the second equation we used the definition of $\cD$ in Definition~\ref{def:data}. 
It therefore suffices to provide an upper bound for $\mathbb{P}_{(\xb,\hat{y})\sim\cD}\big(\hat{y} f(\Wb^{(t)},\xb) \leq 0\big)$. To achieve this, we write $\xb = (\hat{y}\bmu, \bxi)$, and get
\begin{align}
    \hat{y} f(\Wb^{(t)},\xb)&=\frac{1}{m}\sum_{j,r}\hat{y}j[\sigma(\la\wb_{j,r}^{(t)},\hat{y}\bmu\ra)+\sigma(\la\wb_{j,r}^{(t)},\bxi\ra)] \notag \\
    &=\frac{1}{m}\sum_{r}[\sigma(\la\wb_{\hat{y},r}^{(t)},\hat{y}\bmu\ra)+\sigma(\la\wb_{\hat{y},r}^{(t)},\bxi\ra)]-\frac{1}{m}\sum_{r}[\sigma(\la\wb_{-\hat{y},r}^{(t)},\hat{y}\bmu\ra)+\sigma(\la\wb_{-\hat{y},r}^{(t)},\bxi\ra)] \label{eq: test logit decomposition}
\end{align}

Now consider first the expressions $\la \wb_{j, r} ^{(t)}, \hat{y} \bmu \ra$ for $j = \pm \hat{y}$. Recall from \eqref{eq:w_decomposition} the signal-noise decomposition of $\wb_{j,r}^{(t)}$: 
\begin{equation*}
    \wb_{j,r}^{(t)} = \wb_{j,r}^{(0)} + j \cdot \gamma_{j,r}^{(t)} \cdot \| \bmu \|_2^{-2} \cdot \bmu + \sum_{ i = 1}^n \zeta_{j,r,i}^{(t) }\cdot \| \bxi_i \|_2^{-2} \cdot \bxi_{i} + \sum_{ i = 1}^n \omega_{j,r,i}^{(t) }\cdot \| \bxi_i \|_2^{-2} \cdot \bxi_{i},
\end{equation*}
hence the inner product with $j = \hat{y}$ can be bounded as 
\begin{equation}
    \begin{aligned}
    \la \wb_{\hat{y}, r}^{(t)}, \hat{y} \bmu \ra &= \la \wb_{\hat{y}, r}^{(0)}, \hat{y} \bmu \ra + \gamma_{\hat{y}, r}^{(t)} + \sum_{i=1}^n \zeta_{\hat{y},r,i}^{(t)} \cdot \|\bxi_i\|_2^{-2} \cdot \la \bxi_i, \hat{y} \bmu \ra + \sum_{i=1}^n \omega_{\hat{y},r,i}^{(t)} \cdot \|\bxi_i\|_2^{-2} \cdot \la \bxi_i, \hat{y} \bmu \ra \\
    &\geq \gamma_{\hat{y}, r}^{(t)} - \sqrt{2 \log(12m / \delta)} \cdot \sigma_0 \|\bmu\|_2 \\
    &\qquad - \sqrt{2\log(6n/\delta)} \cdot \sigma_p \|\bmu\|_2 \cdot (\sigma_p^2 d / 2)^{-1} \bigg[ \sum_{i=1}^n \zeta_{\hat{y},r,i}^{(t)} + \sum_{i=1}^n |\omega_{\hat{y},r,i}^{(t)}| \bigg] \\
    &= \gamma_{\hat{y}, r}^{(t)} - \Theta ( \sqrt{\log(m / \delta)} \sigma_0 \|\bmu\|_2) - \Theta \big( \sqrt{\log(n / \delta)} \cdot (\sigma_p d)^{-1} \|\bmu\|_2 \big) \cdot \Theta (\mathrm{SNR}^{-2}) \cdot \gamma_{\hat{y}, r}^{(t)} \\
    &= \big[ 1 - \Theta \big( \sqrt{\log (n / \delta)} \cdot \sigma_p / \| \bmu \|_2 \big) \big] \gamma_{\hat{y}, r}^{(t)} -  \Theta \big( \sqrt{\log(m / \delta)} (\sigma_p d)^{-1} \sqrt{n} \| \bmu \|_2 \big) \\
    &= \Theta(\gamma_{\hat{y}, r}^{(t)}), \label{ineq: test inner product 1}
    \end{aligned}
\end{equation}
where 
the inequality is by 
Lemma~\ref{lm: data inner products} and Lemma~\ref{lm: initialization inner products}; 
the second equality is obtained by plugging in the coefficient orders  we summarized at the start of the section; 
the third equality is by the condition $\sigma_0 \leq C^{-1} (\sigma_p d)^{-1} \sqrt{n}$ in Condition~\ref{condition:d_sigma0_eta} and $\mathrm{SNR}=\|\bmu\|_2 / \sigma_p \sqrt{d}$; 
for the fourth equality, notice that $\gamma_{j, r}^{(t)} = \Omega (\hat{\gamma})$, also $\sqrt{\log (n / \delta)} \cdot \sigma_p / \| \bmu \|_2 \leq 1/\sqrt{C}$ and $\sqrt{\log(m / \delta)}(\sigma_p d)^{-1} \sqrt{n} \| \bmu \|_2 / \hat{\gamma} = \sqrt{\log(m / \delta)} \sigma_p / (\sqrt{n} \| \bmu \|_2) \leq \sqrt{\log(m / \delta) / n} \cdot 1 / \big( \sqrt{C \log(n/\delta)} \big) \leq 1/(C\sqrt{\log(n/\delta)})$ holds by $\|\bmu\|_2^2 \geq C \cdot \sigma_p^2 \log(n / \delta)$ and $n\geq C\log(m/\delta)$ in Condition \ref{condition:d_sigma0_eta}, so for sufficiently large constant $C$ the equality holds. Moreover, we can deduce in a similar manner that
\begin{equation}
    \begin{aligned}
    \la \wb_{-\hat{y}, r}^{(t)}, \hat{y} \bmu \ra &= \la \wb_{-\hat{y}, r}^{(0)}, \hat{y} \bmu \ra - \gamma_{-\hat{y}, r}^{(t)} + \sum_{i=1}^n \zeta_{-\hat{y},r,i}^{(t)} \cdot \|\bxi_i\|_2^{-2} \cdot \la \bxi_i, -\hat{y} \bmu \ra + \sum_{i=1}^n \omega_{-\hat{y},r,i}^{(t)} \cdot \|\bxi_i\|_2^{-2} \cdot \la \bxi_i, \hat{y} \bmu \ra \\
    &\leq - \gamma_{-\hat{y}, r}^{(t)} + \sqrt{2 \log(8m / \delta)} \cdot \sigma_0 \|\bmu\|_2 \\
    &\quad + \sqrt{2\log(6n/\delta)} \cdot \sigma_p \|\bmu\|_2 \cdot (\sigma_p^2 d / 2)^{-1} \bigg[ \sum_{i=1}^n \zeta_{-\hat{y},r,i}^{(t)} + \sum_{i=1}^n |\omega_{-\hat{y},r,i}^{(t)}| \bigg] \\
    &= -\Theta (\gamma_{-\hat{y}, r}^{(t)}) < 0, \label{ineq: test inner product 2}
    \end{aligned}
\end{equation}
where the second equality holds based on similar analyses as in \eqref{ineq: test inner product 1}. 


Denote $g(\bxi)$ as $\sum_{r}\sigma(\la\wb_{-\hat{y},r}^{(t)},\bxi\ra)$. According to Theorem 5.2.2 in \citet{vershynin_2018}, we know that for any $x\geq 0$ it holds that
\begin{equation}
    \PP(g(\bxi)-\EE g(\bxi)\geq x)\leq \exp\Big(-\frac{cx^2}{\sigma_{p}^{2}\|g\|_{\mathrm{Lip}}^{2}}\Big), \label{eq: call1}
\end{equation}
where $c$ is a constant. To calculate the Lipschitz norm, we have
\begin{align*}
    |g(\bxi)-g(\bxi')| &= \Bigg| \sum_{r=1}^{m} \sigma(\la\wb_{-\hat{y},r}^{(t)},\bxi\ra)-\sum_{r=1}^{m}\sigma(\la\wb_{-\hat{y},r}^{(t)},\bxi'\ra)\Bigg|\\
    &\leq\sum_{r=1}^{m}\big|\sigma(\la\wb_{-\hat{y},r}^{(t)},\bxi\ra)-\sigma(\la\wb_{-\hat{y},r}^{(t)},\bxi'\ra)\big|\\
    &\leq\sum_{r=1}^{m}|\la\wb_{-\hat{y},r}^{(t)},\bxi-\bxi'\ra|\\
    &\leq\sum_{r=1}^{m}\big\|\wb_{-\hat{y},r}^{(t)}\big\|_{2}\cdot\|\bxi-\bxi'\|_{2},
\end{align*}
where the first inequality is by triangle inequality; the second inequality is by the property of ReLU; the last inequality is by Cauchy-Schwartz inequality. Therefore, we have
\begin{equation}
    \|g\|_{\mathrm{Lip}} \leq\sum_{r=1}^{m}\big\|\wb_{-\hat{y},r}^{(t)}\big\|_{2}, \label{eq: lip}
\end{equation}
and since $\la\wb_{-\hat{y},r}^{(t)},\bxi\ra\sim\cN\big(0,\|\wb_{-\hat{y},r}^{(t)}\|_2^2\sigma_p^2\big)$, we can get
\begin{equation*}
    \EE g(\bxi) = \sum_{r=1}^{m}\EE\sigma(\la\wb_{-\hat{y},r}^{(t)},\bxi\ra)=\sum_{r=1}^{m}\frac{\|\wb_{-\hat{y},r}^{(t)}\|_2\sigma_p}{\sqrt{2\pi}}=\frac{\sigma_p}{\sqrt{2\pi}}\sum_{r=1}^{m}\|\wb_{-\hat{y},r}^{(t)}\|_2. 
\end{equation*}
Next we seek to upper bound the $2$-norm of $\wb_{j, r}^{(t)}$. First, we tackle the noise section in the decomposition, namely: 
\begin{align*}
    &\quad \bigg \| \sum_{ i = 1}^n \rho_{j, r, i}^{(t)} \cdot \| \bxi_i \|_2^{-2} \cdot \bxi_i \bigg \|_2^2 \\
    &= \sum_{i=1}^n {\rho_{j, r, i}^{(t)}}^2 \cdot \| \bxi_i \|_2^{-2} + 2 \sum_{1 \leq i_1 < i_2 \leq n} \rho_{j, r, i_1}^{(t)} \rho_{j, r, i_2}^{(t)} \cdot \| \bxi_{i_1} \|_2^{-2} \cdot \| \bxi_{i_2} \|_2^{-2} \cdot \la \bxi_{i_1}, \bxi_{i_2} \ra \\
    &\leq 4 \sigma_p^{-2} d^{-1} \sum_{i=1}^n {\rho_{j, r, i}^{(t)}}^2 + 2 \sum_{1 \leq i_1 < i_2 \leq n} |\rho_{j, r, i_1}^{(t)} \rho_{j, r, i_2}^{(t)}| \cdot (16 \sigma_p^{-4} d^{-2}) \cdot (2\sigma_p^2 \sqrt{d \log(6n^2 / \delta)}) \\
    &= 4 \sigma_p^{-2} d^{-1} \sum_{i=1}^n {\rho_{j, r, i}^{(t)}}^2 + 32 \sigma_p^{-2} d^{-3/2} \sqrt{\log(6n^2/\delta)} \bigg[ \bigg( \sum_{i=1}^n |\rho_{j, r, i}^{(t)}| \bigg)^2 - \sum_{i=1}^n {\rho_{j, r, i}^{(t)}}^2 \bigg] \\
    &= \Theta (\sigma_p^{-2} d^{-1}) \sum_{i=1}^n {\rho_{j, r, i}^{(t)}}^2 + \tilde{\Theta} (\sigma_p^{-2} d^{-3/2}) \bigg( \sum_{i=1}^n |\rho_{j, r, i}^{(t)}| \bigg)^2 \\
    &\leq \big[ \Theta (\sigma_p^{-2} d^{-1} n^{-1}) + \tilde{\Theta} (\sigma_p^{-2} d^{-3/2}) \big] \bigg( \sum_{i=1}^n |\zeta_{j, r, i}^{(t)}| + \sum_{i=1}^n |\omega_{j, r, i}^{(t)}| \bigg)^2 \\
    &\leq \Theta (\sigma_p^{-2} d^{-1} n^{-1})\bigg( \sum_{i=1}^n \zeta_{j, r, i}^{(t)}\bigg)^2
\end{align*}
where for the first inequality we used Lemma~\ref{lm: data inner products}; for the second inequality we used the definition of $\zeta, \omega$; for the second to last equation we plugged in coefficient orders. 
We can thus upper bound the norm of $\wb_{j,r}^{(t)}$ as: 
\begin{align}
    \|\wb_{j,r}^{(t)}\|_2 &\leq \|\wb_{j,r}^{(0)}\|_2 + \gamma_{j,r}^{(t)} \cdot \| \bmu \|_2^{-1} + \bigg \| \sum_{ i = 1}^n \rho_{j,r,i}^{(t)} \cdot \| \bxi_i \|_2^{-2} \cdot \bxi_i \bigg \|_2 \notag \\
    &\leq \|\wb_{j,r}^{(0)}\|_2 + \gamma_{j,r}^{(t)} \cdot \| \bmu \|_2^{-1} + \Theta(\sigma_{p}^{-1}d^{-1/2}n^{-1/2})\cdot\sum_{i=1}^{n}\zeta_{j,r,i}^{(t)} \notag \\
    &= \Theta ( \sigma_{p}^{-1} d^{-1/2} n^{-1/2}) \cdot \sum_{i = 1}^{n} \zeta_{j, r, i}^{(t)}
\end{align}
where the first inequality is due to the triangle inequality, and the 
equality is due to the following comparisons: 
\begin{equation*}
    \frac{ \gamma_{j, r}^{(t)} \cdot \| \bmu \|_{2}^{-1}} {\Theta (\sigma_{p}^{-1} d^{-1/2} n^{-1/2}) \cdot \sum_{i=1}^{n} \zeta_{j,r,i}^{(t)}} = \Theta( \sigma_{p} d^{1/2} n^{1/2} \| \bmu \|_2^{-1} \mathrm{SNR}^{2}) = \Theta (\sigma_p^{-1} d^{-1/2} n^{1/2} \| \bmu \|_2) = O(1)
\end{equation*}
based on the coefficient order $\sum_{i=1}^{n} \zeta_{j,r,i}^{(t)} / \gamma_{j, r}^{(t)} = \Theta (\mathrm{SNR} ^{-2})$, the definition $\mathrm{SNR} = \| \bmu \|_2 / (\sigma_p \sqrt{d})$, and the condition for $d$ in Condition~\ref{condition:d_sigma0_eta}; and also
\begin{equation*}
    \frac{\|\wb_{j,r}^{(0)}\|_2}{\Theta(\sigma_{p}^{-1}d^{-1/2}n^{-1/2})\cdot\sum_{i=1}^{n}\zeta_{j,r,i}^{(t)}}=\frac{\Theta(\sigma_0\sqrt{d})}{\Theta(\sigma_{p}^{-1}d^{-1/2}n^{-1/2})\cdot\sum_{i=1}^{n}\zeta_{j,r,i}^{(t)}}=O(\sigma_0\sigma_p d n^{-1/2})=O(1) 
\end{equation*}
based on Lemma~\ref{lm: initialization inner products}, the coefficient order $\sum_{i=1}^{n} \zeta_{j,r,i}^{(t)} = \Omega (n)$, and the condition for $\sigma_0$ in Condition~\ref{condition:d_sigma0_eta}. With this and \eqref{ineq: test inner product 1}, we give an analysis of the following the key component,
\begin{equation}
    \frac{\sum_r \sigma(\la\wb_{\hat{y}, r}^{(t)},\hat{y}\bmu\ra)}{\sigma_p\sum_{r=1}^{m}\big\|\wb_{-\hat{y},r}^{(t)}\big\|_{2}}
    \geq \frac{ \Theta \big( \sum_{r} \gamma_{\hat{y}, r}^{(t)} \big)}{\Theta (d^{-1/2} n^{-1/2}) \cdot \sum_{r,i} \zeta_{-\hat{y}, r, i}^{(t)}}
    = \Theta(d^{1/2} n^{1/2} \mathrm{SNR}^2)
    = \Theta(n^{1/2} \|\bmu\|_{2}^{2} / \sigma_p^2 d^{1/2})\label{order of ratio}
\end{equation}

By \eqref{order of ratio} and $n\|\bmu\|_{2}^{4}\geq C_1\sigma_p^4 d$ where $C_1$ is a sufficiently large constant, it directly follows that
\begin{equation}\label{greater than 0}
    \begin{aligned}
        \sum_r \sigma(\la\wb_{\hat{y}, r}^{(t)},\hat{y}\bmu\ra)-\frac{\sigma_p}{\sqrt{2\pi}}\sum_{r=1}^{m}\|\wb_{-\hat{y},r}^{(t)}\|_2>0.
    \end{aligned}
\end{equation}

Now using the method in \eqref{eq: call1} with the results above, we plug \eqref{ineq: test inner product 2} into \eqref{eq: test logit decomposition} and then \eqref{ineq: test error analysis 1}, to obtain
\begin{align}
    \PP_{(\xb, \hat{y}, y) \sim \cD} \big( \hat{y} f(\bW^{(t)}, \xb) \leq 0 \big) 
    &\leq \PP_{(\xb, \hat{y}, y) \sim \cD} \bigg( \sum_{r} \sigma(\la \wb_{-\hat{y}, r}^{(t)}, \bxi \ra) \geq \sum_r \sigma(\la\wb_{\hat{y}, r}^{(t)},\hat{y}\bmu\ra)\bigg) \notag \\
    &=\PP_{(\xb, \hat{y}, y) \sim \cD} \Bigg(g(\bxi)-\EE g(\bxi) \geq \sum_r \sigma(\la\wb_{\hat{y}, r}^{(t)},\hat{y}\bmu\ra)-\frac{\sigma_p}{\sqrt{2\pi}}\sum_{r=1}^{m}\|\wb_{-\hat{y},r}^{(t)}\|_2\Bigg) \notag \\
    &\leq \exp \Bigg[ -\frac{c\Big(\sum_r \sigma(\la\wb_{\hat{y}, r}^{(t)},\hat{y}\bmu\ra)- (\sigma_p / \sqrt{2\pi}) \sum_{r=1}^{m}\|\wb_{-\hat{y},r}^{(t)} \big\|_2\Big)^{2}}{\sigma_{p}^{2}\Big(\sum_{r=1}^{m} \big\|\wb_{-\hat{y},r}^{(t)} \big\|_{2}\Big)^{2}} \Bigg] \notag \\
    &= \exp \bigg[ -c\bigg(\frac{\sum_r \sigma(\la\wb_{\hat{y}, r}^{(t)},\hat{y}\bmu\ra)}{\sigma_p\sum_{r=1}^{m}\big\|\wb_{-\hat{y},r}^{(t)}\big\|_{2}} - 1/\sqrt{2\pi}\bigg)^{2} \bigg] \notag \\
    &\leq \exp(c/2\pi)\exp\bigg(- 0.5c\bigg(\frac{\sum_r \sigma(\la\wb_{\hat{y}, r}^{(t)},\hat{y}\bmu\ra)}{\sigma_p\sum_{r=1}^{m}\big\|\wb_{-\hat{y},r}^{(t)}\big\|_{2}}\bigg)^{2}\bigg)\label{ineq: test error analysis 3}
\end{align} 
where the second inequality is by \eqref{greater than 0} and plugging \eqref{eq: lip} into \eqref{eq: call1}, the third inequality is due to the fact that $(s-t)^{2} \geq s^{2}/2 - t^{2}, \forall s, t \geq 0$.

And we can get from \eqref{order of ratio} and \eqref{ineq: test error analysis 3} that
\begin{align*}
    \PP_{(\xb, \hat{y}, y) \sim \cD} \big( \hat{y} f(\bW^{(t)}, \xb) \leq 0 \big)
    &\leq \exp(c/2\pi)\exp\bigg(- 0.5c\bigg(\frac{\sum_r \sigma(\la\wb_{\hat{y}, r}^{(t)},\hat{y}\bmu\ra)}{\sigma_p\sum_{r=1}^{m}\big\|\wb_{-\hat{y},r}^{(t)}\big\|_{2}}\bigg)^{2}\bigg)\\
    &= \exp\Big(\frac{c}{2\pi}- \frac{n\|\bmu\|_{2}^{4}}{C\sigma_{p}^{4}d}\Big)\\
    &\leq \exp\Big(-\frac{n\|\bmu\|_{2}^{4}}{2C\sigma_{p}^{4}d}\Big)\\
    &=\exp\Big(-\frac{n\|\bmu\|_{2}^{4}}{C_2\sigma_{p}^{4}d}\Big),
\end{align*}
where $C = O(1)$; the last inequality holds if we choose $C_1\geq cC/\pi$; the last equality holds if we choose $C_2$ as $2C$.

\end{proof}

\subsection{Test Error Lower Bound}

In this section, we will give the lower bound of the test error at iteration $t$ defined in Theorem \ref{thm:signal_learning_main} when the training loss converges to $\epsilon$, which, together with Theorem~\ref{thm:signal_learning_main}, shows a sharp phase transition.  First, we give the proof of key Lemma~\ref{lemma:g lower bound}.
\begin{proof}[Proof of Lemma~\ref{lemma:g lower bound}]
 Without loss of generality, let $\max\Big\{\sum_{r}\gamma_{1,r}^{(t)}, \sum_{r}\gamma_{-1,r}^{(t)}\Big\} = \sum_{r}\gamma_{1,r}^{(t)} $. Denote $\vb = \lambda \cdot \sum_{i}\ind(y_i = 1)\bxi_{i}$, where $\lambda = C_{7}\|\bmu\|_{2}^{2}/(d\sigma_p^{2})$ and $C_{7}$ is a sufficiently large constant. Then we only need to prove that
\begin{equation}
    \underbrace{g(\bxi + \vb) - g(\bxi) + g(-\bxi + \vb) - g(-\bxi)}_{I} \geq 4C_{6}\sum_{r}\gamma_{1,r}^{(t)}.\label{eq:g upbound}
\end{equation}

Since ReLU is a convex activation function, we have that 
\begin{align}
\sigma(\la \wb_{1,r}^{(t)}, \bxi + \vb\ra) - \sigma(\la \wb_{1,r}^{(t)}, \bxi\ra) &\geq \sigma'(\la  \wb_{1,r}^{(t)}, \bxi\ra)\la \wb_{1,r}^{(t)}, \vb \ra  \label{eq:conv1}\\
\sigma(\la \wb_{1,r}^{(t)}, -\bxi + \vb\ra) - \sigma(\la \wb_{1,r}^{(t)}, -\bxi\ra) &\geq \sigma'(\la  \wb_{1,r}^{(t)}, -\bxi\ra)\la \wb_{1,r}^{(t)}, \vb \ra \label{eq:conv2}. 
\end{align}
Adding \eqref{eq:conv1} and \eqref{eq:conv2} we have that almost surely for all $\bxi$
\begin{equation}
\begin{aligned}
&\sigma(\la \wb_{1,r}^{(t)}, \bxi + \vb\ra) - \sigma(\la \wb_{1,r}^{(t)}, \bxi\ra) + \sigma(\la \wb_{1,r}^{(t)}, -\bxi + \vb\ra) - \sigma(\la \wb_{1,r}^{(t)}, -\bxi\ra)\\
&\geq \la \wb_{1,r}^{(t)}, \vb \ra \\
&\geq \lambda\bigg[\sum_{y_i = 1}\zeta_{1,r,i}^{(t)} - 2n\sqrt{\log(12mn/\delta)}\cdot \sigma_0 \sigma_p \sqrt{d} - 5n^{2}\alpha\sqrt{\log(6n^{2}/\delta)/d}\bigg],\label{eq:g upbound1}
\end{aligned}
\end{equation}
where the last inequality is by \eqref{ineq: noise product bound2} and Lemma \ref{lm: initialization inner products}.
Since ReLU is a Liptchitz, we also have that 
\begin{equation}
\begin{aligned}
&\sigma(\la \wb_{-1,r}^{(t)}, \bxi + \vb\ra) - \sigma(\la \wb_{-1,r}^{(t)}, \bxi\ra) + \sigma(\la \wb_{-1,r}^{(t)}, -\bxi + \vb\ra) - \sigma(\la \wb_{-1,r}^{(t)}, -\bxi\ra)\\
&\leq 2|\la \wb_{-1,r}^{(t)}, \vb \ra| \\
&\leq 2\lambda\bigg[\sum_{y_i = 1}\omega_{-1,r,i}^{(t)} + 2n\sqrt{\log(12mn/\delta)}\cdot \sigma_0 \sigma_p \sqrt{d} + 5n^{2}\alpha\sqrt{\log(6n^{2}/\delta)/d}\bigg],\label{eq:g upbound2}
\end{aligned}
\end{equation}
where the last inequality is by \eqref{ineq: noise product bound1} and Lemma \ref{lm: initialization inner products}. Therefore, by plugging \eqref{eq:g upbound1} and \eqref{eq:g upbound2} into left hand side $I$ in \eqref{eq:g upbound}, we have that 
\begin{align*}
&g(\bxi + \vb) - g(\bxi) + g(-\bxi + \vb) - g(-\bxi) \\
&\geq \lambda\bigg[\sum_{r}\sum_{y_i = 1}\zeta_{1,r,i}^{(t)} - 6nm\sqrt{\log(12mn/\delta)}\cdot \sigma_0 \sigma_p \sqrt{d} - 15m n^{2}\alpha\sqrt{\log(6n^{2}/\delta)/d}\bigg] \\
&\geq (\lambda/2) \cdot \sum_{r}\sum_{y_i = 1}\zeta_{1,r,i}^{(t)} \\
&\geq \lambda/2 \cdot \Theta(\mathrm{SNR}^{-2})\sum_{r} \gamma_{1,r}^{(t)} \\
&\geq 4 C_{6}\sum_{r} \gamma_{1,r}^{(t)},
\end{align*}
where the second inequality is by Lemma \ref{lm: first stage} and Condition \ref{condition:d_sigma0_eta}; the third inequality is by Lemma \ref{lm: zeta,gamma ratio}. Finally, it is worth noting that the norm 
\begin{align*}
\|\vb\|_{2} = \|\lambda \cdot \sum_{i}\ind(y_i = 1)\bxi_{i}\|_{2} = \Theta\bigg(\sqrt{\frac{n\|\bmu\|_{2}^{4}}{\sigma_{p}^{4}d}}\bigg) \leq 0.06\sigma_p,\\   
\end{align*}
where the last inequality is by condition $n\|\bmu\|_{2}^{4} \leq C_3\sigma_p^{4} d$ with sufficiently large $C_3$ in Theorem~\ref{thm:signal_learning_main}, which completes the proof.
\end{proof}

Then we present an important Lemma, which bounds the Total Variation (TV) distance between two Gaussian with the same covariance matrix.

\begin{lemma}[Proposition 2.1 in \citet{devroye2018total}]\label{lm: TV}
The TV distance between $\cN(0, \sigma_{p}^{2}\Ib_{d})$ and $\cN(\vb, \sigma_{p}^{2}\Ib_{d})$ is smaller than $\|\vb\|_{2}/2\sigma_p$.
\end{lemma}

Finally, we can prove the third part of Theorem \ref{thm:signal_learning_main}: given Lemma~\ref{lm: TV} and Lemma~\ref{lemma:g lower bound}.

\begin{theorem}[Third part of Theorem \ref{thm:signal_learning_main}]\label{thm: test error lower bound}
Suppose that $n\|\bmu\|_{2}^{4} \leq C_3 d\sigma_{p}^{4}$, then we have that $L_{\cD}^{0-1} (\Wb^{(t)}) \geq p + 0.1$, where $C_3$ is an sufficiently large absolute constant.
\end{theorem}

\begin{proof}
For the sake of convenience, we use $(\xb, \hat{y}, y) \sim \cD$ to denote the following: data point $(\xb, y)$ follows distribution $\cD$ defined in Definition~\ref{def:data}, and $\hat{y}$ is its true label.
By \eqref{ineq: test error analysis 1}, we have
\begin{equation}
    \begin{aligned}
        &\mathbb{P}_{(\xb,y)\sim\cD} \big(y\neq \sign(f(\Wb^{(t)},\xb))\big)\\
        &= p\cdot \mathbb{P}_{(\xb, \hat{y}, y) \sim \cD} \big( \hat{y} f(\Wb^{(t)},\xb) \geq 0\big) + (1-p) \cdot \mathbb{P}_{(\xb, \hat{y}, y) \sim \cD} \big( \hat{y} f(\Wb^{(t)},\xb) \leq 0 \big)\\
        &=p+(1-2p)\cdot \mathbb{P}_{(\xb, \hat{y}, y) \sim \cD} \big( \hat{y} f(\Wb^{(t)},\xb) \leq 0 \big). 
    \end{aligned}\label{eq11}
\end{equation}
Therefore, it suffices to provide a lower bound for $\mathbb{P}_{(\xb,\hat{y})\sim\cD}\big(\hat{y} f(\Wb^{(t)},\xb) \leq 0\big)$. To achieve this, we have
\begin{equation}
    \begin{aligned}
    &\PP_{(\xb, \hat{y}, y) \sim \cD} \big( \hat{y} f(\bW^{(t)}, \xb) \leq 0 \big) \\
    &= \PP_{(\xb, \hat{y}, y) \sim \cD} \bigg( \sum_{r} \sigma(\la \wb_{-\hat{y}, r}^{(t)}, \bxi \ra) - \sum_{r} \sigma(\la \wb_{\hat{y}, r}^{(t)}, \bxi \ra) \geq \sum_r \sigma(\la\wb_{\hat{y}, r}^{(t)},\hat{y}\bmu\ra) - \sum_r \sigma(\la\wb_{-\hat{y}, r}^{(t)},\hat{y}\bmu\ra)\bigg)\\   
    &\geq 0.5\PP_{(\xb, \hat{y}, y) \sim \cD} \bigg( \bigg|\sum_{r} \sigma(\la \wb_{1, r}^{(t)}, \bxi \ra) - \sum_{r} \sigma(\la \wb_{-1, r}^{(t)}, \bxi \ra)\bigg| \geq  C_{6}\max\Big\{\sum_{r}\gamma_{1,r}^{(t)}, \sum_{r}\gamma_{-1,r}^{(t)}\Big\}\bigg)\label{eq10}
    \end{aligned}
\end{equation}
where $C_{6}$ is a constant, the inequality holds since if $\bigg|\sum_{r} \sigma(\la \wb_{1, r}^{(t)}, \bxi \ra) - \sum_{r} \sigma(\la \wb_{-1, r}^{(t)}, \bxi \ra)\bigg|$ is too large we can always pick a corresponding $\hat{y}$ given $\bxi$ to make a wrong prediction. Let $g(\bxi) = \sum_{r} \sigma(\la \wb_{1, r}^{(t)}, \bxi \ra) - \sum_{r} \sigma(\la \wb_{-1, r}^{(t)}, \bxi \ra) $. Denote the set 
\begin{align*}
\Omega := \bigg\{\bxi \bigg|  |g(\bxi)| \geq  C_{6}\max\Big\{\sum_{r}\gamma_{1,r}^{(t)}, \sum_{r}\gamma_{-1,r}^{(t)}\Big\}\bigg\}.    
\end{align*}
By plugging the definition of $\Omega$ into \eqref{eq10}, we have
\begin{equation}
    \PP_{(\xb, \hat{y}, y) \sim \cD} \big( \hat{y} f(\bW^{(t)}, \xb) \leq 0 \big)\geq 0.5\PP(\Omega)\label{lower bound1}
\end{equation}
Next, we will give a lower bound of $\PP(\Omega)$. By Lemma~\ref{lemma:g lower bound}, we have that $\sum_{j}[g(j\bxi + \vb) - g(j\bxi)] \geq 4C_{6}\max_{j}\Big\{\sum_{r}\gamma_{j,r}^{(t)}\Big\}$



Therefore, by pigeon's hole principle, there must exist one of the $\bxi$, $\bxi + \vb$, $-\bxi$, $-\bxi + \vb$ belongs $\Omega$.  So we have proved that $\Omega \cup -\Omega \cup \Omega - \{\vb\} \cup -\Omega - \{\vb\} = \RR^{d}$. Therefore at least one of $\PP(\Omega), \PP(-\Omega), \PP(\Omega - \{\vb\}), \PP(\Omega -\{\vb\}), \PP(-\Omega - \{\vb\})$ is greater than $0.25$. Notice that $\mathbb{P}(-\Omega) = \mathbb{P}(\Omega)$ and 
\begin{align*}
|\mathbb{P}(\Omega) - \mathbb{P}(\Omega - \vb)| &= |\mathbb{P}_{\bxi \sim \cN(0, \sigma_p^{2} \Ib_{d})}(\bxi \in \Omega) - \mathbb{P}_{\bxi \sim \cN(\vb, \sigma_{p}^{2}\Ib_{d})}(\bxi \in \Omega)|\\
&\leq \text{TV}(\cN(0, \sigma_p^{2} \Ib_{d}), \cN(\vb, \sigma_p^{2} \Ib_{d}))\\
&\leq \frac{\|\vb\|_{2}}{2\sigma_{p}} \\
&\leq 0.03,
\end{align*}
where the first inequality is by the definition of Total variation (TV) distance, the second inequality is by 
Lemma~\ref{lm: TV}.

Therefore we have proved that $\mathbb{P}(\Omega) \geq 0.22$, and plugging this into \eqref{eq11} and \eqref{lower bound1}, we get
\begin{align*}
    &\mathbb{P}_{(\xb,y)\sim\cD} \big(y\neq \sign(f(\Wb^{(t)},\xb))\big)\\
    &=p+(1-2p)\cdot \mathbb{P}_{(\xb, \hat{y}, y) \sim \cD} \big( \hat{y} f(\Wb^{(t)},\xb) \leq 0 \big)\\
    &\geq p+(0.5-p)\cdot\PP(\Omega)\\
    &\geq 0.78 p+0.11\\
    &\geq p+0.1,
\end{align*}
where the last inequality is by $p<1/C$ from Condition \ref{condition:d_sigma0_eta} and by choosing $C > 22$ a sufficiently large constant, which completes the proof.
\end{proof}

\end{document}